\documentclass[11pt]{article}
\usepackage[margin=1in]{geometry}
\usepackage{setspace}
\usepackage{latexsym}
\usepackage{hyperref}
\usepackage{graphics}
\usepackage{amsmath}
\usepackage{xspace}
\usepackage{amssymb}
\usepackage{epsfig}
\usepackage{fullpage}
\usepackage{amsmath,amsthm}
\usepackage{epsfig}
\usepackage{enumerate}
\usepackage{bbm}
\usepackage{bm}
\usepackage{xcolor} 
\usepackage{algorithm}
\usepackage{algpseudocode}
\usepackage{appendix}
\usepackage{enumitem}
\usepackage{authblk}

\usepackage{natbib}
 \bibpunct[, ]{(}{)}{,}{a}{}{,}%

\newtheorem{theorem}{Theorem}
\newtheorem{lemma}{Lemma}

\newtheorem{proposition}{Proposition}
\newtheorem{corollary}{Corollary}
\newtheorem{remark}{Remark}

\newtheorem{definition}{Definition}

\newtheorem{assumption}{Assumption}

\newtheorem{condition}{Condition}

\newcommand{\PP}{{\mathbb{P}}}
\newcommand{\EE}{{\mathbb{E}}}

\newcommand{\E}{\mathbb E}
\newcommand{\Pp}{\mathbb P}

\newcommand{\barY}{\bar Y}
\newcommand{\barX}{\bar X}

\newcommand{\norm}[1]{\left\lVert #1\right\rVert}

\newcommand{\as}{\asymp}

\newcommand{\bmu}{\bm{\mu}}
\newcommand{\bn}{\mathbf{n}}
\newcommand{\bx}{\mathbf{x}}

\newcommand{\bN}{\mathbf{N}}

\newcommand{\yc}[1]{{\color{blue}[YC: #1]}}

\newcommand{\ycedit}[1]{{\color{black}#1}}
\newcommand{\jledit}[1]{{\color{black}#1}}
\newcommand{\lsedit}[1]{{\color{black}#1}}

\begin{document}

\title{Sharp Characterization of Bias in Post-Bandit Inference}
\author[1]{Lisu Wang}
\author[1]{Yilun Chen}
\author[1,2]{Jiaqi Lu}

\affil[1]{School of Data Science, The Chinese University of Hong Kong, Shenzhen}
\affil[2]{School of Management and Economics, The Chinese University of Hong Kong, Shenzhen}

\date{}
\maketitle
\begin{abstract}
    {Bandit algorithms generate data for downstream inference, but adaptive sampling biases post-bandit sample means. We analyze this bias for stable index algorithms, including \texttt{UCB1} and its generalizations, and derive sharp leading-order expressions for the sample-mean bias and expected $Z$-statistic in bandit experiments of fixed horizon $T$. Our characterization reveals the algorithmic origin of bias through a key index-function-dependent quantity, which we term \emph{effective exploration rate}. For example, under \texttt{UCB1}, the effective exploration rate is of order $\sqrt{\log T}$, and the standardized bias of any arm (that is not uniquely optimal) decays at the extremely slow rate $1/\sqrt{\log T}$. We also show how the choice of the index function affects both regret and bias, which reveals a regret-bias trade-off: more exploratory algorithm reduces bias but increases regret. \jledit{We further show how bias most severely distorts confidence intervals and hypothesis tests when the tested arm is one of the tied-optimal arms.} 
    Our sharp characterization for bias uses a novel empirical fluid approximation of the algorithm's sampling dynamics, which may be of independent interest.}
\end{abstract}

\section{Introduction}


The widespread adoption of decision analytics in modern operations has generated vast amounts of data. Unlike conventional i.i.d. samples, such data is typically adaptively collected. For example, an online content platform uses an online learning algorithm to recommend content to users; the resulting engagement data of each content providers are shaped by how the algorithm allocates user exposure, which in turn depends on past observed user responses. This feedback loop introduces inherent policy dependence, raising challenges when this data is to be used in statistical inference and other downstream decision-making tasks, and calling into question the applicability of classical statistical guarantees.


\ycedit{
To be concrete, recall that under i.i.d. sampling, the sample mean \(\bar X\) is an unbiased estimator of the population mean, and the central limit theorem (CLT) guarantees its asymptotically normal, leading to the familiar asymptotic confidence interval, e.g. \(\bar X \pm 1.96\,\mathrm{SE}\), where \(\mathrm{SE}\) is the standard error. These guarantees form the foundation of classical uncertainty quantification and hypothesis testing. However, when data are collected adaptively by an algorithm, its sampling decisions can bias the sample mean and thereby distort conventional confidence intervals and hypothesis tests. Understanding the magnitude and structure of this adaptive bias is therefore essential for assessing the reliability of inference based on such data.
}

We study this question under the multi-armed bandit (MAB) model, a canonical framework for online learning and sequential allocation in both theory and in practice. \ycedit{In this setting, multiple populations (arms) compete for samples, with allocations determined sequentially by an algorithm typically designed to minimize regret. The resulting sample sizes are therefore highly adaptive to the observed data, causing the sample mean of each arm to be systematically downward biased.} Prior work has characterized the downward direction of bandit adaptive bias \citep{nie2018adaptively,shin2019sample}, established general bounds on its magnitude \citep{shin2019bias,russo2016controlling}, and developed inference methods for batched bandit experiments \citep{zhang2020inference}. However, sharp characterizations of this bias under widely used bandit algorithms remain largely unavailable. The key challenge is full adaptivity: bandit algorithms repeatedly update their allocation rules based on accumulated observations, thereby inducing complex temporal dependencies among past outcomes, future decisions, and the resulting estimator. Moreover, these dependencies vary with the algorithmic design, making the magnitude of bias algorithm-dependent and analytically subtle. This naturally poses the following research questions:\\

\emph{How large is the sample-mean bias induced by canonical bandit algorithms, and how does it depend on their exploration rules? How does this bias affect the accuracy of statistical inference based on bandit data?}\\

In this paper, we answer these questions for a fairly large and important class of bandit algorithms,  the index algorithms, that includes the canonical \texttt{UCB1} and its generalizations. We provide sharp leading-order characterizations of both the sample-mean bias and the bias of the corresponding standardized $Z$-statistic, under the essential condition that the algorithm is stable. Our analysis identifies a critical algorithm-dependent quantity, which we call the \emph{effective exploration rate}, as the key driver of the bias.  The resulting characterizations reveal fundamental differences between instances with a unique optimal arm and those with multiple optimal arms. They further show how algorithmic design goals, such as regret minimization, may trade off against the bias. We summarize our contributions below.

\subsection{Contributions}
\emph{Sharp characterization of bias.} 
{We prove that, under certain regularity conditions on the index algorithm and for sub-Gaussian reward distributions, the bias of each arm's $Z$-statistic, $Z_{a,T}$, admits an explicit leading-order expression as horizon $T \to \infty$:
\begin{align}
    & \Gamma_{a,T}\mathbb E[Z_{a,T}]
    \longrightarrow
    - \left(1-\frac{1}{|\mathcal O|}\mathbbm 1\{a\in\mathcal O\}\right)\sigma_a,
\end{align}
where $\mathcal O$ is the set of optimal arms, $\sigma_a$ is the reward standard deviation of arm $a$, and $\Gamma_{a,T}$ is an index-function-specific quantity that we term \emph{effective exploration rate}. We also prove a similar bias expression of the corresponding sample mean, which has the same leading order as the $Z$-statistic bias after standardization. 
Specializing to the class of generalized \texttt{UCB} algorithms with index $I_t(x, n) = x + \frac{f_t}{\sqrt{n}}$, the bias of the $Z$-statistic reduces to
\begin{align*}
    \mathbb E[Z_{a,T}]\asymp-\frac{1}{f_T}
\end{align*}
except when arm $a$ is the single optimal arm, in which case the bias is order-of-magnitude smaller.}
\jledit{The bias rate $1/f_T$ decays to zero at the same rate as the inverse square-root of regret\footnote{When there is at least one suboptimal arm with $\Delta_a>0$.}, which is $\asymp f_T^2$. For \texttt{UCB1}, it decays at the extremely slow rate of $1/\sqrt{\log T}$. Algorithms that explore more reduce bias and increase regret. Our sharp characterization uncovers the algorithmic origin of bias. Heavier exploitation increases the correlation between the arm's sample mean and number of pull, which drives up bias.}

\ycedit{
\emph{Implications for post-bandit inference.} 
Our results reveal a sharp regime distinction in post-bandit inference. We use the standard error \(\mathrm{SE}\) under i.i.d. sampling as a benchmark for the relative magnitude of the adaptive bias. In particular, for \(n\) i.i.d. observations, \(\mathrm{SE}\asymp 1/\sqrt{n}\). Our main theorem implies that, under \texttt{UCB1}, adaptive sampling induces a bias of order \(\widetilde{\Theta}(\mathrm{SE})\) for a tied-optimal arm. The bias is therefore comparable in magnitude to the sampling uncertainty itself and can substantially distort statistical inference. In sharp contrast, for a suboptimal arm, the bias is only \(O(\mathrm{SE}^2)\), and for a uniquely optimal arm, it is even smaller. For the tied-optimal regime, where the bias is non-negligible relative to the standard error, we further establish a Berry--Esseen-type bound that quantifies the resulting distortion in hypothesis testing. Under Gaussian rewards, we show that the \(\widetilde{\Theta}(\mathrm{SE})\) downward bias leads to systematic underestimation and, consequently, an \(\asymp 1/\sqrt{\log n}\) inflation of the Type I error when testing against a lower alternative, with a corresponding deflation under an upper alternative. This distortion vanishes remarkably slowly. Interestingly, the leading-order \(1/\sqrt{\log n}\) effects cancel in two-sided tests. Our characterization further suggests a simple bias correction for mitigating these inferential distortions.
}

\emph{Novel empirical fluid approximation.} 
To pin down bias in a tractable way, we develop a novel empirical fluid approximation of $\bm N_{T}$. The key difficulty in the analysis of a fully adaptive bandit algorithm is that the number of pulls and sample means of all arms are coupled in a complex way throughout the entire sampling procedure. We exploit the structure of index algorithms, which can be viewed as approximately equalizing the indices of all arms. This leads first to a deterministic fluid approximation $\bm n_{T}$ of $\bm N_{T}$, obtained by replacing empirical means with their population counterparts $\bm\mu$ in the index-equalization equations. We then introduce a second-order empirical fluid approximation $\tilde {\bm n}_{T}$ of $\bm N_{T}$ by reintroducing, in the index-equalization equations, the sample means evaluated at the fluid allocation count $\bm n_T$. This approximation remains tractable because it is not coupled with the entire adaptive sample path, yet it retains the first-order dependence between number of pull and reward fluctuations. We show that, under regularity conditions, $\tilde n_{a,T}$ is a sufficiently accurate proxy for $N_a(T)$ to identify the leading covariance term that determines sample-mean bias. This empirical fluid framework is central to our bias characterization and may be useful for analyzing other post-bandit data beyond those collected by index algorithms.

\jledit{\subsection{Other related work}}
Inference after adaptive experiments has long been an important problem. Classic sequential analysis recognized that when the sample size is determined by the data, sample-mean-based inference can differ substantially from fixed-sample inference \citep{anscombe1952large}; subsequent work developed confidence intervals and bias corrections for estimators computed at data-dependent stopping times \citep{siegmund1978estimation,woodroofe1992estimation,lai2006bias}. Similar issues arise in simulation output analysis: \citet{glynn1990bias} derive bias for estimators based on the random number of replications completed within a fixed computational budget. More recent work studies inference under broader adaptive experiments, asking when ordinary inferential methods remain valid, when they fail, and which correction strategies are appropriate \citep{russo2016controlling,nie2018adaptively,johari2022always,bibaut2025demystifying}.


The widespread deployment of modern learning algorithms raises a further challenge. Abundant data are generated as a byproduct of algorithms designed for various operational purposes, yet their inferential consequences require algorithm-specific analysis and are not adequately understood. Multi-armed bandit (MAB) algorithms \citep{robbins1952some} are a canonical example. 
Modern applications include digital platforms, online marketplaces, and experimental systems \citep{li2010contextual,schwartz2017customer,scott2010modern}. A central objective in the bandit literature is regret minimization, and many canonical algorithms are designed for this purpose \citep{thompson1933likelihood,lai1985asymptotically,Auer2002,lattimore2020bandit}. As bandit algorithms are increasingly deployed, statistical inference with bandit-generated data has become an important and timely problem \citep{zhang2020inference,zhan2021off,bibaut2021post}.


Existing post-bandit inference studies mainly focus on bias and asymptotic normality of sample means. In terms of bias, discussions are mainly around its sign and general bound (\cite{nie2018adaptively,shin2019bias,shin2019sample}), or obtain exact analysis in simple settings with limited adaptivity (\cite{nie2018adaptively,zhang2020inference}). We provide a sharp leading-order characterization of bias under fully adaptive canonical bandit algorithms including \texttt{UCB1}. \citet{lai1982least} establish conditions under which stability of an adaptive regression design ensures asymptotic normality of the least-squares estimator. Building upon this, recent work on post-bandit inference distinguishes stable algorithms, such as \texttt{UCB} (\cite{kalvit2021closer,khamaru2024inference,han2024ucb}), from unstable algorithms such as \texttt{Thompson Sampling}, \texttt{$\epsilon$-greedy}, and \texttt{MOSS}-type methods (\cite{deshpande2018accurate,zhang2020inference,zhang2021statistical,praharaj2025instability}), and focuses on stabilization techniques (\cite{halder2505stable,sengupta2026design}). Our results refine this view by showing that even among stable algorithms, standardized sample-mean bias can decay as slowly as $1/\sqrt{\log T}$ under \texttt{UCB1} or much faster under more exploratory rules.

Another related stream studies debiasing methods for post-policy inference. A major approach is inverse propensity weighting and its variants \citep{horvitz1952generalization,robins1994estimation,dudik2011doubly,hadad2021confidence,zhan2024policy,ham2026design}. These methods require that the sampling probability of the arm at each time is positive, which is suitable for randomized algorithms such as \texttt{Thompson Sampling}, but is not satisfied by deterministic index algorithms. Thus classical propensity-score-based debiasing methods are not directly applicable for our setting. Our sharp bias formula complements this literature by suggesting algorithm-specific bias correction for stable index algorithms.

There is growing recent work that examines consequences of bandit algorithms' operational objective beyond expected reward. \citet{fan2021fragility} show that optimizing expected regret can produce heavy regret tails and study how additional exploration improves tail behavior. \citet{chen2026bandit} establish a tradeoff between worst-case regret and allocation variability, showing that minimax regret-optimal algorithms necessarily incur worst-case fluctuations in arm;s number of pull. \citet{simchi2023multi} characterize a Pareto frontier between cumulative regret and treatment-effect estimation error. The latter arises mainly due to low-regret algorithm's under-sampling behavior of suboptimal arms. Our results identify a different statistical consequence of regret minimization: it changes how strongly sampling decisions respond to reward noise, thereby changing the bias of ordinary estimators. In particular, under UCB1, tied-optimal arms exhibit standardized bias that vanishes only as \(1/\sqrt{\log T}\), despite receiving linearly many observations.\\



\subsection{Notation}
To simplify exposition, we sometimes use the classic Bachmann-Landau notation. 
{In particular, for two nonnegative sequences $\{a_T\}$ and $\{b_T\}$, we write $a_T = O(b_T)$ if there exist constants $C>0$ and $T_0$ such that $a_T \le C b_T$ for all $T \ge T_0$, and $a_T = \Omega(b_T)$ if $b_T = O(a_T)$. We write $a_T = \Theta(b_T)$ if $a_T = O(b_T)$ and $a_T = \Omega(b_T)$. Moreover, $a_T = o(b_T)$ means $a_T/b_T \to 0$ as $T\to\infty$, and $a_T = \omega(b_T)$ means $a_T/b_T \to \infty$ as $T\to\infty$. We also use $\tilde{O}(\cdot)$, $\tilde{\Omega}(\cdot)$, and $\tilde{\Theta}(\cdot)$ to suppress factors that are polylogarithmic in $T$.} All limits are taken w.r.t. $T$ unless otherwise specified. For any positive integer $L$, $[L]$ denotes $\{1, 2, \dots, L\}$, the set of positive integers up to $L.$ $\partial_x f(x, y)$ denote the partial derivative of a function respect to $x$.

\section{The MAB model}
\label{sec:model}
 We consider a stochastic multi-armed bandit with a fixed number of arms \(a\in[K]:=\{1,\ldots,K\}\).  Arm \(a\) is associated with an infinite \emph{i.i.d.} sequence of rewards $X_{a, 1}, X_{a, 2}, \dots$, each drawn independently from a reward distribution with mean $\mu_a$ and variance $\sigma_a^2\jledit{>0}$. The reward sequences are mutually independent across arms. \jledit{We assume that all arms' reward distributions are $\nu$-sub-Gaussian for some $\nu>0$. That is,
 \begin{align*}
     \mathbb E\left[e^{\lambda(X_{a,l}-\mu_a)}\right]\leq e^{\nu^2\lambda^2/2}, \ \ \text{for all }a\in[K],l\in\mathbb N, \text{ and }\lambda\in\mathbb R.
 \end{align*}}
 Let $\mu^{\star} = \max_{a \in [K]} \mu_a$ and $\mathcal{O} = \{a: \mu_a = \mu^{\star}\}$ denote the set of optimal arms. Let $\mathcal{S} = \{1, \dots, K\} \setminus \mathcal{O}$ be the set of suboptimal arms.  We allow \(|\mathcal O|>1\).  For \(a\notin\mathcal O\), define the fixed gap $\Delta_a = \mu^{\star} - \mu_a > 0$ as its sub-optimality gap.

A bandit algorithm is a sequential sampling rule. Let \(A_t\in[K]\) denote the arm pulled at time \(t\), and let $N_a(t)=\sum_{u=1}^t \mathbf 1\{A_u=a\}$ be the number of pull of arm $a$ up to time $t$. If \(A_t=a\), the observed reward is \(R_t=X_{a,N_a(t)}\). The algorithm is adaptive in the sense that \(A_t\) is chosen based only on the past history $\mathcal F_{t-1}=\sigma(A_1,R_1,\ldots,A_{t-1},R_{t-1}),$ possibly together with external randomness independent of the rewards.

Let $\bar X_a(s)=\frac1 {s}\sum_{\ell=1}^{s} X_{a,\ell}$ denote the running sample mean of arm \(a\) based on its first \(s\) observed \footnote{\jledit{We set $\bar X_a(0):=\mu_a$ for every arm $a$.}}. By the terminal horizon \(T\), the data collected from arm \(a\) are $\{X_{a,1},\ldots,X_{a,N_a(T)}\}.$ We consider post-bandit inference. Our main object of interest is the \emph{post-bandit sample mean bias} $\EE[\bar X_a(N_a(T)) - \mu_a]$, denoted by $\textrm{Bias}_{a, T}$. We also study the bias of the associated \emph{post-bandit \(Z\)-statistic}
\[
 \EE[Z_{a,T}]
    =
    \EE\left[\frac{\sqrt{N_a(T)}}{\sigma_a}\{\bar X_{a}(N_a(T))-\mu_a\}\right].
\]

\section{Index Algorithms}\label{sec:index-algo}
An important class of bandit algorithms is the class of \emph{index algorithms}, which make allocation decisions by comparing arm-specific indices. In this paper, we focus on index algorithms in which the index of an arm depends on its current sample mean and current number of pulls.

\begin{definition}[Index algorithm]\label{def:index-algo}
Fix a sequence of index functions \(I_t(\cdot,\cdot)\). The algorithm first pulls each arm once for initialization. For each time \(t\ge K+1\),  the algorithm computes index
\[ 
I_t\!\left(\bar X_a(N_a(t-1)),\,N_a(t-1)\right)
\]
for each arm \(a\in[K]\), and then pulls an arm with the largest index with ties broken according to a fixed rule.
\end{definition}

This formulation covers many commonly used bandit algorithms. {For example, the canonical \texttt{UCB1}  corresponds to $I_t(x,n)=x+\sqrt{\frac{2\log t}{n}}.$ Other examples include \texttt{KL-UCB} and \texttt{MOSS} (detailed expressions omitted). Note that there are other index algorithms, such as the Gittins index and \texttt{UCB-V}, that require additional state variables. There are also popular non-index (and non-deterministic) algorithms such  as \texttt{Thompson Sampling} and \texttt{Exp3}. These algorithms are beyond the index algorithms discussed here.} 




{We next introduce an important quantity associated with the index function, which we term \emph{effective exploration rate}. This quantity, as we show later, crucially determines the size of sample mean bias under the algorithm. }
\begin{definition}[Effective Exploration Rate]\label{def:exploration rate}
Assume the index function \(I_t(x,n)\) is differentiable. Define its local effective
exploration rate at \((x,n,t)\) by
\[
    \Gamma_t(x,n)
    :=
    -\,2n^{3/2}
    \frac{\partial_n I_t(x,n)}{\partial_x I_t(x,n)}.
\]
\end{definition}
The effective exploration rate can be understood through quantifying the change of sample size $n$ in response to a one-standard-error perturbation of the sample mean $x$ in order to keep the index unchanged. Intuitively, suppose a sample size change $\Delta n = n \delta$ offsets a typical sample mean perturbation of $\Delta x \asymp 1/{\sqrt{n}}$. We obtain from linearizing $I_t$ around $(x, n)$ that 
\[
    \partial_x I_t(x,n)\frac{1}{\sqrt n}
    +
    \partial_n I_t(x,n)n\delta
    \approx 0,
\]
leading to $\delta \approx 2/\Gamma_t(x, n)$. Thus larger \(\Gamma_t(x,n)\) makes the index-based allocation less responsive to sample mean fluctuations, thereby alleviating adaptivity-induced selection bias. Our main theorem formalizes this intuition by proving a \(1/\Gamma_{a,T}\) convergence rate of the standardized sample mean bias. For generalized \texttt{UCB} indices \(I_t(x,n)=x+f_t/\sqrt n\), the effective exploration rate is simply \(\Gamma_t=f_t\).


\subsection{Empirical Fluid Approximation}

Index algorithms induce a discrete-time dynamical system with state $(\bar{\mathbf X}_t,\mathbf N_t),$ where \(\bar{\mathbf X}_t\) and \(\mathbf N_t\) denote the vectors of sample means and number of pulls of the \(K\) arms at time \(t\). The exact dynamics of this system are generally difficult to analyze. We therefore introduce an approximation system based on two simplifications. First, we replace the discrete-time one-pull updates by a continuous-time allocation. Second, we ignore stochastic fluctuations and replace the sample means by the population mean vector $\bmu$. Suppose the index function is smooth and {exploration-encouraging}, in the sense that \(I_t(x,n)\) is increasing in \(t\) and decreasing in \(n\). Then under the above approximations, the dynamics are deterministic and    effectively characterized by equalization of the arm indices at each time $t$.  
\begin{definition}[Fluid approximation]\label{def: fluid system}
For each  \(t\), define $\bn_t= (n_{1,t},\ldots,n_{K,t})$ as the solution to
\begin{align}
I_t(\mu_a,n_{a,t})=\lambda_t,
\qquad a\in[K],
\qquad
\sum_{a=1}^K n_{a,t}=t,
\label{eq:n_aT}
\end{align}
where \(\lambda_t\) is the common index level, determined endogenously by the constraint on the total allocation. We call \(\mathbf n_t\) the {fluid approximation} to \(\mathbf N_t\). 
\end{definition}
{Later we show that the fluid approximation to $\bm N_t$ must exist and is unique under mild regularity conditions of the index function.} 

To capture the nontrivial {adaptivity in bandit data}, we need to characterize the correlation between the number of pulls and the empirical rewards. Therefore the deterministic fluid approximation~\eqref{eq:n_aT} is insufficient. This motivates a second-order approximation that reintroduces the stochastic fluctuations, {and more importantly, adaptivity,} into the system. Particularly, we replace \(\mu_a\) in \eqref{eq:n_aT} by the realized sample mean evaluated at the fluid state\footnote{When \(n_{a,t}\) is non-integer, \(\bar X_a(n_{a,t})\) is interpreted using a fixed rounding convention.}, \(\bar X_a(n_{a,t})\). 
\begin{definition}[Empirical fluid approximation]\label{def:empirical fluid system}
For each  \(t\), define $\tilde\bn_t= (\tilde n_{1,t},\ldots, \tilde n_{K,t})$ as the unique solution to
\begin{align}
I_t\bigl(\bar X_a(n_{a,t}),\tilde n_{a,t}\bigr)
=
\tilde\lambda_t,
\qquad a\in[K],
\qquad
\sum_{a=1}^K \tilde n_{a,t}=t,
\label{eq:second-order-fluid}
\end{align}
{where $\tilde\lambda_t$ is the endogenously determined common index level.}
We call $\bm{\tilde n}_t$ the empirical fluid approximation to $\bN_t$.
\end{definition}
{Similar to $\bm n_t$, later we show that the empirical fluid approximation to $\bm N_t$ must exist and is unique under mild regularity conditions of the index function.} 

Our analysis proceeds by showing that while \(\mathbf n_T\) provides a first-order approximation to \(\mathbf N_T\), the empirical fluid approximation \(\tilde{\mathbf n}_T\) is a sharper approximation. Importantly, the empirical fluid approximation retains the dependence between the number of pulls and the realized
rewards. We prove that this dependence captures the leading-order correlation
between \(\mathbf N_T\) and the sample mean, which is the key ingredient for characterizing bias.






\section{Main Results}
Our main results characterize, up to the leading term, the biases of post-bandit sample mean and $Z$-statistic for index algorithms under certain conditions. {We start by introducing and discussing the conditions that we impose on the index algorithms of study.}

\if false
\subsection{Conditions}\label{sec:conditions}

The first set of conditions concerns regularity of the index and the deterministic fluid allocation.

\begin{condition}\label{ass:index-regularity}
For every \(s\), the index function \(I_s(x,n)\) is continuously differentiable on
\(\mathbb R\times(0,\infty)\) and satisfies
\[
    \partial_x I_s(x,n)>0,
    \qquad
    \partial_n I_s(x,n)<0.
\]
Moreover, for every fixed \(x\in\mathbb R\),
\[
    I_s(x,n)\longrightarrow\infty
    \quad\text{as }n\downarrow0,
    \qquad
    I_s(x,n)\longrightarrow x
    \quad\text{as }n\uparrow\infty.
\]

For any fixed \(x_0\in\mathbb R\) and any sequences
\(n_s\uparrow\infty\), \(\varepsilon_s\downarrow0\), and
\(\delta_s\downarrow0\), define
\[
    \mathcal B_s
    =
    \left\{
        (x,n):
        |x-x_0|\le \varepsilon_s,\ 
        |n-n_s|\le \delta_s n_s
    \right\}.
\]
Then
\begin{align*}
    \sup_{(x,n)\in\mathcal B_s}
    \left|
        \frac{\partial_x I_s(x,n)}
             {\partial_x I_s(x_0,n_s)}
        -1
    \right|
    &\longrightarrow0,\\
    \sup_{(x,n)\in\mathcal B_s}
    \left|
        \frac{\partial_n I_s(x,n)}
             {\partial_n I_s(x_0,n_s)}
        -1
    \right|
    &\longrightarrow0
\end{align*}
as \(s\rightarrow\infty\).
\end{condition}

Condition~\ref{ass:index-regularity} states that the index increases locally with the empirical mean, decreases locally with the number of pull, and has stable first derivatives near the fluid allocation. The boundary conditions ensure that every arm receives a positive allocation in the fluid system.

\begin{lemma}\label{lem:existence_uniqueness_of_n}
Under Condition~\ref{ass:index-regularity}, the fluid approximation
\(\bm n_t\) in Definition~\ref{def: fluid system} and the empirical fluid approximation
\(\bm{\tilde n}_t\) in Definition~\ref{def:empirical fluid system} both exist and are unique.
\end{lemma}

\begin{proof}
Fix \(t\) and a vector \(x=(x_1,\ldots,x_K)\). For every
\(\lambda>\max_a x_a\), strict monotonicity and the boundary conditions imply that there is a unique \(n_a(\lambda)>0\) satisfying
\[
    I_t\{x_a,n_a(\lambda)\}=\lambda.
\]
Furthermore, \(n_a(\lambda)\) is continuous and strictly decreasing in \(\lambda\). Hence
\[
    \lambda\longmapsto\sum_{a=1}^K n_a(\lambda)
\]
is continuous and strictly decreasing, tends to infinity as
\(\lambda\downarrow\max_a x_a\), and tends to zero as
\(\lambda\uparrow\infty\). There is therefore a unique common level
\(\lambda\) for which \(\sum_a n_a(\lambda)=t\). Applying this argument first with \(x_a=\mu_a\) and then with the fixed-count empirical means proves the result.
\end{proof}

\yc{term consistency: do we call them fluid counts? }
\begin{condition}\label{ass:index_regularity2}
Every fluid count diverges:
\[
    \min_{a\in[K]} n_{a,T}\longrightarrow\infty.
\]
Moreover, for every optimal arm \(a\in\mathcal O\) and suboptimal arm
\(b\notin\mathcal O\),
\[
    \frac{n_{b,T}}{n_{a,T}}\longrightarrow0,
    \qquad
    \frac{\partial_n I_T(\mu_a,n_{a,T})}
         {\partial_n I_T(\mu_b,n_{b,T})}
    \longrightarrow0.
\]
\end{condition}

The first part of Condition~\ref{ass:index_regularity2} ensures that every fixed-count empirical mean concentrates. The count-ratio condition separates optimal and suboptimal allocation scales. The derivative-ratio condition is not needed for the intrinsic form of the general bias theorem; it is used only to reduce its intrinsic allocation-sensitivity coefficient to the equal-share expression
\[
    1-\frac{1}{|\mathcal O|}\mathbf 1\{a\in\mathcal O\}.
\]
For square-root UCB indices, the derivative-ratio condition follows directly from the count-ratio condition.

The next condition concerns stability of the empirical fluid approximation and the actual terminal allocation.

\begin{condition}\label{ass:concentration}
For every arm \(a\in[K]\), let
\[
    \Gamma_{a,T}
    =
    \Gamma_T(\mu_a,n_{a,T})
\]
be its induced effective exploration rate. Assume
\[
    \Gamma_{a,T}\longrightarrow\infty,
    \qquad
    \frac{\Gamma_{a,T}}{n_{a,T}}\longrightarrow0.
\]
For every arm \(a\), there exist arm-dependent constants \(s_a>1\) and
\[
    r_a>
    \max\left\{
        1,\,
        \frac{s_a}{2(s_a-1)}
    \right\}
\]
such that
\begin{align}
    \Gamma_{a,T}
    \norm{
        \frac{\tilde n_{a,T}-n_{a,T}}
             {n_{a,T}}
    }_{s_a}
    &=O(1),
    \label{eq:n-tilde-concentration}\\
    \Gamma_{a,T}
    \norm{
        \frac{N_a(T)-\tilde n_{a,T}}
             {n_{a,T}}
    }_{s_a}
    &=o(1),
    \label{eq:N-concentration}\\
    \sup_T
    \mathbb E
    \left[
        \left\{
            \frac{n_{a,T}}{N_a(T)}
        \right\}^{r_a}
    \right]
    &<\infty.
    \label{eq:N-no-starving}
\end{align}
\end{condition}

Condition~\ref{ass:concentration} is a verification condition rather than a primitive assumption on the index. Equation~\eqref{eq:n-tilde-concentration} controls the response of the empirical fluid allocation to fixed-count reward fluctuations, while
\eqref{eq:N-concentration} requires the sequential policy to track that empirical allocation on the bias-relevant scale. Together with
\(\Gamma_{a,T}\rightarrow\infty\), these two bounds imply
\[
    \frac{N_a(T)-n_{a,T}}{n_{a,T}}
    \longrightarrow0
\]
in \(L^{s_a}\), and hence in probability.

The inverse-count bound~\eqref{eq:N-no-starving} is a convenient sufficient no-starvation condition. To see how its exponent is used, choose
\(p_a\in(1,s_a)\) sufficiently close to \(s_a\) that
\[
    r_a>\frac{p_a}{2(p_a-1)}.
\]
Then choose \(u_a\) satisfying
\[
    \frac{p_a}{2(p_a-1)}<u_a<r_a.
\]
Count consistency and Hölder's inequality give
\[
    \mathbb E
    \left[
        \left\{
            \frac{n_{a,T}}{N_a(T)}
        \right\}^{u_a}
        \mathbf 1
        \left\{
            N_a(T)<\frac{n_{a,T}}{2}
        \right\}
    \right]
    \longrightarrow0.
\]
The sub-Gaussian random-clock lemma therefore yields
\[
    \sqrt{n_{a,T}}\,
    \norm{
        \bar X_a\{N_a(T)\}
        -
        \bar X_a(n_{a,T})
    }_{p_a/(p_a-1)}
    \longrightarrow0.
\]
Thus a separate assumption requiring random-clock convergence in every
\(L^r\) is unnecessary.

The empirical-fluid delta-method argument should likewise be carried out in
\(L^{p_a}\) for some \(p_a<s_a\), rather than at the endpoint \(s_a\). This avoids incorrectly inferring uniform integrability of the \(s_a\)-th powers from an \(L^{s_a}\) bound alone.

For PolyUCB, Condition~\ref{ass:concentration} is satisfied with
\(s_a=r_a=2\) for every arm. Classical UCB1 with rewards in \([0,1]\) also permits the common choice \(s_a=r_a=2\). More generally, logarithmic UCB under the sharp sub-Gaussian calibration
\[
    \rho>2\nu_{\max}^2
\]
requires arm-dependent exponents: one may use \(1<s_a<\rho/(2\nu_{\max}^2)\) on suboptimal arms and \(s_a=2\) on optimal arms. Here \(\nu_a^2\) denotes a valid sub-Gaussian proxy, not merely the reward variance.
\fi

\subsection{Conditions}\label{sec:conditions}
The first set of conditions are index regularity conditions. 

\begin{condition}\label{ass:index-regularity}
    The index function $I_s(x,n)$ is continuously differentiable, strictly increasing in $x$ and strictly decreasing in $n$. {Moreover, $I_s(x,n)\rightarrow\infty$ as $n\rightarrow0$, and $I_s(x,n)\rightarrow x$ as $n\rightarrow\infty$.}    
    In addition, for any constant $x_0 \jledit{\in\mathbb R}$, sequence of values $n_s\uparrow\infty$, $\varepsilon_s\downarrow0$ and $\delta_s\downarrow0$ as $s\rightarrow\infty$, define $\mathcal B_s=\{(x,n): |x-x_0|\le\varepsilon_s,\ |n-n_s|\le\delta_sn_s\}.$ Then we have
    \begin{align*}
    \sup_{(x,n)\in\mathcal B_s}\left|\frac{\partial_xI_s(x,n)}{\partial_xI_s(x_0,n_s)}-1\right|\to0, \ \ \sup_{(x,n)\in\mathcal B_s}\left|\frac{\partial_nI_s(x,n)}{\partial_nI_s(x_0,n_s)}-1\right|\to0, \quad \text{as } s\rightarrow\infty.
    \end{align*}
\end{condition}
{Condition~\ref{ass:index-regularity} is a regularity condition on the index function. Monotonicity in the sample mean and number of pull captures the basic structure of index algorithms: better empirical performance increases an arm’s index, while additional pulls reduce it. The boundary conditions ensure that when an arm has been sampled very little, its index becomes sufficiently large to encourage exploration, while heavily sampled arms receive diminishing exploration bonuses. The local derivative stability requirement guarantees that the index responds smoothly to small perturbations near the fluid allocation. These are all naturally conditions satisfied by many stable index algorithms including \texttt{UCB1}.

Condition~\ref{ass:index-regularity} immediately implies the existence and uniqueness of $\bm n_T$ and $\bm{\tilde n}_T$.
\begin{lemma}\label{lem:existence_uniqueness_of_n}
    Under Condition~\ref{ass:index-regularity}, the fluid approximation $\bm n_t$ in Definition~\ref{def: fluid system} and the empirical fluid approximation $\bm{\tilde n}_t$ in Definition~\ref{def:empirical fluid system} both exist and are unique.
\end{lemma}
We omit the proof.
}

\begin{condition}\label{ass:index_regularity2}
     For any optimal arm $i\in\mathcal O$ and suboptimal arm $j\in\mathcal S$, we have $n_{j, T} = \omega(1),\,\ \frac{n_{j,T}}{n_{i,T}}=o(1)$, 
    and  $\frac{\partial_n I_T(\mu_i, n_{i,T})}{\partial_n I_T(\mu_j, n_{j,T})} = o(1).$
\end{condition}
{ Condition~\ref{ass:index_regularity2} is the fluid scaling conditions. The assumption $n_{j,T}=\omega(1)$ is information-theoretically required to ensure the concentration of the system's state $(\bar\bx, \bN)$ around its fluid approximation $(\bmu, \bn)$. This is consistent with the Lai--Robbins regret lower bound. The remaining conditions reflect the allocation structure of low-regret algorithms: optimal arms are pulled a positive fraction of the time, while suboptimal arms are sampled mainly to support exploration. The derivative ratio requirement aligns with the intuition that the index function becomes increasingly flat in $n$ as $n$ grows, making it progressively harder to change the index through additional number of pulls. These assumptions are naturally satisfied by \texttt{UCB1}.

}


The next set of conditions are stability conditions.
\begin{condition}\label{ass:concentration}
    For each arm $a\in[K]$, define $\Gamma_{a,T}=\Gamma_T(\mu_a, n_{a,T})$ as the arm's induced effective exploration rate. We assume $\Gamma_{a,T}=\omega\!\left(1\right).$ There exists \jledit{(arm-specific) constant $s>1$ 
     such that} 
    \begin{align}
        &\Gamma_{a,T}\norm{\frac{\tilde n_{a,T}-n_{a,T}}{n_{a,T}}}_s=O(1),\label{eq:n-tilde-concentration}\\
        &\Gamma_{a,T}\norm{\frac{N_a(T)-\tilde n_{a,T}}{n_{a,T}}}_s=o(1).\label{eq:N-concentration}
    \end{align}
    \jledit{Moreover, for some (arm-specific) $r>\max\left\{\frac{s}{2(s-1)}, 1\right\}$,
    \begin{align}\label{eq:N-no-starving}
        \norm{\frac{n_{a,T}}{N_a(T)}}_r=O(1).
    \end{align}}
\end{condition}
Unlike the preceding two conditions, Condition~\ref{ass:concentration} is not imposed directly on the primitives of the model. Instead, it should be understood as a verification condition under which our main results hold. Its requirements can be viewed as strengthening of the stability conditions in the literature: $\frac{N_a(T)-n_{a,T}}{n_{a,T}} \overset{p}{\to} 0$ (see for example, \cite{praharaj2025instability}). First, $\Gamma_{a, T}$ is identified as the convergence rate of $N_a(T)$ on $n_{a, T}$. Together with Conditions~\ref{ass:index-regularity} and ~\ref{ass:index_regularity2}, $\Gamma_{a, T} = \omega(1)$ implies stability.  Second, Condition~\ref{ass:concentration} additionally requires $N_a(T)$ to concentrate around $\tilde n_{a,T}$ at a faster rate when compared to its concentration around $n_{a, T}$. In this sense, $\tilde n_{a,T}$ is a refined approximation of $N_a(T)$. \ycedit{Finally, this condition also additionally controls, through a moment bound, the rare event that $N_a(T)$ becomes overly small compared to $n_{a,T}$.}
Since Condition~\ref{ass:concentration} is a strengthening of the classic stability condition, it automatically excludes unstable algorithms such as \texttt{MOSS} in the presence of multiple optimal arms\footnote{\jledit{When first-order stability in Condition~\ref{ass:concentration} fails, randomness in the allocation may affect the leading distribution of the estimator, requiring an analysis that retains the joint behavior of number of pulls and observed rewards. This does not, by itself, preclude valid statistical inference, and developing bias characterizations and inferential guarantees in these regimes is a natural direction beyond the stable setting studied here.
}}. 
Verifying this condition for stable algorithms constitutes a key technical challenge in the bias analysis.

\ycedit{Although our main results concern index algorithms, the underlying approximation strategy has broader potential. Extending it beyond index algorithms requires an appropriate generalized fluid approximation, and correspondingly, empirical fluid approximation, and corresponding sensitivity and tracking bounds. We expect that the insights of this paper remains solid as long as stability conditions hold.
}



\subsection{A general bias formula}
{We give a sharp characterization of the leading term of bias under any index algorithm that satisfies the regularity conditions in Section \ref{sec:conditions}.}

\begin{theorem}\label{thm:bias}
Under Conditions~\ref{ass:index-regularity}--\ref{ass:concentration}, it is true that for any arm $a\in[K]$,
\begin{align*}
&\lim_{T\rightarrow\infty} \Gamma_{a,T}\sqrt{n_{a,T}}\,\E[(\bar X_{a}(N_a(T))-\mu_a)]
=-\jledit{2}\left(1-\frac{1}{|\mathcal O|}\mathbbm 1\{a\in\mathcal O\}\right)\sigma_a^2,\\
&\lim_{T\rightarrow\infty} \Gamma_{a,T}\E[Z_{a,T}]
=-\left(1-\frac{1}{|\mathcal O|}\mathbbm 1\{a\in\mathcal O\}\right)\sigma_a;
\end{align*} 
\end{theorem}
{Theorem~\ref{thm:bias} shows that \jledit{the sample mean bias for an arm is negative, consistent with \cite{nie2018adaptively}. This is due to the bandit algorithm's tendency to stop pulling an arm when it performs bad (so it remains bad), and pulling the arm more when it performs well (so it ultimately returns to its average position).} For any arm $a$ that is either sub-optimal, or one of the (more than two) tied-optimal arms, its sample mean bias after standardization $\asymp\frac{1}{\Gamma_{a,T}}$, which is the same rate as the bias of the $Z$-statistic. This rate vanishes to zero by Condition~\ref{ass:concentration}. However, it can vanish to zero very slowly if the index function's effective exploration rate grows slowly in time. A striking example is the canonical \texttt{UCB1} algorithm, whose effective exploration rate grows like $\sqrt{\log T}$. This implies that the bias of the standardized sample mean and $Z$-statistic vanish as slowly as $\frac{1}{\sqrt{\log T}}$, for all arms except the unique optimal arm. \jledit{Furthermore, a larger reward variance $\sigma_a^2$ also makes the sample mean bias larger in absolute value, since more reward variability drives larger fluctuations in the arm's number of pull.}
}

\ycedit{When arm $a$ is uniquely optimal, the algorithm-induced sample mean bias is order-of-magnitude smaller. This is in sharp contrast especially to the degenerate case of tied-optimal arms, in which the algorithm oscillates among the optimal arms and creates large correlation. Later in Section~\ref{subsec: sharper bias unique opt}, we provide a sharper characterization of the sample mean bias for a unique optimal arm for \texttt{UCB} algorithms.
}

\jledit{\begin{remark}
    Note that each arm’s normalized bias limit depends only on its own reward variance, its optimality status, and (if the arm is optimal) the number of optimal arms, and is otherwise independent of other arm's information. This decoupling arises because the leading bias is only determined by the covariance between the arm’s reward fluctuation and the fluctuation of its own number of pull. The latter can be further (approximately) decomposed into linear combination of reward fluctuations for each arm. Independence across reward sequences makes the cross-arm covariance terms vanish, leaving only the contribution from the arm’s own noise.
    The decoupling concerns the normalized bias coefficients; the normalizing quantities \(n_{a,T}\) and \(\Gamma_{a,T}\) can still depend on the entire bandit instance.
\end{remark}}



\subsection{The \texttt{UCB} case}\label{sec: ucb bias}
In this section, we specialize to one of the most popular bandit algorithms, $\texttt{UCB}$, and characterize its bias.
\jledit{Recall that the popular \texttt{UCB1} algorithm is an index algorithm defined through $I_t(x,n)=x+\sqrt{\frac{2\log t}{n}}$. Our result covers \texttt{UCB1} and its generalizations, where the index function is of the form
\begin{align*}
    I_t(x,n)=x+\frac{f_t}{\sqrt{n}}.
\end{align*}
We impose the following conditions on $f_t$.
\begin{assumption}\label{cond:f_t_UCB_regularity}
    For all sufficiently large $t$, $f_t$ is continuously differentiable and is nondecreasing. In addition,
    \begin{align}
        &\limsup_{t\rightarrow\infty}\frac{tf_t'}{f_t}<\frac{1}{2},\label{eq:f_t_not_too_fast}\\
        &\jledit{\beta:=}\liminf_{t\rightarrow\infty}\frac{f_t^2}{2\nu^2\log t}>1.\label{eq:f_t_not_too_slow}
    \end{align}
\end{assumption}
\eqref{eq:f_t_not_too_fast} in Assumption \ref{cond:f_t_UCB_regularity} implies that $f_t=o(\sqrt{t})$, hence the algorithm incurs a sublinear regret. We exclude algorithms with $f_t=\Omega(\sqrt{t})$, i.e., with linear regret (as required by Condition \ref{ass:index_regularity2}), and focus only on the bias property of algorithms that actively learns. \eqref{eq:f_t_not_too_slow} requires that the exploration function $f_t$ must grow faster than $\sqrt{2\nu^2\log t}$. This covers \texttt{UCB1} for Bernoulli bandits, where $f_t=\sqrt{2\log t}$ and $\nu=\frac{1}{2}$. 

\ycedit{It's worth mentioning that the lower-growth requirement \eqref{eq:f_t_not_too_slow} is not merely a technical condition, but reflects a critical exploration threshold. For $f_t=\sqrt{\rho\log t}$, the condition becomes $\rho>2\nu^2$, or $\rho>1/2$ when using the universal sub-Gaussian bound $\nu=1/2$ for rewards in $[0,1]$. This is consistent with the interesting logarithmic-polynomial phase transition of expected-regret guarantees over this reward class. In particular, the algorithm has logarithmic expected regret on every fixed instance when $\rho>1/2$, whereas, for $0<\rho<1/2$, \cite{salomon2013lower} show that for every $\epsilon\in(0,1-2\rho)$ there exists an instance whose expected regret grows faster than $T^{1-2\rho-\epsilon}$. At the extreme case $\rho=0$, the policy reduces to the greedy algorithm, which can permanently abandon the optimal arm and incur linear expected regret. By imposing \eqref{eq:f_t_not_too_slow}, we avoid the algorithms entering the regime of over-exploiting, so that they stay in the interesting regime from a regret-minimization perspective.
}

In the rest of the paper, we use \texttt{UCB}-f for short to denote \texttt{UCB}-type algorithms satisfying Assumption \ref{cond:f_t_UCB_regularity}.} 


\begin{theorem}\label{thm:ucb bias}
    Under the \jledit{\texttt{UCB}-f} algorithm, \jledit{Conditions \ref{ass:index-regularity}--\ref{ass:concentration} hold. Consequently,}
for any arm $a\in[K]$,
\begin{align*}
&\lim_{T\rightarrow\infty} f_T \sqrt{n_{a,T}}\,\E[\bar X_{a}(N_a(T))-\mu_a]
=-\jledit{2}\left(1-\frac{1}{|\mathcal O|}\mathbbm 1\{a\in\mathcal O\}\right)\sigma_a^2,\\
&\lim_{T\rightarrow\infty} f_T\E[Z_{a,T}]
=-\left(1-\frac{1}{|\mathcal O|}\mathbbm 1\{a\in\mathcal O\}\right)\sigma_a.
\end{align*}
\end{theorem}
{Theorem~\ref{thm:ucb bias} makes Theorem~\ref{thm:bias} explicit for \texttt{UCB}-f algorithms. 
\jledit{By Definition \ref{def:exploration rate} and straightforward calculation, the effective exploration rate of a \texttt{UCB}-f algorithm is always $f_T$. Since Conditions \ref{ass:index-regularity}--\ref{ass:concentration} hold under \texttt{UCB}-f, Theorem \ref{thm:bias} immediately applies with $\Gamma_{a,T}=f_T$, and}
the standardized sample-mean bias and the expected $Z$-statistic are both of order $1/f_T$ for any arm that is not uniquely optimal. For \texttt{UCB1}, $f_T \asymp \sqrt{\log T}$, so the bias decays only at the slow rate $1/\sqrt{\log T}$. 
\jledit{For the unique optimal arm, the standardized bias decays faster than $1/f_T$. Proposition \ref{prop:unique arm sharper bias} in the next subsection gives a sharper bound, which vanishes polynomially fast, at least with a square-root rate.}
Finally, the theorem shows that bias depends on reward variance: for arms with the same mean reward, the higher-variance arm has larger bias and greater distortion in the associated $Z$-statistic.}

{\paragraph{Regret-bias trade-off.} 
In many applications of bandit algorithms, the primary objective is to minimize expected (pseudo-)regret (hereafter termed as ``regret''), defined by $\text{Reg}_T=\sum_{a\notin\mathcal O}\mathbb E[N_a(T)]\Delta_a$. That is, regret measures the expected loss in total reward compared to always pulling the optimal arms (see \cite{lattimore2020bandit}). The next proposition characterizes regret for the algorithms considered in this paper. 
\begin{proposition}\label{prop:regret}
    Assume $\mathcal O\neq[K]
    $. Under Conditions~\ref{ass:index-regularity}--\ref{ass:concentration}, it is true that
    \begin{align*}
        \lim_{T\rightarrow\infty}\frac{\textup{Reg}_T}{\sum_{a\notin\mathcal O}n_{a,T}\Delta_a}=1.
    \end{align*}
    In particular, under the \texttt{UCB}-f algorithm,
    we have
    \begin{align*}
        \lim_{T\rightarrow\infty}\frac{\textup{Reg}_T}{f_T^2\sum_{a\notin\mathcal O}1/\Delta_a}=1.
    \end{align*}
\end{proposition}
\begin{proof}[Proof of Proposition~\ref{prop:regret}]
    Triangle inequality $|N_a(T)-n_{a,T}|\leq|N_a(T)-\tilde n_{a,T}|+|\tilde n_{a,T}-n_{a,T}|$ and Condition~\ref{ass:concentration} imply that $\frac{\mathbb E[N_a(T)]}{n_{a,T}}\rightarrow 1$ for each suboptimal arm $a$.
    The first result hence follows. For the second result, note that under \texttt{UCB}-f, $n_{a,T}\sim\frac{f_T^2}{\Delta_a^2}$.
\end{proof}
For bandit algorithms that are deployed to achieve low regret, how much bias must be induced? Combining Theorem \ref{thm:ucb bias} and Proposition \ref{prop:regret} gives a sharp characterization of the regret-bias trade-off under \texttt{UCB}-f algorithm.

\begin{corollary}\label{cor:regret-stability-tradeoff}
    Assume $\mathcal O\neq[K]
    $. Under the \texttt{UCB}-f algorithm, 
    it is true that
    \begin{align*}
        \lim_{T\rightarrow\infty}\mathbb E[Z_{a,T}]\cdot\sqrt{\textup{Reg}_T}=\kappa_a
    \end{align*}
    for any arm that is not uniquely optimal, where $\kappa_a<0$ is a constant independent of the algorithm.
\end{corollary}

Corollary~\ref{cor:regret-stability-tradeoff} sheds light on algorithm design within the \texttt{UCB}-f algorithm class. At one extreme, \texttt{UCB1} (with proper constant $\rho$) has an optimal regret rate that grows like $\log T$, but the induced standardized bias vanishes slowly like $\frac{1}{\sqrt{\log T}}$ for a non-unique optimal arm whose sample size is as large as $\asymp T$. A faster growing rate of $f_t$ explores more, decreases this arm's standardized bias to rate $\frac{1}{f_T}$, while increases regret at rate $f_T^2$. At the extreme, when $f_t$ grows at a near-square-root rate, the arm's standardized bias reduces to nearly $\frac{1}{\sqrt{T}}$ rate, with close-to-linear regret.
}

\ycedit{\section{Implication for statistical inference}}\label{sec: ucb inference}
This section translates the preceding sharp asymptotic results into implications for statistical inference. We focus on the generalized \texttt{UCB-f} family, for which the statistical message is particularly clean.

\subsection{The scale of bias}

While Theorem~\ref{thm:bias} provides a unified characterization of the leading bias of the raw and standardized sample mean estimators, its inferential implications differ sharply between optimal and suboptimal arms. To compare, we take the standard error of sample mean estimator under i.i.d. sampling as benchmark. In particular, if $X_1, \dots, X_n$ are $i.i.d.$ with mean $\mu$ and variance $\sigma^2 > 0$, then  
\[
\operatorname{SE} = \operatorname{sd}\left(\bar X_n\right) = \frac{\sigma}{\sqrt{n}} \as \frac{1}{\sqrt{\textrm{sample size}}}.
\]
Next, we take a look at the size of the asymptotic bias relative to $\operatorname{SE}$ for suboptimal and optimal arms, respectively.

\subsubsection{Suboptimal arms}
For every suboptimal arm $j \in \mathcal{S}$, Theorem~\ref{thm:ucb bias} shows that the leading term in sample mean bias is of order \jledit{$f_T^{-1}n_{j, T}^{-1/2}$}. By Proposition~\ref{prop:regret}, we know that $n_{j, T} \as f_T^2$. Also, Condition~\ref{ass:concentration} ensures $N_{j}(T) \sim n_{j, T}$. Considering that $N_{j}(T)$ is the sample size of arm $j$, we thus conclude that for suboptimal arms 
\[
\textrm{Bias} \as -\frac{1}{\textrm{sample size}} \as -(\operatorname{SE})^2
\] 
Thus, the bias of a suboptimal arm is quadratic in the typical size of the standard error based on the same number of observations. This scaling holds for every generalized \texttt{UCB-f} algorithms, with varying exploration function $f$ satisfying Assumption~\ref{cond:f_t_UCB_regularity}. Slower-growing $f$ leaves a suboptimal arm with fewer observations and hence a larger standard error, but its adaptive bias remains order-of-magnitude smaller. The principal inferential limitation for a suboptimal arm \jledit{after a bandit experiment} is therefore its small sample size, rather than the adaptivity-induced bias.

\subsubsection{Unique optimal arm}\label{subsec: sharper bias unique opt}
Suppose $|\mathcal{O}| = 1$ and arm $i$ is the single optimal arm. The limit in Theorem~\ref{thm:ucb bias} is $1 - \frac{1}{|\mathcal{O}|} = 0$. Therefore, the theorem does not provide a precise scaling rate for bias. The next proposition  complements Theorem~\ref{thm:ucb bias} with a sharper bias characterization for the unique optimal arm.
\begin{proposition}\label{prop:unique arm sharper bias}
Suppose $\mathcal O =\{i\}$ and Assumption~\ref{cond:f_t_UCB_regularity} holds. Then, for every fixed $\eta\in(1,\beta)$, as $T\to\infty$,
\[
\left|
\E\left[
\bar X_i\{N_i(T)\}-\mu_i
\right]
\right|
=
O\left(
\frac{f_T}{T^{3/2}}
+
T^{-\eta}
\right), \qquad\qquad\left|\E[Z_{i,T}]\right|
=
O\left(\frac{f_T}{T}\right).
\]
\end{proposition}
Recall that $f_T = o(\sqrt{T})$ under Assumption~\ref{cond:f_t_UCB_regularity}. \jledit{Also, the number of pulls of the optimal arm under an algorithm is always $\asymp T$, provided that Assumption \ref{cond:f_t_UCB_regularity} holds.} Proposition \ref{prop:unique arm sharper bias} thus implies, \jledit{for the unique optimal arm,}
\[
\textrm{Bias} = o\left(\frac{1}{\textrm{sample size}}\right) = o\left( \operatorname{SE}^2\right).
\]
\jledit{This rate vanishes even relative to the squared standard error, and is true uniformly for generalized \texttt{UCB-f} algorithms.}
Consequently, one can effectively treat the sample mean of an optimal arm as unbiased, provided that it is uniquely optimal.

The rates in Proposition~\ref{prop:unique arm sharper bias} are likely still loose. However, further improving these rates and achieving the tight characterization is challenging and less rewarding, considering that they already decay extremely fast. We do not pursue a tight characterization here in this  work.

\subsubsection{\jledit{Tied-optimal} arms}

Suppose $|\mathcal{O}|>1$. Then, for every tied-optimal arm $i \in \mathcal{O}$, Theorem~\ref{thm:ucb bias} shows that the leading term in its sample mean bias is of order $T^{-1/2} f_T^{-1}$. For optimal arms, the number of pulls under an algorithm is always in proportion to $T$, regardless of $f$ provided that Assumption~\ref{cond:f_t_UCB_regularity} holds. Thus we conclude that for each tied-optimal arm, 
\jledit{\[
\textrm{Bias} \as -\frac{1}{f_\textrm{sample size}\sqrt{\textrm{sample size}}} =-\omega(\textrm{SE}^2).
\]}
\ycedit{Thus expressed in terms of an arm's own sample size, or equivalently the sample-size dependent standard error \(\operatorname{SE}\), tied-optimal-arm bias decays much more slowly than that of a suboptimal arm.}
\jledit{
In particular, the slower $f_T$ grows, the larger bias is. For practically relevant bandit algorithms pursuing low regret, $f_T$ typically grows very slowly. For example, \texttt{UCB1} sets $f_T\asymp\sqrt{\log T}$ and achieves rate-optimal $O(\log T)$ regret. This leads to a large bias. For a tied-optimal arm, its leading sample mean bias satisfies}
\jledit{\[\textrm{Bias}=\widetilde\Theta\left(\operatorname{SE}
\right),\]}
\jledit{which is comparable to the standard error even with a fairly large sample size.}
\ycedit{Since the conventional confidence interval has a coverage size of $O(\operatorname{SE})$, the bias} can therefore remain statistically non-negligible at practically relevant horizons and affect downstream hypothesis tests, confidence intervals, and other inferential tasks. We discuss more in the next section.

\ycedit{
\begin{remark}
    The preceding results concern fixed bandit instances. A natural extension for $\texttt{UCB}$-$f$ is to consider instances indexed by $T$, with horizon-dependent means. Suppose a fixed group of $m\ge2$ arms has suboptimality gaps $o(f_T/\sqrt T)$, while the remaining gaps stay bounded away from zero. The fluid equations then assign asymptotically equal shares $1/m$ to these arms. We conjecture that their leading biases follow the tied-optimal formulas in Theorem~\ref{thm:ucb bias}, with $m$ replacing $|\mathcal O|$, even if the optimal arm is unique at each horizon. When the gaps are instead of order $f_T/\sqrt T$, the limiting fluid shares may differ, suggesting different bias constants. Establishing these extensions requires analysis beyond our fixed-instance results and is left for future work.
\end{remark}
}



\subsection{Refined inference for multiple optimal arms}\label{sec:implication_inference}
\ycedit{Discussion in last section reveals that tied-optimal arms suffer from severe bias, with potentially substantial consequences for inference. We now focus on this regime and establish refined inferential results beyond the bias characterization. In this section, we assume Gaussian rewards.

\paragraph{The i.i.d. case.} 
As a benchmark, consider first the standard $Z$-test. Let $X_1,\ldots,X_n$ be i.i.d.\ $\mathcal N(\mu,\sigma^2)$ observations with known variance $\sigma^2>0$. To test $H_0:\mu=\mu_0$, define the $Z$-statistic \(Z_n =\frac{\sqrt n(\overline X_n-\mu_0)}{\sigma}.\) Let $z_\alpha =\Phi^{-1}(1-\alpha),$ where $\Phi$ is the standard normal CDF. The left-tailed, right-tailed, and two-sided tests reject when $Z_n<-z_\alpha$, $Z_n>z_\alpha$, and $|Z_n|>z_{\alpha/2}$, respectively, corresponding to the alternatives $\mu<\mu_0$, $\mu>\mu_0$, and $\mu\ne\mu_0$.
Conventional statistical guarantees are summarized in the next proposition.
\begin{proposition}
\label{prop:iid-z-test}
Under $H_0$, for every integer $n\ge1$ and $\alpha\in(0,1)$, $Z_n\sim\mathcal N(0,1)$ and
\[
\mathbb P_{0}(Z_n<-z_\alpha) = \mathbb P_{0}(Z_n>z_\alpha) = \mathbb P_{0}(|Z_n|>z_{\alpha/2}) =\alpha,
\]
where $\PP_0$ denotes the probability under the null hypothesis.
\end{proposition}
}



\jledit{
\paragraph{The bandit data case.} What happens if we conduct the same normal $Z$-tests as above when the data is actually collected by a bandit algorithm? \jledit{Similar to the i.i.d. case,} \ycedit{consider testing a hypothesized value $\mu_0$ for the true mean $\mu_i$ of arm $i$. The null hypothesis is \(H_0:\mu_i=\mu_0,\) and the left-tailed, right-tailed, and two-sided alternatives are $H_1:\mu_i<\mu_0$, $H_1:\mu_i>\mu_0$, and $H_1:\mu_i\neq\mu_0$, respectively. Denote the $Z$-statistic \(Z_{i, T}(\iota) =\frac{\sqrt{N_i(T)}}{\sigma_i} \left\{\bar X_i(N_i(T))-\iota\right\}.\) Under $H_0$, this coincides with the previously defined statistic $Z_{i,T}=Z_{i,T}(\mu_0)$. A naive $Z$-test applies the conventional Gaussian critical values as in the i.i.d. case, rejecting $H_0$ when $Z_{i,T}<-z_\alpha$, $Z_{i,T}>z_\alpha$, or $|Z_{i,T}|>z_{\alpha/2}$, respectively.}
We next show that, in this case, the Type-I error rate may deviate from the nominal significance level $\alpha$ by a non-negligible amount. We start with the following Berry-Esseen-type bound.}


\begin{theorem}\label{thm:CDF expansion}
    Assume Gaussian rewards and that Assumption~\ref{cond:f_t_UCB_regularity} holds. Under $H_0$, for every $i\in\mathcal O$,
    \[
    \sup_{z \in \mathbb{R}}\left|\mathbb{P}_0\left(Z_{i, T} + \frac{\sigma_i}{f_T}\left(1 - \frac{1}{|\mathcal{O}|}\right)\le z \right) - \Phi(z) \right| = o(f_T^{-1}).
    \]
\end{theorem}

\begin{corollary}\label{cor:CDF expansion}
    Under the same conditions as Theorem~\ref{thm:CDF expansion} and under $H_0$, for every $i\in\mathcal O$,
    \[
    \sup_{z \in \mathbb{R}}\left|\mathbb{P}_0\left(Z_{i, T} \le z\right) - \Phi(z) -\frac{\sigma_i}{f_T}\left(1 - \frac{1}{|\mathcal{O}|}\right)\phi(z)\right| = o(f_T^{-1}),
    \]
    where $\Phi$ and $\phi$ denote the CDF and PDF of the standard normal distribution, respectively.
\end{corollary}


\jledit{Theorem \ref{thm:CDF expansion} and Corollary \ref{cor:CDF expansion} suggest that, for a tied-optimal arm with $N$ post-bandit data, the approximation error of the $Z$-statistic using standard normal is $\Theta(1/f_N)$. Recall that this should be zero when the rewards are non-adaptively sampled. Theorem \ref{thm:CDF expansion} justifies a first-order correction: when \(\sigma_i\) and \(|\mathcal O|\) are known, the corrected statistic \(Z_{i,T}+(\sigma_i/f_T)(1-1/|\mathcal O|)\) has normal approximation error \(o(f_T^{-1})\). A practical implementation would generally require estimating these quantities, which we will not further discuss in the current paper. 
}

Theorem~\ref{thm:CDF expansion} has direct implications for hypothesis testing and confidence interval construction in post-bandit inference.  Under $H_0$, $Z_{i,T}(\mu_0)=Z_{i,T}$, and Corollary~\ref{cor:CDF expansion}  determines \jledit{the Type-I error rate of the $Z$-test} 
up to its leading correction order.

\begin{corollary}\label{cor:hypothesis_test_error}
\jledit{Under the same conditions as Theorem~\ref{thm:CDF expansion}} and under $H_0$,
\begin{align*}
\mathbb P_0\left(Z_{i,T} <-z_\alpha\right) &=
\alpha+\frac{\sigma_i \phi(z_\alpha)}{f_T}
\left(1-\frac{1}{|\mathcal O|}\right)+o(f_T^{-1}),\\
\mathbb P_0\left(Z_{i,T} >z_\alpha\right)
&= \alpha - \frac{\sigma_i \phi(z_\alpha)}{f_T}
\left(1-\frac{1}{|\mathcal O|}\right)+o(f_T^{-1}).
\end{align*}
Yet
\[
\mathbb P_0\left(|Z_{i,T}|>z_{\alpha/2}\right)
= \alpha+o(f_T^{-1}).
\]
\end{corollary}
Thus, for tied optimal arms, the left-tailed test rejects the null more often than the nominal level $\alpha$, whereas the right-tailed test rejects less often. Although these distortions vanish asymptotically, they decay only at rate $f_T^{-1}$ and can therefore remain substantial when $f_T$ grows slowly, as it does for many commonly used regret-oriented bandit algorithms. \jledit{For example, consider the 5\% significance-level left-tailed $Z$-test using observations on the tested arm, 
collected by \texttt{UCB1}. When the tested arm is one of the three tied-optimal arms, and the bandit experiment's horizon is $1000$, 
with $\sigma_i=0.99$, 
the Type-I error rate \jledit{is 6.8\% by the leading terms in Corollary \ref{cor:hypothesis_test_error}.} 
This distortion is 36\% of the nominal level $\alpha=5\%$, and is non-negligibly large. The distortion vanishes very slowly as we increase the sample size. When the bandit experiment's horizon is $10000$, ten times the original size, the Type-I error's distortion is still a significant $32\%$. What's worse, the distortion is bigger in high-stake tests with a smaller tolerance for type-I error. For example, in the previous test, when the significance level is $\alpha=1\%$ and horizon length is $1000$, type-I error's distortion becomes an astonishing 47\%.} Note that the distortion's magnitude is also proportional to $\sigma_i$. \jledit{Doubling $\sigma_i$ would double the distortion by Corollary \ref{cor:hypothesis_test_error}.} 
Thus, even though $Z_{i,T}$ is standardized, arms with greater reward variability experience larger errors in rejection probability and confidence interval coverage. Such directional miscalibration matters whenever a one-sided statistical decision has operational consequences. For example, in a multi-arm adaptive clinical trial in which a test is used to declare efficacy, stop a treatment for futility, or flag a safety concern.

Interestingly, for symmetric two-sided inference, the first-order errors cancel: the increase in the left-tail probability offsets the decrease in the right-tail probability. Consequently, the usual two-sided test has \jledit{Type-I error rate}
$\alpha+o(f_T^{-1})$, and the corresponding symmetric confidence interval has coverage probability $1-\alpha+o(f_T^{-1})$. One natural follow-up question is to pin down the leading term in two-sided test's error. This essentially demands a CDF expansion of the $Z$-statistics beyond the leading term, and we leave rigorous study for future research.

\ycedit{
Corollary~\ref{cor:hypothesis_test_error} suggests correcting the one-sided tests simply by considering the shifted standard $Z$-test. In particular, reject when
\[
\begin{aligned}
Z_{i,T}&<-z_\alpha-\frac{\sigma_i}{f_T}\left(1-\frac{1}{|\mathcal O|}\right)
&&\text{(left-tailed)},\\
Z_{i,T}&>z_\alpha-\frac{\sigma_i}{f_T}\left(1-\frac{1}{|\mathcal O|}\right)
&&\text{(right-tailed)}.
\end{aligned}
\]
When $\sigma_i$ and $|\mathcal O|$ are known, Theorem~\ref{thm:CDF expansion} guarantees Type-I error $\alpha+o(f_T^{-1})$ under its stated conditions. The usual two-sided test already achieves this accuracy because its leading tail errors cancel, so no correction is needed at this order.
}

\ycedit{
The proof of Theorem~\ref{thm:CDF expansion} is deferred to Appendix~\ref{app:CDF-expansion}. It exploits specific Gaussian properties and is not directly generalizable to non-Gaussian distributions. Extending the first-order expansion beyond Gaussian rewards requires a different smoothing argument, including a careful treatment of possible lattice effects. We leave it for future work.
}


\section{Proof}
\subsection{Proof of Theorem~\ref{thm:bias}}
{We first provide a general expression for post-bandit sample mean bias. This expression connects bias to the covariance between the sample mean and the number of pulls.}
\begin{lemma}\label{lem:identities}
For every arm and every deterministic $l>0$,
\begin{equation}
\E[\barX_a(N_a(T)) - \mu_a]=-\frac{1}{l}\E[(N_a(T)-l)(\barX_a(N_a(T)) - \mu_a)].
\label{eq:bias-identity}
\end{equation}
Moreover, with 
$\jledit{\Lambda}_l(k)=\sqrt{k}-k/\sqrt{l}$,
\begin{equation}
\E[\sqrt{N_a(T)}\,(\barX_a(N_a(T)) - \mu_a)]=\E[\jledit{\Lambda}_l(N_a(T))(\barX_a(N_a(T)) - \mu_a)].
\label{eq:Z-statistic-identity}
\end{equation}
\end{lemma}
Next, we give proofs of Lemma~\ref{lem:identities} and Theorem~\ref{thm:bias}. Proofs of Lemmas~\ref{lem:main-bias} and~\ref{lem:remainders} are provided in Section~\ref{sec:proof of lemmas for thm1}.
\begin{proof}[Proof of Lemma~\ref{lem:identities}]
We first establish an identity $\EE[N_a(T)\barX_a(N_a(T))] = \mu_a\EE[N_a(T)]$ for all $a$.
\begin{equation*}
    \EE[N_a(T)\barX_a(N_a(T))] = \EE\left[\sum_{t= 1}^{N_a(T)} X_{a, t}\right] = \EE\left[\sum_{t = 1}^T \mathbbm{1}(A_t = a) 
    \EE\left[X_{a, N_a(t)} | \mathcal{F}_{t-1}\right]\right] = \mu_a \EE[N_a(T)].
\end{equation*}
Then for any deterministic $l>0$,
\begin{align*}
    \E[(N_a(T)-l)(\barX_a(N_a(T)) - \mu_a)] & = \E\left[N_a(T)\barX_a(N_a(T))\right] + l\mu_a - \EE\left[N_a(T)\right] \mu_a -l\E\left[\barX_a(N_a(T))\right] \\
    &= -l\E\left[\barX_a(N_a(T))-\mu_a\right].
\end{align*}
Equation \eqref{eq:bias-identity} follows by rearranging terms. Moreover, with $\Lambda_l(k)=\sqrt{k}-k/\sqrt{l}$, equation \eqref{eq:Z-statistic-identity} follows directly from the identity established above.
\end{proof}

{The next lemma gives the key leading term of bias using the empirical fluid approximation.}
\begin{lemma}\label{lem:main-bias}
Suppose that Conditions~\ref{ass:index-regularity}--\ref{ass:concentration} hold for all arms. Then, for each arm $a$,
\begin{equation}
{\frac{\Gamma_{a,T}}{\sqrt{n_{a,T}}}}\E[(\tilde n_{a,T}-n_{a,T})(\barX_a(n_{a,T}) - \mu_a)]
=\jledit{2}\left(1-{\frac{1}{|\mathcal O|}\mathbbm 1\{a\in\mathcal O\}}\right)\sigma_a^2+o(1).
\label{eq:bias}
\end{equation}
\end{lemma}

{The next lemma controls the residual terms.}
\begin{lemma}\label{lem:remainders}
Under Condition~\ref{ass:concentration}, \jledit{suppose $n_{a,T}\rightarrow\infty
$. Then} \lsedit{for any $\frac{s}{s-1}<q<2r$,}
\begin{align}
    &\norm{\sqrt{n_{a,T}}(\barX_a(N_a(T)) - \mu_a)}_q = O(1),\nonumber\\
&{\frac{\Gamma_{a,T}}{\sqrt{n_{a,T}}}}\E[|N_a(T) - \tilde n_{a,T}|\,|\barX_a(N_a(T)) - \mu_a|] = o(1),
\label{eq:remainder term 1}\\
&{\frac{\Gamma_{a,T}}{\sqrt{n_{a,T}}}}\E[|\tilde n_{a,T} - n_{a,T}|\,|\barX_a(N_a(T))-\barX_a(n_{a,T})|] = o(1).
\label{eq:remainder term 2}
\end{align}
\lsedit{
Furthermore,
\begin{align}
    \frac{\Gamma_{a,T}}{\sqrt{n_{a,T}}}\EE\left[\left(\sqrt{N_a(T)} - \sqrt{n_{a,T}}\right)^2\,\left|\barX_a(N_a(T)) - \mu_a \right|\right] = o(1)
    \label{eq:remainder term 3}
\end{align}
}
\end{lemma}

{We are now ready to prove Theorem \ref{thm:bias}.}
\begin{proof}[\bf{Proof of Theorem \ref{thm:bias}}]
{We first prove the sample mean bias.} By Lemma~\ref{lem:identities}, {taking $l=n_{a,T}$, we have}
\begin{equation*}
    \E[\barX_a(N_a(T)) - \mu_a]=-\frac{1}{n_{a,T}}\E[(N_a(T)-n_{a,T})(\barX_a(N_a(T)) - \mu_a)].
\end{equation*}
{It is also true that} 
\begin{align*}
   \E[(N_a(T)-n_{a,T})(\barX_a(N_a(T)) - \mu_a)] & =  \E[(\tilde n_{a,T}-n_{a,T})(\barX_a(n_{a,T})- \mu_a)] \\ 
   & \ \ + \E[(N_a(T) -\tilde n_{a,T})(\barX_a(N_a(T))- \mu_a)] \\
   & \ \ + \E[(\tilde n_{a,T}-n_{a,T})(\barX_a(N_a(T))- \barX_a(n_{a,T}))].
\end{align*}
Combining Lemma~\ref{lem:main-bias} with equations \eqref{eq:remainder term 1} and \eqref{eq:remainder term 2} in Lemma~\ref{lem:remainders} yields the sample mean bias.

{We now prove the $Z$-statistic bias.} Again, by Lemma~\ref{lem:identities},
\begin{equation*}
\E[\sqrt{N_a(T)}\,(\barX_a(N_a(T)) - \mu_a)]=\E[{\Lambda}_l(N_a(T))(\barX_a(N_a(T)) - \mu_a)],
\end{equation*}
where 
\begin{align*}
    \Lambda_l(N_a(T))=-\frac{N_a(T)-n_{a,T}}{2\sqrt{n_{a,T}}}-\frac{(\sqrt{N_a(T)}-\sqrt {n_{a,T}})^2}{2\sqrt{n_{a,T}}}.
\end{align*}
by the definition of $\Lambda_l(k)$ in Lemma~\ref{lem:identities}, and taking $l=n_{a,T}, k=N_a(T)$.
\lsedit{
The first term is handled by the same decomposition as in the proof of the sample mean bias. After multiplication by $\Gamma_{a,T}$, the leading term is obtained from Lemma~\ref{lem:main-bias}, while the corresponding remainder terms vanish by \eqref{eq:remainder term 3} in Lemma~\ref{lem:remainders}. The additional factor $1/2$ cancels the factor $2$ in Lemma~\ref{lem:main-bias}, which gives the $Z$-statistic bias.
}
\end{proof}

\subsection{Proof of Theorem~\ref{thm:ucb bias}}
{
Theorem~\ref{thm:ucb bias} rests on the fact that the empirical fluid
approximation $\tilde n_{a,T}$ approximates the actual number of pull
$N_a(T)$ more accurately than the deterministic fluid approximation
$n_{a,T}$ does. The following three lemmas supply the key concentration bounds for establishing this fact. In what follows, we reserve index $i$ for optimal arms, namely $i \in \mathcal{O}$, and index $j$ for suboptimal arms, namely $j \in \mathcal{S}$, unless otherwise specified. Recall that $\beta  =\liminf_{t\rightarrow\infty}\frac{f_t^2}{2\nu^2\log t}.$

\begin{lemma}
\label{lem:empirical-fluid-concentration}
For every arm $a\in[K]$ and every fixed exponent $p$ satisfying
\[
1<p<\beta \quad\text{if } a\in\mathcal S,
\qquad
1<p<\infty \quad\text{if } a\in\mathcal O,
\]
there exist constants $C, T_0 > 0$ that may depend on the fixed bandit instance, the exploration function $f$, and the exponent $p$, but not on $T$, such that,
for every $T\geq T_0$,
\[
f_T
\left\|
\frac{\tilde n_{a,T}-n_{a,T}}{n_{a,T}}
\right\|_p
\leq C.
\]
\end{lemma}

Lemma~\ref{lem:empirical-fluid-concentration} identifies the scale of the gap between the empirical and deterministic fluid allocations \(\tilde n_{a,T} = n_{a,T}\left(1+O_{L^p}\!\left(f_T^{-1}\right)\right). \) This is a static fluid-approximation result and does not yet concern the actual number of pull $N_a(T)$ under the algorithm. The restriction \(p<\beta\) cannot generally be relaxed, as even for a two-arm Gaussian instance with logarithmic exploration, the rare event that a suboptimal arm's empirical mean exceeds that of an optimal arm would cause the normalized moments to diverge for every \(p>\beta\).

Next, we state concentration lemmas regarding $N$.


\begin{lemma}
\label{lem: ultra tail}
Fix $\gamma\in(0,\beta)$. There exist constants $C\ge1$ and
$T_0<\infty$, depending only on $\gamma$, the fixed bandit instance,
and the algorithm, such that, for every integer $T\ge T_0$ and 
every real \(C \left(1+ f_T\{\log(e+f_T)\}^{-2}\right)\le x\le T,
\)
we have
\begin{align}
\mathbb P\Biggl(
&\max_{j\in\mathcal S}|N_j(T)-\tilde n_{j,T}|>x
\quad\text{or}\quad
\max_{i\in\mathcal O}\frac{f_T^2}{T}
|N_i(T)-\tilde n_{i,T}|>x
\Biggr)
\le Cx^{-\gamma}.
\label{eq:ultra tail}
\end{align}
\end{lemma}

The factor $f_T^2/T$ places the optimal and suboptimal deviations on the same relative scale. Indeed, $n_{i, T} \asymp T$ for $i\in\mathcal O$, whereas $n_{j, T}\asymp f_T^2$ for $j\in\mathcal S$. Thus, a common $x$ can be read as a common relative deviation (relative to the respective fluid approximations)
of either optimal or suboptimal arms' number of pull. The bound is uniform for 
$x\ge \ycedit{C(1 + f_T \{\log(e + f_T)\}^{-2}) },$ this threshold is $o(f_T).$ For a suboptimal arm $j$, this becomes $o(\sqrt{n_{j, T}})$, a deliberately chosen scale for the needed delicate characterization of $N_{j}(T) - \tilde{n}_{j, T}$, which is sharper than the standard CLT-type fluctuation near $n_{j, T}$, typically on the order of $O(\sqrt{n_{j, T}})$. On the other hand, for an optimal arm $i$, the deviation of $|N_{i}(T) - \tilde{n}_{i, T}|$, after normalized by $f_T^2/T$, is at most \(f_T^2\), because both $N_i(T)$ and $\tilde{n}_{i, T}$ are bounded by $T$. Hence its tail event is empty for \(x>f_T^2\), and only the suboptimal-arm error can contribute in that range. We nevertheless state the tail bounds in a unified form for ease of exposition.

As noted earlier, $\beta=1$ marks a sharp phase transition in the algorithm's expected regret. Assumption~\ref{cond:f_t_UCB_regularity} ensures that $\beta>1$. We may therefore choose $1<\gamma<\beta$ in Lemma~\ref{lem: ultra tail}, obtaining a tail strictly lighter than the Cauchy tail. This suffices for all subsequent applications. When $\beta$ is larger, we may choose a more convenient exponent, such as
$\gamma=2$; for example, $f_t=t^\alpha$ with $\alpha>0$ gives
$\beta=\infty$. Let us finally mention that \eqref{eq:ultra tail} is nearly tight for right-tail, but not so for left-tail, as detailed in the following remarks.

\begin{remark}
 \jledit{The tail bound mainly reflects the heavy-tailed large deviations governed by the algorithm's exploitability.}
 This \jledit{bound }
 provides a particularly sharp characterization for suboptimal arms, when the underlying algorithm is near-optimized, a phenomenon termed fragility, identified in \citep{fan2021fragility}. Indeed, \eqref{eq:f_t_not_too_slow} allows $\beta$ to be arbitrarily close to \(1\), capturing algorithms with regret approaching the Lai-Robbins bound (for Gaussian rewards). This subsequently allows the largest admissible \(\gamma\) in Lemma~\ref{lem: ultra tail} to be arbitrarily close to \(1\), leading to a truncated-Cauchy tail. As is justified by the lower bound on regret tails in \cite{fan2021fragility},  this truncated-Cauchy tail turns out to be fundamental and is not merely a technical artifact of the proof. 
\end{remark}

\begin{remark}\label{rmk:sharper underpull concentration}
\eqref{eq:ultra tail} by itself does not give a strong enough bound in the regime where a suboptimal arm is substantially under-pulled. Indeed, the behavior of a suboptimal arm is fundamentally asymmetric. The polynomial term in \eqref{eq:ultra tail} tightly captures the possibility that $N_j(T)$ becomes exceedingly large. In contrast, the event that $N_j(T)$ becomes smaller than $n_{j, T}\jledit{/2}$ decays at an exponentially rate. The two-sided tail bound \eqref{eq:ultra tail} cannot capture this sharper underpull concentration. 
\end{remark}

In light of Remark~\ref{rmk:sharper underpull concentration}, we therefore develop separate, stronger lower-tail bounds for the suboptimal and optimal arms. This is needed particularly for the verification of \eqref{eq:N-no-starving} in Condition~\ref{ass:concentration}.


 \begin{lemma}
\label{lem:lower-count-tails}
The following holds for suboptimal and optimal arms, respectively.
\begin{enumerate}[label=\textup{(\roman*)}]

\item There exist constants $c_1,c_2>0$ and $T_0<\infty$ such that, for every $T\geq T_0$ and every $j\in\mathcal S$,
\begin{equation}
\Pp\left(N_j(T)<\frac{n_{j,T}}{2} \right) \leq c_1e^{-c_2f_T^2}.
\label{eq:suboptimal-lower-tail}
\end{equation}
The constants are independent of $T$.

\item For every fixed $1<\gamma<\beta$, there exist constants $\kappa\in(0,1)$, $c>0$, and $T_0<\infty$ such that, for every $T\geq T_0$ and every $i\in\mathcal O$, \begin{equation} 
\Pp\left(N_i(T)\leq\kappa T\right) \leq cT^{-\gamma}.
\label{eq:optimal-lower-tail}
\end{equation}
The constants may depend on $\gamma$ and the fixed bandit instance, but
not on $T$.

\end{enumerate}
\end{lemma}
\paragraph{Proof strategy.}
\jledit{ 
The main task in proving Theorem 2 is to establish the moment
bounds in Condition \ref{ass:concentration} using Lemmas \ref{lem:empirical-fluid-concentration}--\ref{lem:lower-count-tails}. Among Lemmas \ref{lem:empirical-fluid-concentration}--\ref{lem:lower-count-tails}, the main difficulty is Lemma \ref{lem: ultra tail}, which requires the actual
allocation $N_{a,T}$ to track the empirical-fluid allocation $\tilde n_{a,T}$ with small error. To prove this, we exploit the fact that
additional pulls reduce an arm's exploration bonus.
On a good event controlling reward fluctuations (see event $\mathcal G_T$ in Appendix \ref{subsec:ultra-tail-preliminaries}), a selection
taking an arm's pull count too far above its terminal target $\tilde n_{a,T}$ would be
impossible, because another arm would have a larger index.
To control suboptimal underpull, we examine the last
optimal-arm selection: if the suboptimal arms collectively
receive too few pulls, at least one would retain a sufficiently
large exploration bonus to outrank the selected arm. Lemmas \ref{lem:terminal-index-comparison}--\ref{lem:terminal-tracking} in Appendix \ref{subsec:ultra-tail-proof} formalize these comparisons.
Combining these comparisons with the probability of the
good event and a separate bound for very large deviations (see Lemma \ref{lem:nonlocal_suboptimal_tracking} in Appendix \ref{subsec:ultra-tail-proof})
yields Lemma \ref{lem: ultra tail}.
}

With the concentration bounds of  Lemma~\ref{lem:empirical-fluid-concentration}~-~\ref{lem:lower-count-tails} in place, we are now ready to finish the proof of Theorem~\ref{thm:ucb bias}.

\begin{proof}[Proof of Theorem~\ref{thm:ucb bias}]
By Theorem~\ref{thm:bias}, to establish the first part of the theorem it suffices to verify Conditions~\ref{ass:index-regularity}~-~\ref{ass:concentration} under the \texttt{UCB}-f index algorithm satisfying Assumption~\ref{cond:f_t_UCB_regularity}. 

\paragraph{Verifying Condition~\ref{ass:index-regularity}.}
The continuous differentiability, monotonicity and limits at $n \to 0$ and $n \to \infty$ follow from the index form  $I_t(x, n) = x + \frac{f_t}{\sqrt{n}}$ and Assumption~\ref{cond:f_t_UCB_regularity}. Furthermore, the derivatives admit the clean form $ \partial_x I_t(x,n) = 1$ and $\partial_n I_t(x,n) = -f_t/(2 n^{3/2}),$  from which the regularity of partial derivative in Condition~\ref{ass:index-regularity} follows immediately. This completes the verification of Condition~\ref{ass:index-regularity}.

\paragraph{Verifying Condition~\ref{ass:index_regularity2}.}
The condition is vacuous when $\mathcal S=\varnothing$. Otherwise, put $\delta_T=\lambda_T-\mu^\star>0$, where we recall that $\lambda_T$ is the common index value as a solution to the fluid equations. The fluid equations give $n_{i, T} = f_T^2/\delta_T^2$ for $i \in \mathcal{O}$ and $n_{j,T}=\left(f_T/(\Delta_j+\delta_T)\right)^2$ for $j \in \mathcal{S}.$ Thus $\sum_{j \in \mathcal{S}} n_{j, T} = O(f_T^2)$, which is $o(T)$ by \eqref{eq:f_t_not_too_fast} in Assumption~\ref{cond:f_t_UCB_regularity}. Since $\sum_{a \in [K]} n_{a, T} = T$, combining the above then yields $n_{i, T} = \frac{T}{|\mathcal{O}|} + O(f_T^2)$, and subsequently $\delta_T = f_T/\sqrt{n_{i, T}} = O(f_T/\sqrt{T}) = o(1)$. Therefore, $n_{j, T} = \frac{f_T^2}{\Delta_j^2}(1 + o(1))$, which diverges, whereas $\frac{n_{j,T}}{n_{i, T}} = (\delta_T/(\Delta_j + \delta_T))^2 = o(1)$. Finally, since
$\partial_n I_T(x,n)=-f_T/(2n^{3/2})$,
\[
\frac{\partial_n I_T(\mu_i,n_{i,T})}
     {\partial_n I_T(\mu_j,n_{j,T})}
=
\left(\frac{n_{j,T}}{n_{i,T}}\right)^{3/2}
=o(1).
\]
Thus Condition~\ref{ass:index_regularity2} holds.

\paragraph{Verifying Condition~\ref{ass:concentration}.}
By definition, 
\[
\Gamma_{a, T} = -2 {n_{a, T}}^{\frac{3}{2}} \frac{\partial_n I_T(\mu_a, n_{a, T})}{\partial_x I_T(\mu_a, n_{a, T})} =  2 {n_{a, T}}^{\frac{3}{2}} f_T \left(2 {n_{a, T}}^{\frac{3}{2}}\right)^{-1} =  f_T,
\]
which is $\omega(1)$ by Assumption~\ref{cond:f_t_UCB_regularity}. We now verify the three stated moment conditions respectively. 

Lemma~\ref{lem:empirical-fluid-concentration} immediately implies \eqref{eq:n-tilde-concentration}, where in particular we set the exponent $s_i = 2$ for $i \in \mathcal{O}$, and $1 < s_j < \beta$ for $j \in \mathcal{S}$, as is permitted by Lemma~\ref{lem:empirical-fluid-concentration}.  

\jledit{We next verify (5). For simplicity, we introduce additional notation \(d_{0,T}
     :=1+ f_T/\{\log(e+f_T)\}^2.\) By Assumption \ref{cond:f_t_UCB_regularity}, $d_{0,T}\to\infty$ and
$d_{0,T}=o(f_T)$. All estimates below hold for sufficiently
large $T$. We begin by examining the
suboptimal arm. Fix $j\in S$ and choose $\gamma_j$ such that
\(
    1<s_j<\gamma_j<\beta.
\)
Let $C_j\ge1$ be the constant in Lemma \ref{lem: ultra tail} corresponding to
$\gamma_j$. Since
$|N_j(T)-\tilde n_{j,T}|\le T$, the layer-cake formula
and Lemma \ref{lem: ultra tail} give
\begin{align*}
    \mathbb E\!\left[
        |N_j(T)-\tilde n_{j,T}|^{s_j}
    \right]
    &\le
    (C_jd_{0,T})^{s_j}
    +C_js_j\int_{C_jd_{0,T}}^T
        x^{s_j-1-\gamma_j}\,dx\\
    &\le
    (C_jd_{0,T})^{s_j}
    +\frac{C_js_j}{\gamma_j-s_j}
        (C_jd_{0,T})^{s_j-\gamma_j}\\
    &=o(f_T^{s_j}).
\end{align*}
Here the first term is $o(f_T^{s_j})$ because
$d_{0,T}=o(f_T)$, and the second term tends to zero
because $s_j<\gamma_j$ and $d_{0,T}\to\infty$.
Consequently,
\[
    \|N_j(T)-\tilde n_{j,T}\|_{s_j}=o(f_T).
\]
Since $n_{j,T}\asymp f_T^2$, it follows that
\[
    f_T
    \left\|
        \frac{N_j(T)-\tilde n_{j,T}}{n_{j,T}}
    \right\|_{s_j}
    =o(1).
\]
Now fix $i\in O$. Take $s_i=2$, and choose
\(
    1<\gamma_i<\min\{\beta,2\}.
\)
Define
\[
    W_{i,T}
    :=\frac{f_T^2}{T}
        |N_i(T)-\tilde n_{i,T}|.
\]
Because both $N_i(T)$ and $\tilde n_{i,T}$ belong to
$[0,T]$, we have $0\le W_{i,T}\le f_T^2$.
Let $C_i\ge1$ be the constant in Lemma \ref{lem: ultra tail} corresponding
to $\gamma_i$. Since $f_T^2=o(T)$, Lemma \ref{lem: ultra tail} applies
throughout the integration range below. The layer-cake
formula therefore gives
\begin{align*}
    \mathbb E[W_{i,T}^2]
    &\le
    (C_id_{0,T})^2
    +2C_i\int_{C_id_{0,T}}^{f_T^2}
        x^{1-\gamma_i}\,dx\\
    &\le
    (C_id_{0,T})^2
    +\frac{2C_i}{2-\gamma_i}
        f_T^{4-2\gamma_i}\\
    &=o(f_T^2).
\end{align*}
The last equality follows from $d_{0,T}=o(f_T)$
and $\gamma_i>1$. Thus $\|W_{i,T}\|_2=o(f_T)$.
Using $n_{i,T}\asymp T$, we obtain
\[
    f_T
    \left\|
        \frac{N_i(T)-\tilde n_{i,T}}{n_{i,T}}
    \right\|_2
    =
    \frac{T}{f_Tn_{i,T}}\|W_{i,T}\|_2
    =o(1).
\]
Together with the suboptimal-arm bound, this proves
(5) for every arm.}

It remains to verify \eqref{eq:N-no-starving}. For $j\in\mathcal S$ and any fixed $r_j>0$, initialization (so that $N_j(T) \ge 1$), $n_{j,T}=O(f_T^2)$, and
Lemma~\ref{lem:lower-count-tails}(i) give
\[
\begin{aligned}
\E\left(\frac{n_{j,T}}{N_j(T)}\right)^{r_j} \leq 2^{r_j} + n_{j,T}^{r_j} \Pp\left(N_j(T)<\frac{n_{j,T}}2\right)  \leq  2^{r_j} + C_1 f_T^{2r_j}e^{-c_2 f_T^2},
\end{aligned}
\]
which is $O(1)$  for arbitrary fixed constant $r_j > 0$. We can therefore choose $r_j > \max\left\{ \frac{s_j}{2(s_j - 1)}, 1\right\}$.

For an optimal arm $i \in \mathcal{O}$, we choose $1< r_i <\gamma_i$, with the same $\gamma_i$ fixed in the verification of \eqref{eq:N-concentration} above. Since $n_{i,T}\leq T$, and $1 < \gamma_i < \beta$, Lemma~\ref{lem:lower-count-tails}(ii) gives
\[
\begin{aligned}
\E\left(\frac{n_{i,T}}{N_i(T)}\right)^{r_i} \leq \kappa^{-r_i} + T^{r_i}\Pp(N_i(T)\leq\kappa T)   \leq
\kappa^{-r_i} + cT^{r_i-\gamma_i},
\end{aligned}
\]
which is $O(1)$ for $r_i < \gamma_i.$ Since $s_i = 2$ as set in the verification of \eqref{eq:N-concentration}, $\frac{s_i}{2(s_i - 1)} = 1.$ Since $\gamma_i > 1$, we may then choose $1 < r_i < \gamma_i$. Hence all parts of Condition~\ref{ass:concentration} hold.

Combining the above verification with Theorem~\ref{thm:bias} thus completes the proof of Theorem~\ref{thm:ucb bias}.

\end{proof}
}

\subsection{Proof of Theorem~\ref{thm:CDF expansion}}
The following lemmas reduce the proof to a first-order \(L^2\) approximation of \(N_i(T)\). There proofs are given in Appendix~\ref{app:CDF-expansion}. For simplicity, let \(\delta_i = \frac{\sigma_i}{f_T}
\left(1-\frac{1}{|\mathcal O|}\right).\)

\begin{lemma}[Gaussian CDF comparison]
\label{lem:gaussian-cdf-comparison}
Assume Gaussian rewards and Assumption~\ref{cond:f_t_UCB_regularity}. Fix \(i\in\mathcal O\). Suppose that
\[
\left\|
\frac{N_i(T)}{n_{i,T}}-1
\right\|_2
=
O(f_T^{-1}).
\]
Then there exists a constant \(C<\infty\), independent of \(T\),
such that
\[
\begin{aligned}
&\sup_{z\in\mathbb R}
\left|
\mathbb P(Z_{i,T}+\delta_i\leq z)-\Phi(z)
\right|
\\
&\quad\leq
C\Bigg\|
\frac{N_i(T)}{n_{i,T}}-1
-
\frac{2\sqrt{n_{i,T}}}{f_T}
\Bigg[
\bar X_i(n_{i,T})-\mu_i
-
\frac{1}{|\mathcal O|}
\sum_{a\in\mathcal O}
\bigl\{\bar X_a(n_{a,T})-\mu_a\bigr\}
\Bigg]
\Bigg\|_2
+
o(f_T^{-1}).
\end{aligned}
\]
\end{lemma}

\begin{lemma}
\label{lem:two-tracking-estimates}
Assume Gaussian rewards and
Assumption~\ref{cond:f_t_UCB_regularity}.
For every \(i\in\mathcal O\),
\[
\begin{aligned}
\Bigg\|
\frac{N_i(T)}{n_{i,T}}-1
-
\frac{2\sqrt{n_{i,T}}}{f_T}
\Bigg[
\bar X_i(n_{i,T})-\mu_i
-
\frac{1}{|\mathcal O|}
\sum_{a\in\mathcal O}
\bigl\{\bar X_a(n_{a,T})-\mu_a\bigr\}
\Bigg]
\Bigg\|_2
=
o(f_T^{-1}).
\end{aligned}
\]
\end{lemma}

\begin{proof}[Proof of Theorem~\ref{thm:CDF expansion}]
Lemma~\ref{lem:two-tracking-estimates} shows that the norm
on the right-hand side of
Lemma~\ref{lem:gaussian-cdf-comparison} is \(o(f_T^{-1})\).
The linear Gaussian component has
\(L^2\)-norm \(O(f_T^{-1})\), so the triangle inequality also gives
\[
\left\|
\frac{N_i(T)}{n_{i,T}}-1
\right\|_2
=
O(f_T^{-1}).
\]
Lemma~\ref{lem:gaussian-cdf-comparison} therefore proves the theorem.
\end{proof}

\begin{proof}[Proof of Corollary~\ref{cor:CDF expansion}]
The corollary follows from taking $z = z + \delta_i$ and expanding $\Phi(z + \delta_i)$ at $z$ in Theorem~\ref{thm:CDF expansion}, whose second order term is uniformly $O(f_T^{-2})$.
\end{proof}

\section{Concluding Remarks}

{This paper studies the bias of sample-mean estimators computed from data generated by bandit algorithms. We focus on stable index algorithms and show that the bias admits a sharp leading-order characterization. The resulting formula identifies the algorithmic origin of bias: it is governed by the effective exploration rate of the index function, except for the unique optimal arm whose bias is smaller. Specializing the result to generalized \texttt{UCB} algorithms, we show that the standardized bias can vanish very slowly under UCB1, at the rate $1/\sqrt{\log T}$. We also show how the general choice of the exploration function affects both regret and bias in the opposite direction.



Several directions remain open. First, extending the analysis to \texttt{Thompson Sampling}, \texttt{MOSS}-type algorithms, and contextual or linear bandits would clarify how broadly these mechanisms apply. Further questions concern distributional corrections beyond Gaussian rewards and guarantees that remain informative when several arms are nearly optimal. Implementing the corrections suggested by our formulas also requires accounting for uncertainty in reward variances and the number of optimal arms, and establishing the resulting accuracy of tests and confidence intervals. Finally, our results point to a broader question in adaptive decision-making: how should algorithms be designed when they must serve both online learning objectives and downstream tasks like statistical inference? Understanding this trade-off is essential for the responsible use of adaptive algorithms in modern data-driven systems.}

\bibliographystyle{abbrvnat}
\bibliography{bandit}
\newpage
\appendix



\section{Proofs of Lemmas~\ref{lem:main-bias} and~\ref{lem:remainders}}
\label{sec:proof of lemmas for thm1}

We use the centered sample means and reward sums
\[
\bar Y_a(n)=\bar X_a(n)-\mu_a,
\qquad
S_a(n)=\sum_{j=1}^n(X_{a,j}-\mu_a),
\qquad a\in[K],\quad n\ge1.
\]
The arm reward sequences are independent, with
$\sigma_a^2=\operatorname{Var}(X_{a,1})>0$.
Nonintegral arguments of sample means and reward sums are rounded
down, as in Definition~\ref{def:empirical fluid system}; the fluid
counts themselves remain real valued.

\begin{proof}[Proof of Lemma~\ref{lem:main-bias}]
Choose deterministic sequences $\varepsilon_T,\delta_T\downarrow0$
such that, for every arm,
\[
\varepsilon_T\sqrt{n_{a,T}}\longrightarrow\infty,
\qquad
\delta_T\Gamma_{a,T}\longrightarrow\infty.
\]
Such sequences exist because $n_{a,T}\to\infty$ and
$\Gamma_{a,T}\to\infty$ under
Conditions~\ref{ass:index_regularity2}--\ref{ass:concentration}.
Let
\[
\mathcal E_T=
\left\{
\max_{a\in[K]}|\bar Y_a(n_{a,T})|\le\varepsilon_T,
\quad
\max_{a\in[K]}
\frac{|\widetilde n_{a,T}-n_{a,T}|}{n_{a,T}}
\le\delta_T
\right\}.
\]
Sub-Gaussian concentration and Markov's inequality, using
\eqref{eq:n-tilde-concentration}, give
\begin{equation}
\mathbb P(\mathcal E_T^c)=o(1).
\label{eq:bias-remainder-term-0}
\end{equation}

Write $s_0=\min_{c\in[K]}s_c>1$, where $s_c$ is the exponent
in Condition~\ref{ass:concentration}, and set
\[
W_{a,T}=\sqrt{n_{a,T}}\,\bar Y_a(n_{a,T}),
\qquad
D_{a,T}=\frac{\Gamma_{a,T}}{n_{a,T}}
(\widetilde n_{a,T}-n_{a,T}).
\]
Then $\|D_{a,T}\|_{s_0}=O(1)$, while $W_{a,T}$ has uniformly
bounded moments of every fixed order. For any $1<\alpha<s_0$,
H\"older's inequality yields
\[
\|D_{a,T}W_{a,T}\|_\alpha
\le
\|D_{a,T}\|_{s_0}
\|W_{a,T}\|_{\alpha s_0/(s_0-\alpha)}
=O(1).
\]
Consequently, the contribution outside $\mathcal E_T$ is negligible:
\begin{equation}
\left|
\mathbb E[D_{a,T}W_{a,T}\mathbb I_{\mathcal E_T^c}]
\right|
\le
\|D_{a,T}W_{a,T}\|_\alpha
\mathbb P(\mathcal E_T^c)^{1-1/\alpha}
=o(1).
\label{eq:bias-remainder-term-1}
\end{equation}

On $\mathcal E_T$, we linearize the empirical-fluid equations
around their deterministic counterparts. Write
\[
L_{a,T}=\partial_xI_T(\mu_a,n_{a,T}),
\qquad
M_{a,T}=-\partial_nI_T(\mu_a,n_{a,T}).
\]
The integral form of Taylor's theorem gives
\begin{equation}
\widetilde\lambda_T-\lambda_T
=L_{a,T}\frac{W_{a,T}}{\sqrt{n_{a,T}}}
-M_{a,T}(\widetilde n_{a,T}-n_{a,T})+r_{a,T}.
\label{eq:bias-main-term-1}
\end{equation}
By Condition~\ref{ass:index-regularity}, there is a deterministic
$\eta_T\to0$ such that, on $\mathcal E_T$, simultaneously for all arms,
\[
|r_{a,T}|
\le\eta_T\left(
M_{a,T}|\widetilde n_{a,T}-n_{a,T}|
+L_{a,T}\frac{|W_{a,T}|}{\sqrt{n_{a,T}}}
\right).
\]
Rearranging \eqref{eq:bias-main-term-1} gives
\begin{equation}
\widetilde n_{a,T}-n_{a,T}
=\frac{L_{a,T}}{M_{a,T}}
\frac{W_{a,T}}{\sqrt{n_{a,T}}}
-\frac{\widetilde\lambda_T-\lambda_T}{M_{a,T}}
+\frac{r_{a,T}}{M_{a,T}}.
\label{eq:bias-main-term-1-1}
\end{equation}
Both fluid allocations sum to $T$. Summing this identity therefore
determines the common level:
\begin{equation}
\widetilde\lambda_T-\lambda_T
=
\frac{\displaystyle
\sum_{c=1}^K\left(
\frac{L_{c,T}W_{c,T}}{M_{c,T}\sqrt{n_{c,T}}}
+\frac{r_{c,T}}{M_{c,T}}\right)}
{\displaystyle\sum_{c=1}^KM_{c,T}^{-1}}.
\label{eq:bias-main-term-1-2}
\end{equation}
Define
\[
\pi_{a,T}=\frac{M_{a,T}^{-1}}{\sum_{c=1}^KM_{c,T}^{-1}},
\qquad
h_{a,c,T}=\pi_{a,T}\frac{\Gamma_{a,T}}{n_{a,T}}
\frac{n_{c,T}}{\Gamma_{c,T}},
\]
and
\[
\rho_{a,T}=\frac{\Gamma_{a,T}}{n_{a,T}}
\left(
\frac{r_{a,T}}{M_{a,T}}
-\pi_{a,T}\sum_{c=1}^K\frac{r_{c,T}}{M_{c,T}}
\right).
\]
Since $\Gamma_{c,T}=2n_{c,T}^{3/2}M_{c,T}/L_{c,T}$,
substituting \eqref{eq:bias-main-term-1-2} into
\eqref{eq:bias-main-term-1-1} yields
\begin{equation}
D_{a,T}=2W_{a,T}-2\sum_{c=1}^Kh_{a,c,T}W_{c,T}+\rho_{a,T}.
\label{eq:bias-main-term-1-3}
\end{equation}
Moreover, on $\mathcal E_T$,
\begin{equation}
|\rho_{a,T}|
\le\eta_T\left\{
|D_{a,T}|+2|W_{a,T}|
+\sum_{c=1}^Kh_{a,c,T}(|D_{c,T}|+2|W_{c,T}|)
\right\}.
\label{eq:bias-main-term-1-3-1}
\end{equation}

We next show that $H_{a,T}:=\sum_{c=1}^Kh_{a,c,T}$ is bounded;
this prevents the common-level adjustment from amplifying the
Taylor remainder. Otherwise, along a subsequence,
$H_{a,T}\to\infty$ and
$h_{a,c,T}/H_{a,T}\to\omega_c$, where
$\omega_c\ge0$ and $\sum_c\omega_c=1$.
Divide \eqref{eq:bias-main-term-1-3} by $H_{a,T}$.
The moment bounds on $D_{c,T}$ and $W_{c,T}$, together with
\eqref{eq:bias-main-term-1-3-1} and
$\mathbb P(\mathcal E_T^c)\to0$, imply
\[
\sum_{c=1}^K\omega_cW_{c,T}\xrightarrow{p}0.
\]
These weighted sums are uniformly sub-Gaussian, so their squares
are uniformly integrable. Their second moments would therefore
converge to zero. But independence across arms gives
\[
\mathbb E\left[\left(\sum_{c=1}^K\omega_cW_{c,T}\right)^2\right]
=\sum_{c=1}^K\omega_c^2\sigma_c^2
\frac{n_{c,T}}{\lfloor n_{c,T}\rfloor}
\longrightarrow\sum_{c=1}^K\omega_c^2\sigma_c^2>0,
\]
a contradiction. Thus $H_{a,T}=O(1)$.

It now follows from \eqref{eq:bias-main-term-1-3-1} that
$\|\rho_{a,T}\mathbb I_{\mathcal E_T}\|_{s_0}=o(1)$.
Another application of H\"older's inequality gives
\begin{equation}
\mathbb E[\rho_{a,T}W_{a,T}\mathbb I_{\mathcal E_T}]=o(1).
\label{eq:bias-main-term-2}
\end{equation}
Multiplying \eqref{eq:bias-main-term-1-3} by
$W_{a,T}\mathbb I_{\mathcal E_T}$ and taking expectations yields
\begin{align}
\mathbb E[D_{a,T}W_{a,T}\mathbb I_{\mathcal E_T}]
&=2(1-h_{a,a,T})
\mathbb E[W_{a,T}^2\mathbb I_{\mathcal E_T}]
\nonumber\\
&\quad-2\sum_{c\ne a}h_{a,c,T}
\mathbb E[W_{a,T}W_{c,T}\mathbb I_{\mathcal E_T}]+o(1)
\nonumber\\
&=2\left(1-\frac{\mathbf 1_{\{a\in\mathcal O\}}}{|\mathcal O|}\right)
\sigma_a^2+o(1).
\label{eq:bias-main-term-result}
\end{align}
Indeed, $h_{a,a,T}=\pi_{a,T}$ and
$\pi_{a,T}\to\mathbf 1_{\{a\in\mathcal O\}}/|\mathcal O|$
by Condition~\ref{ass:index_regularity2} and equality of the
optimal-arm fluid counts. Also,
$\mathbb E[W_{a,T}^2]=\sigma_a^2+o(1)$ and
$\mathbb E[W_{a,T}W_{c,T}]=0$ for $a\ne c$.
Removing the indicators changes these moments by $o(1)$, using
\eqref{eq:bias-remainder-term-0}, the sub-Gaussian moment bounds,
and $H_{a,T}=O(1)$.
Combining \eqref{eq:bias-remainder-term-1} and
\eqref{eq:bias-main-term-result} proves the lemma.
\end{proof}

To prove Lemma~\ref{lem:remainders}, we first justify replacing
the random sample size by its deterministic fluid approximation.

\begin{lemma}
\label{lemma:empirical_mean_reward_replacement}
Under Condition~\ref{ass:concentration}, suppose
$n_{a,T}\to\infty$. For the corresponding arm-specific exponents
$s,r$ and every $s/(s-1)<q<2r$,
\begin{equation}
\sqrt{n_{a,T}}\,
\|\bar X_a(N_a(T))-\bar X_a(n_{a,T})\|_q=o(1).
\label{eq:sample-mean-concentration}
\end{equation}
\end{lemma}

\begin{proof}
Fix an arm and write $n=\lfloor n_{a,T}\rfloor$ and $N=N_a(T)$.
Since $n/n_{a,T}\to1$, it suffices to use $\sqrt n$ in the
normalization. Condition~\ref{ass:concentration} gives
\[
\left\|\frac{N-n}{n}\right\|_s=o(1),
\qquad
\left\|\frac nN\right\|_r=O(1).
\]
We first establish convergence in probability and then a uniform
moment bound of order strictly larger than $q$.

For fixed $\delta\in(0,1/2)$ and $\epsilon>0$, localization to
$|N-n|\le\delta n$, followed by Lemma~\ref{lem:local_mean}
with its parameter fixed at $1/2$, gives
\[
\begin{aligned}
\mathbb P\{\sqrt n\,|\bar Y_a(N)-\bar Y_a(n)|>\epsilon\}
&\le C\exp\left\{-\frac{c\epsilon^2}{\nu^2\delta}\right\}
+\delta^{-s}\left\|\frac{N-n}{n}\right\|_s^s.
\end{aligned}
\]
Here the integer window is $\lfloor\delta n\rfloor$.
Letting $T\to\infty$ and then $\delta\downarrow0$ proves
\begin{equation}
\sqrt n\,|\bar Y_a(N)-\bar Y_a(n)|\xrightarrow{p}0.
\label{eq:empirical_mean_reward_replacement_1}
\end{equation}

Choose $q<p<2r$. Sub-Gaussianity gives
$\|S_a(n)/\sqrt n\|_p=O(1)$, so the inequality
\[
\sqrt n\,|\bar Y_a(N)-\bar Y_a(n)|
\le\frac{\sqrt n\,|S_a(N)|}{N}
+\frac{|S_a(n)|}{\sqrt n}
\]
reduces the moment bound to the first term.
On $2^jn\le N<2^{j+1}n$, Lemma~\ref{lem:maximal inequality}
and integration of its tail bound give
\[
\begin{aligned}
\left\|
\frac{\sqrt n\,S_a(N)}N
\mathbb I_{\{2^jn\le N<2^{j+1}n\}}
\right\|_p
&\le\frac{2^{-j}}{\sqrt n}
\left\|\max_{1\le m\le2^{j+1}n}|S_a(m)|\right\|_p
\le C2^{-j/2}.
\end{aligned}
\]
Summing over $j\ge0$ controls the contribution from $N\ge n$.

For $N<n$, the inverse-count moment bound controls the cost of
small denominators. On the blocks
$2^{-j-1}n\le N<2^{-j}n$, H\"older's inequality gives, for
$\alpha>1$,
\[
\begin{aligned}
\left\|
\frac{\sqrt n\,S_a(N)}N
\mathbb I_{\{2^{-j-1}n\le N<2^{-j}n\}}
\right\|_p
&\le\frac{2^{j+1}}{\sqrt n}
\left\|\max_{1\le m\le2^{-j}n}|S_a(m)|\right\|_{p\alpha}
\mathbb P\{N\le2^{-j}n\}^{(\alpha-1)/(\alpha p)}
\\
&\le C2^{-j\{r(\alpha-1)/(\alpha p)-1/2\}}.
\end{aligned}
\]
The last inequality uses the maximal bound and
$\mathbb P\{N\le2^{-j}n\}\le2^{-rj}\|n/N\|_r^r$.
Only blocks with $2^{-j}n\ge1$ can contribute; all maxima are
over integer indices. Choose $\alpha>2r/(2r-p)$, which makes
the displayed geometric series summable. Thus the normalized
sample-mean difference is uniformly bounded in $L^p$.
Together with \eqref{eq:empirical_mean_reward_replacement_1}
and $p>q$, this implies convergence in $L^q$, proving the lemma.
\end{proof}

We now prove Lemma~\ref{lem:remainders}, with $n_{a,T}\to\infty$
as above; this holds under the fluid conditions in its application
to Theorem~\ref{thm:bias}.

\begin{proof}[Proof of Lemma~\ref{lem:remainders}]
Fix an arm and its exponents $s,r$. For every
$s/(s-1)<q<2r$, the preceding lemma and sub-Gaussianity give
\begin{equation}
\begin{aligned}
\|\sqrt{n_{a,T}}\,\bar Y_a(N_a(T))\|_q
&\le
\sqrt{n_{a,T}}\,
\|\bar Y_a(N_a(T))-\bar Y_a(n_{a,T})\|_q
+\|\sqrt{n_{a,T}}\,\bar Y_a(n_{a,T})\|_q
=O(1).
\end{aligned}
\label{eq:empirical mean moment bound}
\end{equation}
This proves the moment assertion. By H\"older's inequality,
\begin{align*}
&\frac{\Gamma_{a,T}}{\sqrt{n_{a,T}}}
\mathbb E\left[
|N_a(T)-\widetilde n_{a,T}|\,|\bar Y_a(N_a(T))|
\right]
\\
&\qquad\le
\left\|\frac{\Gamma_{a,T}}{n_{a,T}}
(N_a(T)-\widetilde n_{a,T})\right\|_s
\|\sqrt{n_{a,T}}\,\bar Y_a(N_a(T))\|_{s/(s-1)}
=o(1),
\end{align*}
and
\begin{align*}
&\frac{\Gamma_{a,T}}{\sqrt{n_{a,T}}}
\mathbb E\left[
|\widetilde n_{a,T}-n_{a,T}|\,
|\bar Y_a(N_a(T))-\bar Y_a(n_{a,T})|
\right]
\\
&\qquad\le
\left\|\frac{\Gamma_{a,T}}{n_{a,T}}
(\widetilde n_{a,T}-n_{a,T})\right\|_s
\left\|\sqrt{n_{a,T}}
\bigl(\bar Y_a(N_a(T))-\bar Y_a(n_{a,T})\bigr)
\right\|_{s/(s-1)}
=o(1).
\end{align*}
These prove \eqref{eq:remainder term 1} and
\eqref{eq:remainder term 2}, using
\eqref{eq:N-concentration}, \eqref{eq:n-tilde-concentration},
and norm monotonicity from any admissible $q>s/(s-1)$.

For \eqref{eq:remainder term 3}, set
\[
B_T=
\frac{|\sqrt{N_a(T)}-\sqrt{n_{a,T}}|}
{\sqrt{N_a(T)}+\sqrt{n_{a,T}}}.
\]
Condition~\ref{ass:concentration} implies
$N_a(T)/n_{a,T}\xrightarrow{p}1$. Hence $B_T\xrightarrow{p}0$
and $0\le B_T\le1$, so $\|B_T\|_p=o(1)$ for every fixed
finite $p$. Choose $s/(s-1)<q<2r$ and $p$ such that
$1/p+1/s+1/q=1$. Since
\[
(\sqrt{N_a(T)}-\sqrt{n_{a,T}})^2
=B_T|N_a(T)-n_{a,T}|,
\]
H\"older's inequality gives
\begin{align*}
&\frac{\Gamma_{a,T}}{\sqrt{n_{a,T}}}
\mathbb E\left[
(\sqrt{N_a(T)}-\sqrt{n_{a,T}})^2
|\bar Y_a(N_a(T))|
\right]
\\
&\qquad\le
\|B_T\|_p
\left\|\frac{\Gamma_{a,T}}{n_{a,T}}
(N_a(T)-n_{a,T})\right\|_s
\|\sqrt{n_{a,T}}\,\bar Y_a(N_a(T))\|_q
=o(1).
\end{align*}
The count norm is bounded by
\eqref{eq:n-tilde-concentration} and \eqref{eq:N-concentration},
and the sample-mean norm by \eqref{eq:empirical mean moment bound}.
This proves the last remainder bound.
\end{proof}

\section{Empirical reward concentrations}

This section collects the concentration estimates used below.
For every arm $a\in[K]$, write
\[
Y_a(r):=X_{a,r}-\mu_a,
\qquad
S_a(n)=\sum_{r=1}^n Y_a(r),
\qquad
S_a(0):=0.
\]
We continue to use $\bar Y_a(n)=\bar X_a(n)-\mu_a=S_a(n)/n$
for integers $n\ge1$. The rewards $Y_a(r)$ are independent,
centered, and $\nu$-sub-Gaussian.

\begin{lemma}[Maximal sub-Gaussian sums]
\label{lem:maximal inequality}
For every arm $a\in[K]$, integers $m\ge0$ and $n\ge1$,
and $x>0$,
\begin{align*}
\mathbb P\left\{
\max_{1\le k\le n}\bigl(S_a(m+k)-S_a(m)\bigr)\ge x
\right\}
&\le \exp\left\{-\frac{x^2}{2\nu^2n}\right\},\\
\mathbb P\left\{
\max_{1\le k\le n}|S_a(m+k)-S_a(m)|\ge x
\right\}
&\le 2\exp\left\{-\frac{x^2}{2\nu^2n}\right\}.
\end{align*}
The same bounds hold for partial sums of any deterministic finite
block of centered rewards read in reverse order.
\end{lemma}

\begin{proof}
Fix $a,m,n$, and let $\lambda>0$. Independence and
sub-Gaussianity imply that
\[
M_k(\lambda)
:=
\exp\left\{
\lambda\bigl(S_a(m+k)-S_a(m)\bigr)
-\frac{\nu^2\lambda^2k}{2}
\right\},
\qquad 0\le k\le n,
\]
is a nonnegative supermartingale with respect to the filtration
generated by $Y_a(m+1),\ldots,Y_a(m+k)$, with $M_0(\lambda)=1$.
If $S_a(m+k)-S_a(m)\ge x$ for some $k\le n$, then
\[
M_k(\lambda)
\ge \exp\left\{\lambda x-\frac{\nu^2\lambda^2n}{2}\right\}.
\]
The Doob--Ville maximal inequality therefore gives
\[
\mathbb P\left\{
\max_{1\le k\le n}\bigl(S_a(m+k)-S_a(m)\bigr)\ge x
\right\}
\le
\exp\left\{-\lambda x+\frac{\nu^2\lambda^2n}{2}\right\}.
\]
Choosing $\lambda=x/(\nu^2n)$ proves the first bound.
Applying the same argument to the negative rewards and taking a
union bound proves the second. The argument applies to any
independent centered $\nu$-sub-Gaussian sequence, so it also applies
to a deterministic finite block read in reverse order.
\end{proof}

\begin{lemma}[Local sample-mean fluctuation bound]
\label{lem:local_mean}
For every $\delta\in(0,1)$, arm $a\in[K]$, integers $n\ge1$
and $1\le w\le(1-\delta)n$, and $x>0$,
\[
\mathbb P\left\{
\max_{\substack{k\ge1\\ |k-n|\le w}}
|\bar Y_a(k)-\bar Y_a(n)|>x
\right\}
\le
6\exp\left\{-\frac{\delta^2n^2x^2}{8\nu^2w}\right\},
\]
where the maximum is over integer sample sizes $k$.
\end{lemma}

\begin{proof}
A change in the sample mean comes from the observations added or
removed and from the change in the denominator. Indeed,
\[
\bar Y_a(k)-\bar Y_a(n)
=
\frac{S_a(k)-S_a(n)}{k}
-\frac{k-n}{nk}S_a(n).
\]
For $|k-n|\le w$, we have $k\ge\delta n$, and hence
\[
|\bar Y_a(k)-\bar Y_a(n)|
\le
\frac{|S_a(k)-S_a(n)|}{\delta n}
+\frac{w|S_a(n)|}{\delta n^2}.
\]
Thus a fluctuation larger than $x$ requires either
$|S_a(k)-S_a(n)|>\delta nx/2$ for some such $k$, or
$|S_a(n)|>\delta n^2x/(2w)$.

For $k>n$, the increments are partial sums of
$Y_a(n+1),\ldots,Y_a(n+w)$. For $k<n$, their absolute values
are those of partial sums of the reversed block
$Y_a(n),\ldots,Y_a(n-w+1)$.
Applying Lemma~\ref{lem:maximal inequality} to both blocks and
to $S_a(n)$ gives
\begin{align*}
\mathbb P\left\{
\max_{\substack{k\ge1\\ |k-n|\le w}}
|\bar Y_a(k)-\bar Y_a(n)|>x
\right\}
&\le
4\exp\left\{-\frac{\delta^2n^2x^2}{8\nu^2w}\right\}
+2\exp\left\{-\frac{\delta^2n^3x^2}{8\nu^2w^2}\right\}\\
&\le
6\exp\left\{-\frac{\delta^2n^2x^2}{8\nu^2w}\right\},
\end{align*}
where the last inequality uses $w\le n$.
\end{proof}

\begin{corollary}[Sample-mean boundary with a positive gap]
\label{lemma:exact-decision-time-bonus}
For every arm $a\in[K]$, $b>0$, and $z\ge0$,
\[
\mathbb P\left\{
\exists n\ge1:
-\bar Y_a(n)\ge b+\frac{z}{\sqrt n}
\right\}
\le
\frac{\exp\{-z^2/(2\nu^2)\}}
{\exp\{b^2/(2\nu^2)\}-1},
\]
where $n$ ranges over the positive integers.
\end{corollary}

\begin{proof}
For each integer $n\ge1$, Lemma~\ref{lem:maximal inequality},
applied to the negative rewards, gives
\begin{align*}
\mathbb P\left\{-\bar Y_a(n)\ge b+\frac{z}{\sqrt n}\right\}
&\le
\exp\left\{-\frac{(b\sqrt n+z)^2}{2\nu^2}\right\}\\
&\le
\exp\left\{-\frac{z^2}{2\nu^2}\right\}
\exp\left\{-\frac{b^2n}{2\nu^2}\right\}.
\end{align*}
Taking a union bound over $n\ge1$ and summing the geometric
series proves the result.
\end{proof}

\section{Proof of Lemma \ref{lem: ultra tail}}
\label{app:tailbound-proof}

\subsection{Preliminaries}
\label{subsec:ultra-tail-preliminaries}

We work under Assumption~\ref{cond:f_t_UCB_regularity}
throughout. All notation from the model and the fluid systems
remains in force. We write $i\in\mathcal O$ for an optimal arm
and $j\in\mathcal S$ for a suboptimal arm. Recall the centered
rewards, their partial sums, and the centered sample means:
\[
Y_a(r)=X_{a,r}-\mu_a,
\qquad
S_a(n)=\sum_{r=1}^n Y_a(r),
\qquad
\bar Y_a(n)=\frac{S_a(n)}{n},
\]
with $S_a(0)=0$. Empirical means evaluated at nonintegral fluid
counts use downward rounding:
\[
\bar Y_a(n_{a,t})
:=
\bar Y_a(\lfloor n_{a,t}\rfloor).
\]
The fluid counts themselves remain real valued.

Recall
\[
\ell_T=\{\log(e+f_T)\}^2,
\qquad
d_{0,T}=1+\frac{f_T}{\ell_T},
\qquad
\beta=
\liminf_{t\to\infty}
\frac{f_t^2}{2\nu^2\log t}>1.
\]
Fix the exponent $\gamma\in(0,\beta)$ appearing in
Lemma~\ref{lem: ultra tail}. 
The horizon $T$ is an integer. Constants may depend on the
fixed bandit instance, the exploration schedule, and $\gamma$,
but not on $T$ or on a reward realization. Any additional
dependence is specified in the corresponding statement.

We will compare indices at earlier decision times with the fixed terminal empirical-fluid allocation $(\tilde n_{a,T})_{a\in[K]}$ and its common level $\tilde\lambda_T$. This subsection records the auxiliary results needed for these comparisons.


\begin{lemma}
\label{lem:f_scaling}
There exist constants $\theta\in(0,1)$ and $t_0<\infty$
such that, for every pair of real numbers $T\ge t\ge t_0$,
\[
\frac{t}{T}
\le
\left(\frac{f_t}{f_T}\right)^{2/\theta}
\le
\left(\frac{f_t}{f_T}\right)^2
\le1,
\qquad
0\le f_T-f_t\le f_T\frac{T-t}{T}.
\]
\end{lemma}

\begin{proof}
The assumption permits us to choose $\theta\in(0,1)$ and
$t_0>0$ such that $f_s>0$ and
$0\le2s f_s'/f_s\le\theta$ for every $s\ge t_0$.
For $T\ge t\ge t_0$, integration gives
\[
2\log\frac{f_T}{f_t}
=\int_t^T\frac{2f_s'}{f_s}\,\mathrm ds
\le\theta\log\frac{T}{t}.
\]
Exponentiating yields the first inequality. The next two follow
from $0<f_t/f_T\le1$ and $\theta<1$.
Finally, $f_t/f_T\ge\sqrt{t/T}\ge t/T$, so
$0\le f_T-f_t\le f_T(1-t/T)$, proving the remaining bound.
\end{proof}Y


\begin{lemma}
\label{lem:ucb-fluid-approximation}
Under Assumption~\ref{cond:f_t_UCB_regularity}, there exist
constants $C\ge1$ and $t_0<\infty$ such that, for every real
$t\ge t_0$,
\begin{align*}
0
&\le \frac{t}{|\mathcal O|}-n_{i,t}
\le C f_t^2,
&&i\in\mathcal O,\\
0
&\le \frac{f_t^2}{\Delta_j^2}-n_{j,t}
\le C\frac{f_t^3}{\sqrt t},
&&j\in\mathcal S.
\end{align*}
The constants depend only on the fixed bandit instance and the algorithm.
\end{lemma}

\begin{proof}
The fluid equations give the same count to every optimal arm
and imply $\lambda_t>\mu^\star$. Thus, for $j\in\mathcal S$,
\[
n_{j,t}
=\frac{f_t^2}{(\Delta_j+\lambda_t-\mu^\star)^2}
\le\frac{f_t^2}{\Delta_j^2}.
\]
By conservation, for every $i\in\mathcal O$,
\[
0\le\frac{t}{|\mathcal O|}-n_{i,t}
=\frac1{|\mathcal O|}\sum_{j\in\mathcal S}n_{j,t}
\le\frac{f_t^2}{|\mathcal O|}
\sum_{j\in\mathcal S}\Delta_j^{-2}.
\]
Since $f_t^2/t\to0$, it follows that
$n_{i,t}\ge t/(2|\mathcal O|)$ for all sufficiently large $t$.
The optimal-arm fluid equation therefore yields
\[
0<\lambda_t-\mu^\star
=\frac{f_t}{\sqrt{n_{i,t}}}
\le\sqrt{2|\mathcal O|}\frac{f_t}{\sqrt t}.
\]
Applying the mean-value theorem to $x\mapsto x^{-2}$ now gives,
for every $j\in\mathcal S$,
\[
0\le\frac{f_t^2}{\Delta_j^2}-n_{j,t}
\le\frac{2f_t^2}{\Delta_j^3}(\lambda_t-\mu^\star)
\le\frac{2\sqrt{2|\mathcal O|}}{\Delta_j^3}
\frac{f_t^3}{\sqrt t}.
\]
Both conclusions follow, for example, with
\[
C:=\max\left\{
1,\ \sum_{j\in\mathcal S}\Delta_j^{-2},\
2\sqrt{2K}\sum_{j\in\mathcal S}\Delta_j^{-3}
\right\}.
\]
Empty sums are interpreted as zero, so the proof also covers
$\mathcal S=\varnothing$.
\end{proof}

\medskip
\noindent\textbf{Good event $\mathcal{G}_T$.\ } Fix $T$. let $\mathcal G_T$ be the event that, for every arm $a\in[K]$
and integer $1\le k\le T$,
\[
|S_a(k)|\le \ell_T\sqrt k,
\qquad
|S_a(k)-S_a(n_{a,T})|
\le \ell_T\sqrt{|k-n_{a,T}|+1}.
\]
Additionally when $\mathcal S\ne\varnothing$, for every optimal arm $i\in\mathcal O$ and $1\le n\le t\le T$ satisfying $t\ge \frac12\min_{j\in\mathcal S}n_{j,T}$, it holds that \(-\bar Y_i(n) < \frac14\min_{j\in\mathcal S}\Delta_j +\frac{f_t}{\sqrt n}\). Event $\mathcal{G}_T$ is the ``good event'' that happens with high probability, on which we compare indices and control the behavior of $N$. In particular,

\begin{corollary}
\label{lem:terminal-reward-event}
There exists $T_0<\infty$ such that, for every integer $T\ge T_0$,
\[
\mathbb P(\mathcal G_T^c)\le f_T^{-2\gamma}.
\]
\end{corollary}

\begin{proof}
Fix $p\in(\gamma,\beta)$. We bound separately the failures of
the reward conditions and the optimal-index condition defining
$\mathcal G_T$.

For each arm, the reference count
$m_a:=\lfloor n_{a,T}\rfloor$ is deterministic and belongs to
$\{1,\ldots,T\}$ for all sufficiently large $T$.
Apply Lemma~\ref{lem:maximal inequality} to the reward sequence
starting at zero, the sequence starting after $m_a$, and the
finite sequence ending at $m_a$ read in reverse order.
For each sequence, split the positive partial-sum lengths into
blocks $2^r\le h<2^{r+1}$, with
$0\le r\le\lfloor\log_2T\rfloor$.
On each block, crossing the boundary $\ell_T\sqrt h$ requires
a maximal absolute partial sum exceeding $\ell_T2^{r/2}$.
The lemma bounds this probability by
$2\exp\{-\ell_T^2/(4\nu^2)\}$; shorter sequences are truncated
at their endpoints.

For the increments about $m_a$, observe that
$|k-m_a|\le |k-n_{a,T}|+1$. Thus these three families control
both reward conditions in $\mathcal G_T$, with their stated
coefficient one. A union bound gives
\[
\mathbb P\{\text{either reward condition fails}\}
\le
6K\bigl(1+\lfloor\log_2T\rfloor\bigr)
\exp\left\{-\frac{\ell_T^2}{4\nu^2}\right\}
=o(f_T^{-2\gamma}).
\]
The last relation follows from $\log T\le f_T^2/(2\nu^2p)$
for all sufficiently large $T$ and
$\ell_T^2=\{\log(e+f_T)\}^4$.

If $\mathcal S\ne\varnothing$, put
\[
d:=\frac12\min_{j\in\mathcal S}n_{j,T},
\qquad
b:=\frac14\min_{j\in\mathcal S}\Delta_j.
\]
Lemma~\ref{lem:ucb-fluid-approximation} gives
$d\ge f_T^2/(4\max_{j\in\mathcal S}\Delta_j^2)$ for all
sufficiently large $T$, so $d\to\infty$.
By eventual monotonicity, every permitted decision time satisfies
$f_t\ge f_d$. Failure of the optimal-index condition therefore
implies that $-\bar Y_i(n)\ge b+f_d/\sqrt n$ for some
$i\in\mathcal O$ and integer $n\ge1$.
Corollary~\ref{lemma:exact-decision-time-bonus} yields
\[
\mathbb P\{\text{the optimal-index condition fails}\}
\le
\frac{|\mathcal O|}{\exp\{b^2/(2\nu^2)\}-1}
\exp\left\{-\frac{f_d^2}{2\nu^2}\right\}
=o(f_T^{-2\gamma}).
\]
Indeed, $f_d^2/(2\nu^2)\ge p\log d$ eventually, and the lower
bound on $d$, together with $p>\gamma$, proves the last relation.
When $\mathcal S=\varnothing$, this condition is absent.

Increasing $T_0$ makes the sum of the failure probabilities at
most $f_T^{-2\gamma}$ for every $T\ge T_0$, proving the claim.
\end{proof}


\begin{lemma}
\label{lem:empirical_fluid_geometry}
There exist constants $C\ge1$ and $T_0<\infty$ such that,
for every integer $T\ge T_0$, on $\mathcal G_T$,
\[
\max_{a\in[K]}
\left|\frac{\tilde n_{a,T}}{n_{a,T}}-1\right|
\le \frac{C\ell_T}{f_T},
\qquad
|\tilde\lambda_T-\lambda_T|
\le \frac{C\ell_T}{\sqrt T}.
\]
\end{lemma}

\begin{proof}
Write $n_a:=n_{a,T}$ and $q_a:=\widetilde n_{a,T}$.
We first bound the change in the common index level, then use
the individual fluid equations to bound the count changes.
By Lemma~\ref{lem:ucb-fluid-approximation} and
$\ell_T/f_T\to0$, take $T$ sufficiently large that every
$n_a\ge2$ and $10\sqrt K\,\ell_T/f_T\le1/2$.
The prefix reward bound defining $\mathcal G_T$ gives
\[
|\bar Y_a(n_a)|
\le\frac{\ell_T}{\sqrt{\lfloor n_a\rfloor}}
\le\frac{2\ell_T}{\sqrt{n_a}}
\le\frac{f_T}{2\sqrt{n_a}}.
\]

Subtracting the two fluid equations and rationalizing yields
\[
q_a-n_a
=\frac{w_a}{f_T}
\bigl(\bar Y_a(n_a)-(\widetilde\lambda_T-\lambda_T)\bigr),
\qquad
w_a:=\sqrt{n_aq_a}\bigl(\sqrt{n_a}+\sqrt{q_a}\bigr).
\]
Since both allocations sum to $T$, summing this identity gives
\[
\widetilde\lambda_T-\lambda_T
=\frac{\sum_a w_a\bar Y_a(n_a)}{\sum_a w_a}.
\]
The numerator is controlled by the reward bound and
Cauchy--Schwarz:
\[
\left|\sum_a w_a\bar Y_a(n_a)\right|
\le2\ell_T\sum_a\bigl(\sqrt{n_aq_a}+q_a\bigr)
\le4\ell_T T.
\]
To bound the denominator, observe from
$f_T/\sqrt{q_a}=f_T/\sqrt{n_a}
+\widetilde\lambda_T-\lambda_T-\bar Y_a(n_a)$ that
$q_a\le4n_a$ for every arm if
$\widetilde\lambda_T\ge\lambda_T$, whereas
$q_a\ge4n_a/9$ for every arm otherwise.
In the first case $w_a\ge q_a^{3/2}/2$; in the second,
$w_a\ge2n_a^{3/2}/3$. The power-mean inequality therefore gives
\[
\sum_a w_a\ge\frac{T^{3/2}}{2\sqrt K},
\qquad
|\widetilde\lambda_T-\lambda_T|
\le\frac{8\sqrt K\,\ell_T}{\sqrt T}.
\]

Returning to an individual fluid equation, we obtain
\[
\sqrt{\frac{n_a}{q_a}}
=1+\frac{\sqrt{n_a}}{f_T}
\bigl(\widetilde\lambda_T-\lambda_T-\bar Y_a(n_a)\bigr).
\]
Because $n_a\le T$, the term added to one has absolute value
at most $10\sqrt K\,\ell_T/f_T\le1/2$.
The derivative of $x\mapsto x^{-2}$ has absolute value at most
$16$ on $[1/2,3/2]$. Hence
\[
\left|\frac{q_a}{n_a}-1\right|
\le\frac{160\sqrt K\,\ell_T}{f_T}.
\]
Both conclusions follow with $C=160\sqrt K$.
\end{proof}

Corollary~\ref{lem:terminal-reward-event} and
Lemma~\ref{lem:empirical_fluid_geometry} yield the following
sample-mean comparison.

\begin{lemma}
\label{lem:terminal-mean-comparison}
Fix $\eta>0$ and $0<\underline c<1<\overline c$.
There exists $T_0<\infty$ such that, for every integer $T\ge T_0$,
on $\mathcal G_T$,
\begin{equation}
\frac{\tilde n_{a,T}^{3/2}}{f_T}
\left|\bar Y_a(k)-\bar Y_a(n_{a,T})\right|
\le
\eta\left(
|k-\tilde n_{a,T}|
+\frac{\tilde n_{a,T}d_{0,T}}{f_T^2}
\right)
\label{eq:terminal-mean-comparison}
\end{equation}
for every $a\in[K]$ and integer $1\le k\le T$ satisfying
$\underline c\,\tilde n_{a,T}\le k
\le\overline c\,\tilde n_{a,T}$.
\end{lemma}
\begin{proof}
Fix an arm and an admissible $k$, and write
$n:=n_{a,T}$, $q:=\widetilde n_{a,T}$, and $m:=\lfloor n\rfloor$.
All constants below are uniform over arms, admissible counts,
and realizations in $\mathcal G_T$.
Lemmas~\ref{lem:ucb-fluid-approximation} and
\ref{lem:empirical_fluid_geometry} give, for all sufficiently
large $T$,
\[
q\ge\frac{f_T^2}{C},
\qquad
\frac q4\le m\le2q,
\qquad
|n-q|+1\le C\frac{q\ell_T}{f_T}.
\]
Here the rounding error of at most one is absorbed because
$q\ell_T/f_T\to\infty$ uniformly. We also have $k\ge\underline c q$.

The change in the sample mean has two contributions: the reward
increment and the change in the denominator. Precisely,
\[
\bar Y_a(k)-\bar Y_a(n)
=\frac{S_a(k)-S_a(m)}{k}
+S_a(m)\left(\frac1k-\frac1m\right).
\]
The bounds defining $\mathcal G_T$ control these terms by
$|S_a(k)-S_a(m)|\le\ell_T\sqrt{|k-n|+1}$ and
$|S_a(m)|\le\ell_T\sqrt m$. Using the preceding count bounds,
we obtain
\begin{align*}
\frac{q^{3/2}}{f_T}|\bar Y_a(k)-\bar Y_a(n)|
&\le
\frac{C\ell_T\sqrt q}{f_T}
\sqrt{|k-q|+\frac{q\ell_T}{f_T}}
+\frac{C\ell_T}{f_T}
\left(|k-q|+\frac{q\ell_T}{f_T}\right)\\
&\le
\frac{C\ell_T\sqrt q}{f_T}\sqrt{|k-q|}
+\frac{C\ell_T}{f_T}|k-q|
+\frac{Cq}{f_T^2}
\left(\ell_T^2+\sqrt{f_T}\,\ell_T^{3/2}\right).
\end{align*}
By Young's inequality, the first term in the last line is at most
\[
\frac{\eta}{2}|k-q|
+\frac{C^2}{2\eta}\frac{q\ell_T^2}{f_T^2}.
\]
Finally, the definitions of $\ell_T$ and $d_{0,T}$ give
\[
\frac{\ell_T}{f_T}\longrightarrow0,
\qquad
\frac{\ell_T^2+\sqrt{f_T}\,\ell_T^{3/2}}{d_{0,T}}
\longrightarrow0.
\]
After increasing $T_0$, the coefficient of $|k-q|$ is at most
$\eta$, and the remaining terms sum to at most
$\eta qd_{0,T}/f_T^2$. This proves the claimed inequality
simultaneously for all the stated arms and counts.
\end{proof}

\subsection{Proof of Lemma~\ref{lem:empirical-fluid-concentration}}\label{subsec:emp-fluid-conc-proof}
\begin{proof}[Proof of Lemma~\ref{lem:empirical-fluid-concentration}]
By Lemma~\ref{lem:ucb-fluid-approximation},
$n_{i,T}\asymp T$ for $i\in\mathcal O$ and
$n_{j,T}\sim f_T^2/\Delta_j^2$ for $j\in\mathcal S$.
Thus $\min_a n_{a,T}\to\infty$, and fixed-count sub-Gaussian
concentration gives, for every fixed $p>1$,
\[
\mathbb E\left[
\left(\max_{a\in[K]}\sqrt{n_{a,T}}
       |\bar Y_a(n_{a,T})|\right)^p\right]=O(1).
\]

Fix $i\in\mathcal O$ and $\varepsilon\in(0,1)$, to be chosen below.
First suppose that
\begin{equation}
\bar Y_j(n_{j,T})-\bar Y_i(n_{i,T})
\le(1-\varepsilon)\Delta_j
\qquad\text{for every }j\in\mathcal S.
\label{eq:empirical-fluid-gap-condition}
\end{equation}
The empirical-fluid equations then imply
\[
\frac{f_T}{\sqrt{\tilde n_{j,T}}}
=\Delta_j+\bar Y_i(n_{i,T})-\bar Y_j(n_{j,T})
 +\frac{f_T}{\sqrt{\tilde n_{i,T}}}
\ge\varepsilon\Delta_j.
\]
Consequently, $\tilde n_{j,T}\le f_T^2/(\varepsilon\Delta_j)^2$,
$\tilde n_{a,T}\le Cn_{a,T}$ for every arm, and
$\sum_{b\in\mathcal O}\tilde n_{b,T}=T-O(f_T^2)\ge T/2$
for all sufficiently large $T$. Constants here and below may depend
on $\varepsilon$.
Using the weighted identity for $\tilde\lambda_T-\lambda_T$
in the proof of Lemma~\ref{lem:empirical_fluid_geometry}, followed
by Cauchy--Schwarz, gives
\begin{align*}
|\tilde\lambda_T-\lambda_T|
&\le
\frac{\sum_{b\in[K]}
 (\sqrt{n_{b,T}\tilde n_{b,T}}+\tilde n_{b,T})}
 {\sum_{b\in\mathcal O}\sqrt{n_{b,T}}\,\tilde n_{b,T}}
\max_{b\in[K]}\sqrt{n_{b,T}}|\bar Y_b(n_{b,T})|\\
&\le\frac{C}{\sqrt T}
\max_{b\in[K]}\sqrt{n_{b,T}}|\bar Y_b(n_{b,T})|.
\end{align*}
The rationalized fluid identity from the same proof and
$\tilde n_{a,T}\le Cn_{a,T}$ therefore yield
\begin{align*}
f_T\left|\frac{\tilde n_{a,T}-n_{a,T}}{n_{a,T}}\right|
&\le C\sqrt{n_{a,T}}
\bigl(|\bar Y_a(n_{a,T})|+|\tilde\lambda_T-\lambda_T|\bigr)\\
&\le C\max_{b\in[K]}\sqrt{n_{b,T}}|\bar Y_b(n_{b,T})|.
\end{align*}
Hence the contribution to the required $p$th moment from event
\eqref{eq:empirical-fluid-gap-condition} is $O(1)$.

On the complementary event, some $j\in\mathcal S$ satisfies
$\bar Y_j(n_{j,T})-\bar Y_i(n_{i,T})>(1-\varepsilon)\Delta_j$.
Independence across arms and fixed-count sub-Gaussian concentration
bound its probability by
\[
\sum_{j\in\mathcal S}
\exp\left\{
-\frac{(1-\varepsilon)^2\Delta_j^2}
 {2\nu^2\bigl(\lfloor n_{j,T}\rfloor^{-1}
             +\lfloor n_{i,T}\rfloor^{-1}\bigr)}
\right\}
\le |\mathcal S|
\exp\left\{-\frac{(1-\varepsilon)^3 f_T^2}{2\nu^2}\right\}
\]
for all sufficiently large $T$, by the deterministic-fluid scales.
Since $\tilde n_{a,T}\le T$, the normalized deviation is bounded
by $Cf_T$ for optimal arms and by $CT/f_T$ for suboptimal arms.
It follows that
\[
f_T^p\mathbb E\left[
\left|\frac{\tilde n_{i,T}-n_{i,T}}{n_{i,T}}\right|^p\right]
\le C+Cf_T^p
\exp\left\{-\frac{(1-\varepsilon)^3 f_T^2}{2\nu^2}\right\}
=O(1)
\qquad(1<p<\infty).
\]
For a suboptimal arm and $1<p<\beta$, choose $\varepsilon>0$
small enough that $p<(1-\varepsilon)^3\beta$.
The definition of $\beta$ makes the preceding exponential at most
$T^{-p}$ eventually, and therefore
\[
f_T^p\mathbb E\left[
\left|\frac{\tilde n_{j,T}-n_{j,T}}{n_{j,T}}\right|^p\right]
\le C+C(T/f_T)^pT^{-p}=C+Cf_T^{-p}=O(1).
\]
If $\mathcal S=\varnothing$, condition
\eqref{eq:empirical-fluid-gap-condition} holds everywhere and
there is no complementary contribution.
Taking $p$th roots completes the proof.
\end{proof}

\subsection{Proof of Lemma~\ref{lem: ultra tail}}
\label{subsec:ultra-tail-proof}
The proof of Lemma~\ref{lem: ultra tail} is divided \jledit{into two key steps. We first bound the tracking error on $\mathcal G_T$ in Lemmas \ref{lem:terminal-index-comparison}--\ref{lem:terminal-tracking},
then establish a large-deviation bound in Lemma \ref{lem:nonlocal_suboptimal_tracking}, and finally
combine these estimates to prove Lemma \ref{lem: ultra tail}.}
\begin{lemma}
\label{lem:terminal-index-comparison}
There exists $T_0<\infty$ such that, for every integer $T\ge T_0$,
on $\mathcal G_T$,
\[
\max_{a\in[K]}
\left\{
\bar X_a(N_a(t-1))+\frac{f_t}{\sqrt{N_a(t-1)}}
\right\}
>
\tilde\lambda_T-\frac{d_{0,T}}{f_T\sqrt T}
\]
for every integer $K+1\le t\le T$ satisfying
$f_t\ge f_T/(4\sqrt K)$.
\end{lemma}

\begin{proof}[Proof of Lemma~\ref{lem:terminal-index-comparison}]
Work on $\mathcal G_T$ and take $T$ sufficiently large. We argue by contradiction. Suppose there exists $t$ satisfying the conditions in the lemma, but the index of every arm at time $t$ is at most $\tilde{\lambda}_T - d_{0, T}(f_T \sqrt{T})^{-1}$. Pick an arm $a$. For brevity, denote by $q_a = \tilde{n}_{a, T}, k_a = N_a(t - 1)$ and $I_{a, t}$ is the index value of arm $a$ at the beginning of time $t$. Then 
\begin{equation}
I_{a,t} = \mu_a + \bar Y_a(k_a) + \frac{f_t}{\sqrt{k_a}}
\le
\tilde\lambda_T-\frac{d_{0,T}}{f_T\sqrt T}.
\label{eq:terminal-index-contrary-assumption}
\end{equation}
By definition of the empirical fluid equation, we have 
\[
 \tilde{\lambda}_T =  \mu_a + \bar Y_a(n_{a, T}) + \frac{f_T}{\sqrt{q_a}} = \mu_a + \bar Y_a(n_{a, T}) + \frac{f_t}{\sqrt{r^2 q_a}} ,
\]
where $r = f_t/f_T$ is introduced to make the comparison explicit. Particularly, by assumption in the lemma and Assumption~\ref{cond:f_t_UCB_regularity}, it must be that \(1/(4 \sqrt{K}) \le r \le 1.\) Combining the above expressions and we get
\begin{equation}
    \bar Y_a(n_{a, T}) - \bar Y_a(k_a) \ge \frac{f_t}{\sqrt{k_a}} + \frac{d_{0, T}}{f_T \sqrt{T}} - \frac{f_t}{\sqrt{r^2 q_a}}.
\label{eq:index-comparison-1}
\end{equation}
We first establish a rough lower bound on $k_a$. By definition of $\mathcal G_T$ and
Lemma~\ref{lem:empirical_fluid_geometry}, we have
\[
|\bar Y_a(k_a)|\le\frac{\ell_T}{\sqrt{k_a}},
\qquad
|\bar Y_a(n_{a,T})|\le\frac{C\ell_T}{\sqrt{q_a}}.
\]
Thus \eqref{eq:index-comparison-1} implies
\[
\frac{\ell_T}{\sqrt{k_a}}+\frac{C\ell_T}{\sqrt{q_a}}
\ge
\frac{f_t}{\sqrt{k_a}}-\frac{f_T}{\sqrt{q_a}}.
\]
For all sufficiently large $T$, $f_t>\ell_T$, so rearranging gives \(k_a
\ge
\left(\frac{f_t-\ell_T}{f_T+C\ell_T}\right)^2q_a
=
\left(
\frac{r-\ell_T/f_T}{1+C\ell_T/f_T}
\right)^2q_a.\) Since $r\ge1/(4\sqrt K)$ and $\ell_T/f_T\to0$, for sufficiently large $T$,   we must have
\[
k_a>\left(\frac{r}{4}\right)^2q_a \ge \frac{q_a}{256K}.
\]
Now with this rough bound on $k_a$, we apply Lemma~\ref{lem:terminal-mean-comparison}. Let's start by considering the case that $k_a \le 2 q_a$. Then fix $\underline c = \frac{1}{256 K}, \overline c = 2$ and $\eta \in (0, 1/4)$ to be specified later and apply Lemma~\ref{lem:terminal-mean-comparison}, we have, for $T$ sufficiently large,
\[
\left|\bar Y_a(k_a)-\bar Y_a(n_{a,T})\right|
\le
\eta\frac{f_T}{q_a^{3/2}}
\left(
|k_a-q_a|+\frac{q_a d_{0,T}}{f_T^2}
\right).
\]
This combining with \eqref{eq:index-comparison-1} yields
\[
\eta\frac{f_T}{q_a^{3/2}}
\left(
|k_a-q_a|+\frac{q_a d_{0,T}}{f_T^2}
\right) \ge \frac{f_t}{\sqrt{k_a}} + \frac{d_{0, T}}{f_T \sqrt{T}} - \frac{f_t}{\sqrt{r^2 q_a}}.
\]
Observe that $x\mapsto f_t/\sqrt x$ is convex. Hence
\[
\frac{f_t}{\sqrt{k_a}}-\frac{f_t}{\sqrt{r^2 q_a}}
\ge
-\frac{f_T}{2r^2q_a^{3/2}}(k_a-r^2q_a),
\]
where $-f_T/(2r^2q_a^{3/2})$ is the derivative at $ x= r^2q_a$. Plugging back, rearranging terms and with a triangle inequality,
\[
\frac{k_a-r^2q_a}{2r^2}
+\eta|k_a-r^2q_a|
\ge
\frac{q_a^{3/2}d_{0,T}}{f_T^2\sqrt T}
-
\eta q_a\left(
1-r^2+\frac{d_{0,T}}{f_T^2}
\right).
\]
By assumption, $1/2 \le (2 r^2)^{-1} \le 8 K$ and $0 < \eta < 1/4$. With some straightforward algebra that accounts for the sign of $k_a - r^2 q_a$, the above bound yields
\begin{equation}
k_a
\ge
r^2q_a
+
\frac{q_a^{3/2}d_{0,T}}{9Kf_T^2\sqrt T}
-
4\eta q_a\left(
1-r^2+\frac{d_{0,T}}{f_T^2}
\right).
\label{eq:terminal-index-forced-count}
\end{equation}
The same bound also holds when $k_a>2q_a$. Indeed, $q_a\le T$ and
$d_{0,T}/f_T^2\to0$, so the right-hand side of
\eqref{eq:terminal-index-forced-count} is at most \(q_a+\frac{q_a d_{0,T}}{9Kf_T^2}<2q_a\) for all sufficiently large $T$. The above argument and thus 
\eqref{eq:terminal-index-forced-count} holds for every arm. Using the fact $\sum_a k_a = t-1$ and $\sum_a q_a = T$, we summing \eqref{eq:terminal-index-forced-count} over all arms to get
\begin{align}
    t - 1 &\ge r^2 T + \frac{\sum_a q_a^{3/2} d_{0, T}}{9 K f_T^2\sqrt{T}} - 4 \eta T\left(1 - r^2 + \frac{d_{0,T}}{f_T^2}\right),\nonumber\\
    & \ge r^2 T + \frac{T d_{0, T}}{9 K^{3/2} f_T^2} - 4 \eta T\left(1 - r^2 + \frac{d_{0,T}}{f_T^2}\right),
    \label{eq:contradiction_raw}
\end{align}
where the second inequality follows from the power mean inequality. Finally, we determine the choice of $\eta$ and deduce a contradiction. By Lemma~\ref{lem:f_scaling}, for some fixed $\theta \in (0, 1)$, we have $t/T \le (f_t/f_T)^{2/\theta}$ for $t$ and $T$ sufficiently large. Hence we have
\begin{equation}
r^2T-t
\ge
T\bigl(r^2-r^{2/\theta}\bigr)
\ge
c_{\mathrm s}T(1-r^2),
\label{eq:terminal-index-schedule-slack}
\end{equation}
for all $r \le 1$, where $c_{\mathrm s} > 0$ depends only on $\theta$ and $K$. We now specify our choice of $\eta$ to be sufficiently small such that
\[
4\eta
\le
\frac12\min\left\{
c_{\mathrm s},\frac{1}{9K\sqrt K}
\right\}.
\]
Combining with \eqref{eq:terminal-index-schedule-slack} and plugging  back to \eqref{eq:contradiction_raw}, we conclude that
\[
t-1 \ge t + (c_{\mathrm s}-4\eta)T(1-r^2) + \left(\frac{1}{9K\sqrt K}-4\eta\right) \frac{T d_{0,T}}{f_T^2} >t.
\]
This contradiction proves the desired strict inequality, completing the proof.
\end{proof}

\begin{lemma}
\label{lem:terminal-tracking}
There exist constants $C\ge1$ and $T_0<\infty$ such that,
for every integer $T\ge T_0$, on $\mathcal G_T$,
\begin{align*}
|N_j(T)-\tilde n_{j,T}|
&\le C d_{0,T},
&&j\in\mathcal S,\\
|N_i(T)-\tilde n_{i,T}|
&\le C\frac{T d_{0,T}}{f_T^2},
&&i\in\mathcal O.
\end{align*}
\end{lemma}

\begin{proof}[Proof of Lemma~\ref{lem:terminal-tracking}]
The proof follows two consequences of the index rule. First, an arm
cannot receive a pull that takes its count substantially above its
terminal target: at that selection time, its index would already be
smaller than another arm's index. Second, if the suboptimal arms
collectively remain substantially below their terminal targets,
then the last optimal-arm selection should instead have gone to
one of those suboptimal arms.

Similar as in the last proof, we denote \(q_a = \tilde n_{a, T}\) to be the empirical fluid approximation, and $I_{a,t}$ to be the index value at the beginning of time $t$. Throughout, we use the following bounds. For all $T$ sufficiently large and on $\mathcal{G}_T$, Lemmas~\ref{lem:ucb-fluid-approximation} and
\ref{lem:empirical_fluid_geometry} imply that
\[
\frac{T}{2K}\le q_i\le T,
\qquad i\in\mathcal O,
\qquad
\frac{f_T^2}{2\Delta_j^2}\le q_j
\le\frac{2f_T^2}{\Delta_j^2},
\qquad j\in\mathcal S,
\]
and \(\frac12 n_{a,T}\le q_a\le2n_{a,T}\) for \( a\in[K].\) Throughout the proof, we take $\eta = 1/64, \underline c = 1/16, \bar c = 4$. We then choose a sufficiently large constant $B \ge 1$ below. For $B$ fixed, we take $T$ sufficiently large. Similar as in the last proof, the above bounds combining with the definition of  $\mathcal G_T$ yield \(|\bar Y_a(n_{a,T})|
\le\frac{C\ell_T}{\sqrt{q_a}}\). Constants $C$ depends only on the fixed bandit instance and the algorithm, and are independent of $B$ and $T$.

\medskip
\noindent
\textbf{Controlling overpull.}  We first prove that there exists a sufficiently large fixed constant $B > 200$, such that for all $T$ sufficiently large,
\begin{equation}
N_a(T)-q_a
\le
B\frac{q_a d_{0,T}}{f_T^2},
\qquad a\in[K].
\label{eq:terminal-tracking-upper}
\end{equation}
We argue by contradiction. Suppose for some arm $a$  \eqref{eq:terminal-tracking-upper} fails. There is a first time $u$ at which arm $a$ is selected by the algorithm, and \eqref{eq:terminal-tracking-upper} is violated. At the beginning of time $u$, we have \(N_a(u - 1) = \lfloor q_a + B q_a d_{0,T} f_T^{-2}\rfloor\). Put $k_a = N_a(u-1)$. Since for every fixed $B$, $d_{0, T} = o(f_T^2)$, eventually for sufficiently large $T$ we have 
\[
k_a - q_a \ge \frac{B}{2}\frac{q_a d_{0,T}}{f_T^2} > 0, \qquad q_a < k_a \le 2 q_a.
\]
In that way, we have located a time $u$ and an arm $a$ with excessive pulls at $u$, but still remains close enough to the empirical fluid approximation $q_a$. Notice also that $u - 1 \ge k_a > q_a$, so $u$ also diverges as $T$ grows.

We apply Lemma~\ref{lem:terminal-index-comparison} to show that there exists an arm with larger index at time $u$, forcing a contradiction. We first check the conditions in Lemma~\ref{lem:terminal-index-comparison}. For an optimal arm, $q_a \asymp T$. Hence $u > k_a > q_a$ must also be a fixed fraction of $T$. By Lemma~\ref{lem:f_scaling}, Lemma~\ref{lem:ucb-fluid-approximation} and Lemma~\ref{lem:empirical_fluid_geometry}, we have, for $T$ sufficiently large,
\[
\frac{f_u^2}{f_T^2} \ge \frac{u}{T} > \frac{q_a}{T} \ge \frac{1}{2 |\mathcal{O}|} \ge \frac{1}{2 K}.
\]
For a suboptimal arm, also for sufficiently large $T$, \(u > q_a \ge \frac{1}{2}n_{a, T} \ge \frac{1}{2}\min_{j \in \mathcal{S}} n_{j, T}. \) At such a time $u$, on event $\mathcal{G}_T$, the index of an optimal arm $i$  satisfies $I_{i, u} > \mu^{\star} - \frac{1}{4}\min_{j \in \mathcal{S}} \Delta_j.$ Since arm $a$ is pulled at time $u$, it then follows that \(I_{a, u} > \mu^{\star} - \frac{1}{4}\min_{j \in \mathcal{S}} \Delta_j.\) More precisely,
\begin{align*}
     \mu^{\star} - \frac{1}{4}\min_{j \in \mathcal{S}} \Delta_j <  \mu_a + \bar Y_a(k_a) + \frac{f_u}{\sqrt{k_a}} \le \mu_a + \frac{\ell_T + f_u}{\sqrt{k_a}},
\end{align*}
where the definition of $\mathcal{G}_T$ is used in the second inequality. The above  implies  
\begin{align*}
    f_u > \frac{3}{4} \Delta_a\sqrt{k_a} - \ell_T \ge \frac{\Delta_a}{2}\sqrt{q_a} \ge\frac{f_T}{2\sqrt{2}},
\end{align*}
where we used $\Delta_a \ge \min_{j \in \mathcal{S}} \Delta_j$, $k_a > q_a \ge f_T^2/(2 \Delta_a^2)$ and $\ell_T = o(f_T)$.

Thus in either optimal arm and suboptimal arm case, we have $f_u \ge f_T/(4 \sqrt{K})$, which allows us to use Lemma~\ref{lem:terminal-index-comparison}, guaranteeing an arm $a'$ whose index at time $u$ exceeds 
\[\tilde\lambda_T - d_{0,T}/(f_T\sqrt{T}).\]
We next prove that arm $a$ has strictly smaller index, forcing the contradiction.

The empirical-fluid equation, together with the definition of $I$ and $u$ give
\[
I_{a,u}-\tilde\lambda_T
=
\bar Y_a(k_a)-\bar Y_a(n_{a,T})
+\frac{f_u}{\sqrt k_a}-\frac{f_T}{\sqrt{q_a}}.
\]
First, consider the case of $k_a \le 4 q_a$. In this case, we have \(q_a < k_a \le 4 q_a\). We apply Lemma~\ref{lem:terminal-mean-comparison} with $\underline c=1/16$, $\overline c=4$, and the fixed $\eta\in(0,1/32)$ to  get, for sufficiently large $T$,
\[
\left|\bar Y_a(k_a) - \bar Y_a(n_{a, T})\right| \le \frac{\eta f_T}{q_a^{3/2}}\left(|k_a - q_a| + \frac{q_a d_{0, T}}{f_T^2}\right).
\]
We then use convexity of $x \mapsto \frac{1}{\sqrt{x}}$ to get, for $k_a \in [q_a, 4 q_a]$, 
\(\frac1{\sqrt k_a}-\frac1{\sqrt{q_a}}
\le
-\frac{k_a-q_a}{16q_a^{3/2}}.\) Combining the above facts, 
\[
\begin{aligned}
I_{a,u}-\tilde\lambda_T
&\le
\left|\bar Y_a(k_a)-\bar Y_a(n_{a,T})\right|
+
f_T\left(\frac1{\sqrt k_a}-\frac1{\sqrt{q_a}}\right)\\
&\le
-\left(\frac1{16}-\eta\right)
\frac{f_T(k_a-q_a)}{q_a^{3/2}}
+
\eta\frac{d_{0,T}}{f_T\sqrt{q_a}}\\
&\le
-\left[
\frac B2\left(\frac1{16}-\eta\right)-\eta
\right]
\frac{d_{0,T}}{f_T\sqrt{q_a}},
\end{aligned}
\]
where in the first inequality we used $f_u \le f_T$. Since $\eta < 1/32$, our choice of $B \ge 200$ makes the coefficient in square bracket at least $2$. Thus, 
\(I_{a,u}-\tilde\lambda_T \le -\frac{2 d_{0,T}}{f_T\sqrt{T}}.\)

If instead, $k_a > 4 q_a$. Similar as in the rough bounding of $k_a$ in the proof of Lemma~\ref{lem:terminal-index-comparison}, we have
\[
\begin{aligned}
   |\bar Y_a(k_a)|\le\frac{\ell_T}{\sqrt{k_a}},
\qquad
|\bar Y_a(n_{a,T})|\le\frac{C\ell_T}{\sqrt{q_a}}, 
\end{aligned}
\]
which yields \(I_{a, u} - \tilde \lambda_T \le \frac{\ell_T + f_u}{\sqrt{k_a}} + \frac{C\ell_T - f_T}{\sqrt{q_a}}.\)  Combining the facts that $\ell_T = o(f_T)$, that $d_{0,T} = o(f_T^2)$, that $k_a \ge 4 q_a$, that $f_u \le f_T$, and that $q_a \le T$, the above bound then yields \(I_{a, u} - \tilde \lambda_T \le - \frac{2 d_{0,T}}{f_T \sqrt{T}}\) for $T$ sufficiently large.

In both cases above, $I_{a, u}$ does not exceed \(\tilde\lambda_T-\frac{2d_{0,T}}{f_T\sqrt T}.\) This, together with the larger-index arm $a'$ at time $u$ guaranteed by Lemma~\ref{lem:terminal-index-comparison}, contradicts the selection of arm $a$ at time $u$ under the algorithm. This then proves \eqref{eq:terminal-tracking-upper}.

For an optimal arm $i$, conservation of the total pull, together with \eqref{eq:terminal-tracking-upper}, give
\[
q_i - N_i(T)  = \sum_{a \neq i}  (N_a(T) - q_a) \le B \frac{d_{0,T}}{f_T^2} \sum_{a \neq i} q_a = B\frac{(T - q_i)d_{0,T}}{f_T^2}.
\]
Together with a direct use of \eqref{eq:terminal-tracking-upper}, and $q_i \le T$, for the opposite bound, we thus get
\[ |N_i(T)-q_i|
\le B\frac{T d_{0,T}}{f_T^2},
\qquad i\in\mathcal O.\]
This completes the proof of the optimal arm's case. 

For a suboptimal arm $j$, \eqref{eq:terminal-tracking-upper} implies
\[
N_j(T) - q_j \le B\frac{q_j d_{0,T}}{f_T^2} = Bd_{0,T} \cdot \frac{n_{j,T} }{f_T^2}\cdot \frac{q_j}{n_{j,T}} \le \frac{2 B}{\Delta_j^2} d_{0,T}
\]
for all $T$ sufficiently large, following Lemma~\ref{lem:ucb-fluid-approximation} and Lemma~\ref{lem:empirical_fluid_geometry}. This completes the proof of the suboptimal arm's case on its excessive side. It remains to control the case of underpull for a suboptimal arm.

\medskip
\noindent
\textbf{Controlling suboptimal underpull.}  The idea is to compare index at the last pull of an optimal arm. Concretely, we prove that there exists constant $B$ such that
\begin{equation}
\sum_{j\in\mathcal S}\bigl(q_j-N_j(T)\bigr)
\le B d_{0,T}.
\label{eq:terminal-tracking-total-deficit}
\end{equation}
We argue by contradiction. Suppose the bound in \eqref{eq:terminal-tracking-total-deficit} fails so that \(\sum_{j\in\mathcal S}\bigl(q_j-N_j(T)\bigr) > B d_{0,T}.\) Let $u$ denote the last time an optimal arm is pulled. Let the pulled arm be $i \in \mathcal{O}$ at time $u$. Let $h = T - u$ denote the remaining pulls. Then all $h$ pulls in time $[u+1, T]$ are suboptimal. Hence
\[
h \le \sum_{j \in \mathcal{S}} N_j(T) < \sum_{j \in \mathcal{S}} q_j \le C f_T^2,
\]
where the final bound follows from Lemma~\ref{lem:ucb-fluid-approximation}, Lemma~\ref{lem:empirical_fluid_geometry} and $|\mathcal{S}| \le K.$ Since $f_T^2 = o(T)$, the bound then implies $u/T \to 1 $, and in particular, $u \ge T/2$ for all sufficiently large $T$. 

Now we pick the target suboptimal arm $j \in \mathcal{S}$. We have
\[
\sum_{a\in\mathcal S}\bigl(q_a-N_a(u-1)\bigr)
=
\sum_{a\in\mathcal S}\bigl(q_a-N_a(T)\bigr)+h
>
Bd_{0,T}+h.
\]
Here we used the fact that between time $u$ and time $T$, exactly $h$ additional pulls are on suboptimal arms. Hence by the pigeonhole principle, there exists an arm $j$, its number of pull $k_j = N_j(u-1)$ must satisfy \(q_j-k_j>\frac{Bd_{0,T}+h}{|\mathcal S|}.\). We show that this arm should have been selected
instead of arm $i$.

Indeed, we compute the optimal arm $i$'s index at time $u$. Let $k_i = N_i(u-1)$. Since time $u$ is its last pull, we have $k_i = N_i(T)  - 1.$ We have already proved that for optimal arm, \( |N_i(T) - q_i|\) is bounded by $B T d_{0, T} (f_T)^{-2}.$ Hence for $T$ sufficiently large, we have 
\[
|k_i - q_i| \le 2 B \frac{T d_{0, T}}{f_T^2}, \qquad q_i/2 \le k_i \le 2 q_i.
\]
The relative bounds follow from the additive bounds and $d_{0, T} = o(f_T^2)$ and $q_i \ge T/(2 K)$. Therefore, $k_i$ is well controlled to be near $q_i.$  Applying the same sample-mean comparison and mean-value estimate as above, in particular using Lemma~\ref{lem:terminal-mean-comparison} with $\eta = 1/64$, for which we omit the details here that are largely repetitive, gives
\begin{equation}
I_{i,u}-\tilde\lambda_T \le 3(2 K)^{3/2} B\frac{d_{0,T}}{f_T\sqrt T}.\label{eq:terminal-tracking-selected-optimal-index}
\end{equation}

We now show that the index of arm $j$ is larger. We prove that 
\begin{equation}
I_{j,u}-\widetilde\lambda_T
\ge
\frac{\Delta_j^3}{16}\frac{q_j-k_j}{f_T^2}
-\frac{\Delta_j}{32}\frac{d_{0,T}}{f_T^2}
-6\Delta_j\frac hT.
\label{eq:terminal-tracking-deficient-index}
\end{equation}

Indeed, we consider two cases. First suppose $q_j/16 \le k_j < q_j.$ Apply Lemma~\ref{lem:terminal-mean-comparison} with $\eta = 1/64$, and using convexity of $x\mapsto x^{-1/2}$, gives
\begin{align*}
I_{j,u}-\widetilde\lambda_T
&\ge
\left(\frac12-\frac1{64}\right)
\frac{f_T(q_j-k_j)}{q_j^{3/2}}
-\frac{d_{0,T}}{64f_T\sqrt{q_j}}
-\frac{f_T-f_u}{\sqrt {k_j}}\\
&\ge
\frac{f_T(q_j-k_j)}{4q_j^{3/2}}
-\frac{d_{0,T}}{64f_T\sqrt{q_j}}
-\frac{4f_T h}{T\sqrt{q_j}}.
\end{align*}
Here the last term uses $k_j\ge q_j/16$ and the fact that $0\le f_T-f_u\le \frac{f_T h}{T}$ following Lemma~\ref{lem:f_scaling}. 

In the other case $k_j < q_j/16$. Then arm $j$ is significantly underpulled, and we can apply the rough bound on the empirical mean supplied in the definition of $\mathcal{G}_T$ directly, and have 
\[
I_{j,u}-\widetilde\lambda_T
\ge
\frac{f_u-\ell_T}{\sqrt k_j}
-\frac{f_T}{\sqrt{q_j}}
-\left|\bar Y_j(n_{j,T})\right|
\ge
\frac{f_T}{2\sqrt{q_j}}
\]
for all sufficiently large $T$. The last bound uses $\ell_T = o(f_T), f_u/f_T \to 1$ since Lemma~\ref{lem:f_scaling} and $u/T \to 1$, and $\left|\bar Y_j(n_{j,T})\right| \le C\ell_T/\sqrt{q_j}.$ Observe that \(\frac{f_T}{2\sqrt{q_j}} \ge \frac{f_T(q_j - k_j)}{4 q_j^{3/2}}\), hence this is a stronger bound than the one proved in the previous case. Combining both cases with Lemma~\ref{lem:ucb-fluid-approximation} that guarantees $f_T^2/(2\Delta_j^2)\le q_j\le2f_T^2/\Delta_j^2$, we've done with the proof of \eqref{eq:terminal-tracking-deficient-index}.  Recall that we choose $j$ such that $q_j - k_j > \frac{1}{|\mathcal{S}|}(B d_{0, T} + h)$. Plugging into \eqref{eq:terminal-tracking-deficient-index} yield
\[
I_{j, u} - \tilde \lambda_T \ge \left(
\frac{\Delta_j^3 B}{16|\mathcal S|}
-\frac{\Delta_j}{32}
\right)\frac{d_{0,T}}{f_T^2} +
\left(
\frac{\Delta_j^3}{16|\mathcal S|}
-6\Delta_j\frac{f_T^2}{T}
\right)\frac h{f_T^2}.
\]
We choose $B > \max\{200, |\mathcal{S}|/\Delta_j^2\}$ to make the coefficient of $d_{0,T}(f_T)^{-2}$ positive. With this choice of $B$, and for $T$ sufficiently large, with the fact that $h = O(f_T^2)$, $f_T^2 = o(T)$, we may compare \eqref{eq:terminal-tracking-selected-optimal-index} with the above bound on $I_{j, u}$ and conclude that for $T$ sufficiently large, $I_{j, u} > I_{i, u}$. This then contradicts the selection of arm $i$ at time $u$, and in turn proves \eqref{eq:terminal-tracking-total-deficit}.  
With \eqref{eq:terminal-tracking-total-deficit}, we then have for every $a\in\mathcal S$,
\[
\begin{aligned}
q_a-N_a(T)
&=
\sum_{b\in\mathcal S}\bigl(q_b-N_b(T)\bigr)
+
\sum_{\substack{b\in\mathcal S; b\ne a}}
\bigl(N_b(T)-q_b\bigr)\\
&\le
Bd_{0,T}
+
B\frac{d_{0,T}}{f_T^2}
\sum_{\substack{b\in\mathcal S; b\ne a}}q_b \le CBd_{0,T}.
\end{aligned}
\]
Absorbing the constants proves the underpull supoptimal-arm conclusion. Combining the above completes the proof of the lemma.
\end{proof}

\begin{lemma}
\label{lem:nonlocal_suboptimal_tracking}
For the fixed $\gamma\in(0,\beta)$, there exist constants
$C\ge1$ and $T_0<\infty$ such that, for every integer $T\ge T_0$,
every $j\in\mathcal S$, and every real $C f_T^2\le x\le T$,
\[
\mathbb P\left\{
|N_j(T)-\tilde n_{j,T}|>x
\right\}
\le x^{-\gamma}.
\]
\end{lemma}

\begin{proof}[Proof of Lemma~\ref{lem:nonlocal_suboptimal_tracking}]
Note that 
\[
\left\{
|N_j(T)-\tilde n_{j,T}|>x
\right\}
\subseteq
\{N_j(T)>x\}\cup\{\tilde n_{j,T}>x\}.
\]
We show that both events $\{N_j(T)>x\}$ and $\{\tilde n_{j,T}>x\}$ have small probability that satisfy the lemma. Fix an optimal arm $i \in \mathcal{O}$. Choose $q > \gamma$ and $\varepsilon \in (0, 1)$ such that $q < (1 - \varepsilon)^3 \beta.$ Choose a fixed constant $A \ge 1$ sufficiently large that  \( A \ge \max\{32, 1/\varepsilon^2\} (\min_{j \in \mathcal{S}}\Delta_j)^{-2}. \) For a sufficiently large $T$, an arm $j \in \mathcal{S}$, we bound the event involving $N_j(T)$ and $\tilde n_{j, T}$ for any $A f_T^2 \le x \le T$ respectively.

\medskip
\noindent
\textbf{Case of $N_j.$\ }
We slightly abuse notation and assume $x$ only takes integer values. This is without loss of generality for our purpose. Then on the event $\{N_j(T)>x\}$, there is a decision time $u\le T$ at which arm $j$ is selected, with $N_j(u-1) = x$. Denote by $k_i = N_i(u-1)$. Compare the index of $j$ and $i$, namely, $I_{j,u} \ge I_{i, u}$ and we get
\[
\bar Y_j(x) - \bar Y_i(k_i) \ge \Delta_j + \frac{f_u}{\sqrt{k_i}} - \frac{f_u}{\sqrt{x}} \ge \frac{3 \Delta_j}{4} + \frac{f_x}{\sqrt{k_i}},
\]
where the last bound follows from the choice of $A$. Therefore, 
\begin{align*}
&\mathbb P\{N_j(T)>x\}\\
&\qquad\le \PP\left(\bar Y_j(x)\ge\Delta_j/4\right) + \PP\left(-\bar Y_{i}(k_i)
\ge\frac{\Delta_j}{2}+\frac{f_x}{\sqrt k_i}\right), \\
&\qquad\le
\exp\left\{-\frac{x \Delta_j^2}{32\nu^2}\right\}
+
\sum_{n=1}^{\infty}
\exp\left\{
-\frac{(\Delta_j\sqrt n/2+f_x)^2}{2\nu^2}
\right\},\\
&\qquad\le
\exp\left\{-\frac{x \Delta_j^2}{32\nu^2}\right\} 
+\exp\left\{-\frac{f_x^2}{2\nu^2}\right\}\frac{1}
{\exp\{\Delta_{j}^2/(8\nu^2)\}-1}.
\end{align*}
Here the inequality follows from sub-Gaussian concentration and a union bound. By definition of $\beta$ and the fact that $q < \beta$, we have $f_x^2/(2 \nu^2) \ge q \log x$ for all sufficiently large $x$. The first term in the above right hand side is exponential in $x$. Thus, for some fixed constant $C$, when $T$ is sufficiently large, uniformly for all $ A f_T^2 \le x \le T$, we have   \(\mathbb P\{N_j(T)>x\}\le Cx^{-q}.\)

\medskip
\noindent
\textbf{Case of $\tilde n_j$.\ }
Let $q_a = \tilde n_{a, T}.$ Subtracting the empirical fluid equations for $j$ and $i$ gives
\[
\bar Y_j(n_{j,T})-\bar Y_{i}(n_{i,T})
=
\Delta_j
+\frac{f_T}{\sqrt{q_i}}
-\frac{f_T}{\sqrt{q_j}}.
\]
On the event $\{\tilde n_{j,T}>x\}$, the choice of $A$
therefore implies
\[
\bar Y_j(n_{j,T})-\bar Y_{i}(n_{i,T})
\ge
\Delta_j-\frac{f_T}{\sqrt x}
\ge(1-\varepsilon)\Delta_j.
\]

Here $n_{i, T}$ and $n_{j, T}$ are deterministic. By Lemma~\ref{lem:ucb-fluid-approximation}, those sizes are asymptotic to $T/|\mathcal O|$ and $f_T^2/\Delta_j^2$, respectively. Since $f_T^2/T\to0$, the sum of their reciprocals is at most \(\frac{\Delta_j^2}{(1-\varepsilon)f_T^2}\) for all sufficiently large $T$.
Independence across arms and sub-Gaussianity consequently give
\begin{align*}
\mathbb P\{\tilde n_{j,T}>x\}
&\le
\mathbb P\left\{
\bar Y_j(n_{j,T})-\bar Y_{i}(n_{i,T})
\ge(1-\varepsilon)\Delta_j
\right\}\\
&\le
\exp\left\{
-\frac{(1-\varepsilon)^3 f_T^2}{2\nu^2}
\right\} \le T^{-q}
\le x^{-q}.
\end{align*}
The third inequality follows from
$q<(1-\varepsilon)^3\beta$ and the definition of $\beta$;
the last uses $x\le T$.

Combining the preceding bounds yields
\[
\mathbb P\left\{
|N_j(T)-\tilde n_{j,T}|>x
\right\}
\le Cx^{-q}.
\]
Since $q>\gamma$ and $x\ge A f_T^2\to\infty$, increasing
$T_0$ ensures that $Cx^{-q}\le x^{-\gamma}$ throughout the
stated range. Taking $C=A$ as the cutoff constant in the
lemma completes the proof.
\end{proof}

With the above lemmas, we are ready to complete the proof of Lemma~\ref{lem: ultra tail}. 

\begin{proof}[Proof of Lemma~\ref{lem: ultra tail}]
Fix $\gamma\in(0,\beta)$. Choose $A\ge1$ at least as large as the constants in Lemmas~\ref{lem:terminal-tracking} and \ref{lem:nonlocal_suboptimal_tracking}. Choose $T_0$ sufficiently large that these lemmas and Corollary~\ref{lem:terminal-reward-event} apply, and that $d_{0,T}\le f_T^2$ for every $T\ge T_0$.

Fix an integer $T\ge T_0$. First consider \(A d_{0,T}\le x\le \min\{A f_T^2,T\}.
\) On $\mathcal G_T$, Lemma~\ref{lem:terminal-tracking}  implies that the event in \eqref{eq:ultra tail} is contained in $\mathcal G_T^c$. By
Corollary~\ref{lem:terminal-reward-event}, its probability is at most
\[
\mathbb P(\mathcal G_T^c)
\le f_T^{-2\gamma}
\le A^\gamma x^{-\gamma},
\]
where the last inequality uses $x\le A f_T^2$.

Next consider \(A f_T^2<x\le T.\) For every optimal arm, both $N_i(T)$ and $\tilde n_{i,T}$ belong to $[0,T]$, so \(\frac{f_T^2}{T} |N_i(T)-\tilde n_{i,T}| \le f_T^2<x. \)  Consequently, only suboptimal arms can contribute to the event
in \eqref{eq:ultra tail}. A union bound and
Lemma~\ref{lem:nonlocal_suboptimal_tracking} give 
\[ \mathbb P\left\{ \max_{j\in\mathcal S} |N_j(T)-\tilde n_{j,T}|>x \right\} \le \sum_{j\in\mathcal S} \mathbb P\left\{ |N_j(T)-\tilde n_{j,T}|>x
\right\} \le |\mathcal S|x^{-\gamma}.
\]
Combining the two ranges and taking \(C =\max\{A,A^\gamma,|\mathcal S|\}\) completes the proof of Lemma~\ref{lem: ultra tail}. All constants are independent of $T$
and $x$.
\end{proof}

\ycedit{

\section{ Proof of Lemma \ref{lem:lower-count-tails}}}

\begin{proof}[Proof of Lemma~\ref{lem:lower-count-tails}]
\jledit{The case of $K=1$ is trivial. Suppose $K\geq 2$.} Set $c=1/[8(K-1)]$. We use the following observation: if an arm
has at most $T/4$ pulls by time $T$, then there must exist another arm that is pulled
at least $T/[4(K-1)]$ times during
$\lceil T/2\rceil,\ldots,T$. Immediately before the last of
those pulls, this other arm must have at least $cT$ observations,
for all sufficiently large $T$. Comparing their indices at that
time forces a downward reward fluctuation on the underpulled arm,
unless the other arm has an unusually large sample mean.

\medskip
\noindent\emph{Suboptimal arms.}
Fix $j\in\mathcal S$ and put $q_T=n_{j,T}/2$.
Lemma~\ref{lem:ucb-fluid-approximation} gives
$q_T\le f_T^2/(2\Delta_j^2)=o(T)$.
By Lemma~\ref{lem:f_scaling}, there is a fixed $b>0$ such that,
for all sufficiently large $T$,
\[
\Delta_j\sqrt{q_T}\le(1-2b)f_{T/2},
\qquad
\frac{f_T}{\sqrt{cT}}\le b\Delta_j.
\]
Indeed, one may take
$b=(1-2^{(\theta-1)/2})/2$, where $\theta<1$ is supplied by
that lemma; the second inequality follows from $f_T/\sqrt T\to0$.

On $\{N_j(T)<q_T\}$, the preceding observation gives $N_j(T)<T/4$ for all sufficiently large $T$, and there exists an arm
$a\ne j$ selected at some $t\ge T/2$, with
$m=N_a(t-1)\ge cT$ and $1\le n=N_j(t-1)<q_T$.
The index comparison, using $\mu_a\le\mu^\star$, yields
\begin{align*}
\bar Y_a(m)-\bar Y_j(n)
&\ge-\Delta_j+\frac{f_{T/2}}{\sqrt n}
-\frac{f_T}{\sqrt{cT}}
\ge b\frac{f_{T/2}}{\sqrt n}.
\end{align*}
Since $f_{T/2}/\sqrt n\ge\Delta_j$, this implies either
\[
\bar Y_a(m)\ge\frac{b\Delta_j}{4},
\qquad\text{or}\qquad
-\bar Y_j(n)\ge
\frac{b\Delta_j}{4}+\frac{bf_{T/2}}{2\sqrt n}.
\]
For the first event, $m\ge cT$ implies
$S_a(m)\ge cb\Delta_j T/4$.
Applying Lemma~\ref{lem:maximal inequality} to this event and
Corollary~\ref{lemma:exact-decision-time-bonus} to the second,
then taking a union bound over $a\ne j$, gives
\begin{align*}
\mathbb P\{N_j(T)<q_T\}
&\le
(K-1)\exp\left\{-\frac{c^2b^2\Delta_j^2T}{32\nu^2}\right\}
+\frac{\exp\{-b^2f_{T/2}^2/(8\nu^2)\}}
{\exp\{b^2\Delta_j^2/(32\nu^2)\}-1}
\\
&\le C\exp\{-c_1f_T^2\}
\end{align*}
for some constants $C, c_1>0$.
The last inequality uses $f_{T/2}^2\ge2^{-\theta}f_T^2$
and $f_T^2=o(T)$. Since there are finitely many suboptimal arms,
the constants $C, c_1$ can be chosen uniformly in $j$. This proves
\eqref{eq:suboptimal-lower-tail}.

\medskip
\noindent\emph{Optimal arms.}
Fix $1<\gamma<\beta$. Choose $\bar\gamma\in(\gamma,\beta)$
and $\varepsilon\in(0,1)$ such that
$\gamma<(1-\varepsilon)\bar\gamma$, and set
\[
\kappa=\frac{c^2\varepsilon^2}{64}.
\]
It's easy to verify that $\kappa<1/4$. Fix $i\in\mathcal O$. On $\{N_i(T)\le\kappa T\}$, the same
observation earlier gives an arm $a\ne i$ selected at some $t\ge T/2$,
with $1\le n=N_i(t-1)\le\kappa T$ and $cT\le m=N_a(t-1)\le T$.
Since $\kappa<c$ and $\mu_i\ge\mu_a$, the index comparison gives
\[
\bar Y_a(m)-\bar Y_i(n)
\ge f_{T/2}\left(\frac1{\sqrt n}-\frac1{\sqrt m}\right).
\]

The well-sampled arm's fluctuation can be controlled by a single
maximal inequality. Except on an event of probability at most
$(K-1)\exp\{-f_{T/2}^2/(2\nu^2)\}$,
Lemma~\ref{lem:maximal inequality} gives, simultaneously for
every $a\ne i$ and $cT\le m\le T$,
\begin{align}\label{eq:event_1}
    S_a(m)<f_{T/2}\sqrt T,
\qquad
\bar Y_a(m)<\frac{f_{T/2}}{c\sqrt T}.
\end{align}
The index comparison therefore forces
\begin{align*}
-\bar Y_i(n)
&\ge\frac{f_{T/2}}{\sqrt n}
\left(1-\sqrt{\frac\kappa c}-\frac{\sqrt\kappa}{c}\right)
\ge\left(1-\frac\varepsilon4\right)
\frac{f_{T/2}}{\sqrt n}
\end{align*}
for some $1\le n\le\kappa T$.

We now bound the probability of this event by dividing the range of $n$ into blocks. 
Put $b=1+\varepsilon/2$ and partition the possible integers $n$
into blocks $b^j\le n<b^{j+1}$, $j\ge0$.
On each block, Lemma~\ref{lem:maximal inequality}, applied to
$-S_i$ up to $\lfloor b^{j+1}\rfloor$, gives
\begin{align*}
&\mathbb P\left\{
\exists n\in[b^j,b^{j+1})\cap\mathbb N:
-S_i(n)\ge\left(1-\frac\varepsilon4\right)f_{T/2}\sqrt n
\right\}
\\
&\qquad\le
\exp\left\{-\frac{(1-\varepsilon/4)^2f_{T/2}^2}
{2\nu^2(1+\varepsilon/2)}\right\}
\le
\exp\left\{-\frac{(1-\varepsilon)f_{T/2}^2}{2\nu^2}\right\}.
\end{align*}
There are at most $1+\log T/\log b$ relevant blocks.
Apply the union bound and combine with \eqref{eq:event_1}, we have
\begin{align*}
\mathbb P\{N_i(T)\le\kappa T\}&\le (K+\log T/\log b)\exp\left\{-\frac{(1-\varepsilon)f_{T/2}^2}{2\nu^2}\right\}\le C(1+\log T)
\exp\left\{-\frac{(1-\varepsilon)f_{T/2}^2}{2\nu^2}\right\}
\end{align*}
for some constant $C>0$.
Finally, the definition of $\beta$ implies
$f_{T/2}^2/(2\nu^2)\ge\bar\gamma\log T$ for all sufficiently
large $T$. Thus
\[
\mathbb P\{N_i(T)\le\kappa T\}
\le C(1+\log T)T^{-(1-\varepsilon)\bar\gamma}
\le C T^{-\gamma}.
\]
The constants are uniform over $i\in\mathcal O$, proving
\eqref{eq:optimal-lower-tail}.
\end{proof}

\section{Proof of Proposition~\ref{prop:unique arm sharper bias}}
\begin{proof}[Proof of Proposition~\ref{prop:unique arm sharper bias}]
Fix \(1<\eta<\gamma<\beta\) and \(1<p<\gamma\). Lemma~\ref{lem:ucb-fluid-approximation} gives
\[
n_{i,T}\asymp T,
\qquad
n_{j,T}\asymp f_T^2,
\qquad j\in \mathcal S.
\]
For every \(j\in \mathcal S\), tail integration in
Lemma~\ref{lem: ultra tail} and Lemma~\ref{lem:empirical-fluid-concentration} together imply
\begin{equation}
\left\|
N_i(T)-n_{i,T}
\right\|_p
=
O(f_T).
\label{eq:unique-count-moment-simple}
\end{equation}

Lemma~\ref{lem:lower-count-tails}\textup{(ii)} gives a constant \(\kappa\in(0,1)\) such that
\begin{equation}
\mathbb P\left\{
N_i(T)\leq\kappa T
\right\}
=
O(T^{-\gamma}).
\label{eq:unique-low-count-simple}
\end{equation}

Standard maximal inequality Lemma~\ref{lem:maximal inequality} with a dyadic decomposition give the following fixed \(q<\infty\) moment bounds
\begin{equation}
\begin{gathered}
\left\|
\max_{\kappa T<k\leq T}
\left|\bar X_i(k)-\mu_i\right|
\right\|_q
=
O(T^{-1/2}),
\\
\left\|
\max_{1\leq k\leq T}
\left|\bar X_i(k)-\mu_i\right|
\right\|_q
=
O(1),
\\
\left\|
\max_{1\leq k\leq T}
\sqrt{k}\left|\bar X_i(k)-\mu_i\right|
\right\|_q
=
O(\sqrt{\log T}).
\end{gathered}
\label{eq:unique-maximal-bounds}
\end{equation}

Applying Lemma~\ref{lem:identities} with \(l=n_{i,T}\) yields
\begin{equation}
\begin{aligned}
\mathbb E\left[
\bar X_i(N_i(T))-\mu_i
\right]
=
-\frac{1}{n_{i,T}}
\mathbb E\left[
\bigl(N_i(T)-n_{i,T}\bigr)
\bigl(\bar X_i(N_i(T))-\mu_i\bigr)
\right].
\end{aligned}
\label{eq:unique-raw-identity-simple}
\end{equation}
On \(\{N_i(T)>\kappa T\}\), Hölder's inequality,
\eqref{eq:unique-count-moment-simple}, and
\eqref{eq:unique-maximal-bounds} give
\[
\begin{aligned}
&
\frac{1}{n_{i,T}}
\mathbb E\left[
\left|N_i(T)-n_{i,T}\right|
\left|\bar X_i(N_i(T))-\mu_i\right|
\mathbf 1\{N_i(T)>\kappa T\}
\right]
\\
&\qquad\leq
\frac{1}{n_{i,T}}
\left\|N_i(T)-n_{i,T}\right\|_p
\left\|
\max_{\kappa T<k\leq T}
\left|\bar X_i(k)-\mu_i\right|
\right\|_{p/(p-1)} =
O\left(\frac{f_T}{T^{3/2}}\right).
\end{aligned}
\]
On \(\{N_i(T)\leq\kappa T\}\), since \(N_i(T)\leq T\) and
\(n_{i,T}\asymp T\),
\[
\frac{\left|N_i(T)-n_{i,T}\right|}{n_{i,T}}
\left|\bar X_i(N_i(T))-\mu_i\right|
\leq
C\max_{1\leq k\leq T}
\left|\bar X_i(k)-\mu_i\right|.
\]
Choose \(q>1\) sufficiently large that \(\gamma\left(1-\frac{1}{q}\right)>\eta.\) Hölder's inequality, \eqref{eq:unique-low-count-simple}, and
\eqref{eq:unique-maximal-bounds} then show that the contribution of
\(\{N_i(T)\leq\kappa T\}\) is
\[
O\left(
T^{-\gamma(1-1/q)}
\right)
=
O(T^{-\eta}).
\]
Substitution into \eqref{eq:unique-raw-identity-simple} proves
\[
\left|
\mathbb E\left[
\bar X_i(N_i(T))-\mu_i
\right]
\right|
=
O\left(
\frac{f_T}{T^{3/2}}
+
T^{-\eta}
\right).
\]

The argument for standardized statistic is similar. Lemma~\ref{lem:identities} gives
\begin{equation}
\begin{aligned}
\sigma_i\mathbb E[Z_{i,T}]
=
\mathbb E\left[
\left(
\sqrt{N_i(T)}
-
\frac{N_i(T)}{\sqrt{n_{i,T}}}
\right)
\bigl(\bar X_i(N_i(T))-\mu_i\bigr)
\right].
\end{aligned}
\label{eq:unique-z-identity-simple}
\end{equation}
On \(\{N_i(T)>\kappa T\}\), since \(n_{i,T}\asymp T\),
\[
\left|
\sqrt{N_i(T)}
-
\frac{N_i(T)}{\sqrt{n_{i,T}}}
\right|
\leq
C\frac{\left|N_i(T)-n_{i,T}\right|}{\sqrt T}.
\]
Therefore, by Hölder's inequality,
\eqref{eq:unique-count-moment-simple}, and
\eqref{eq:unique-maximal-bounds}, this event contributes \(O\left(\frac{f_T}{T}\right)\)
to the right-hand side of \eqref{eq:unique-z-identity-simple}.

On \(\{N_i(T)\leq\kappa T\}\),
\[
\begin{aligned}
&
\left|
\sqrt{N_i(T)}
-
\frac{N_i(T)}{\sqrt{n_{i,T}}}
\right|
\left|\bar X_i(N_i(T))-\mu_i\right| \leq
C
\max_{1\leq k\leq T}
\sqrt{k}\left|\bar X_i(k)-\mu_i\right|.
\end{aligned}
\]
Hence \eqref{eq:unique-low-count-simple},
\eqref{eq:unique-maximal-bounds}, and Hölder's inequality give
\[
O\left(
\sqrt{\log T}\,
T^{-\gamma(1-1/q)}
\right)
=
O(T^{-\eta}).
\]
Since \(\eta>1\) and \(f_T\to\infty\), \(T^{-\eta}
=
o\left(\frac{f_T}{T}\right).\) Plugging into Equation~\eqref{eq:unique-z-identity-simple} therefore proves
\[
\left|
\mathbb E[Z_{i,T}]
\right|
=
O\left(\frac{f_T}{T}\right).
\]
\end{proof}

\section{Proof of Lemma~\ref{lem:gaussian-cdf-comparison} and Lemma~\ref{lem:two-tracking-estimates}}\label{app:CDF-expansion}

We first recall the standard Gaussian exponential martingale, the proof of which we omit.

\begin{lemma}[Gaussian exponential martingale]
\label{lem:exponential-martingale}
Let \((Y_n)_{n\geq1}\) be adapted to
\((\mathcal G_n)_{n\geq0}\), with
\(Y_n\sim N(0,\sigma^2)\) independent of
\(\mathcal G_{n-1}\) for every \(n\geq1\).
Set \(S_0=0\) and \(S_n=\sum_{\ell=1}^nY_\ell\).
For every fixed \(\lambda\in\mathbb R\),
\[
M_n(\lambda)
=
\exp\left\{
\lambda S_n-\frac{\lambda^2\sigma^2}{2}n
\right\},
\qquad n\geq0,
\]
is an integrable \((\mathcal G_n)\)-martingale. \(\mathbb E M_\tau(\lambda)=1\)
for every bounded \((\mathcal G_n)\)-stopping time \(\tau\).
\end{lemma}

\medskip
\medskip
\medskip

\begin{proof}[Proof of Lemma~\ref{lem:gaussian-cdf-comparison}]
Fix $i \in \mathcal{O}.$ We first introduce a clipped version of $Z_{i, T}$ to facilitate our analysis. Define 
\[
\varepsilon =\frac{1}{ 16|\mathcal O|\sigma_i^2 \sum_{a=1}^K\sigma_a^{-2} }, \qquad N = \bigl(N_i(T)\vee\lfloor(1-\varepsilon)n_{i, T}\rfloor\bigr) \wedge\lceil(1+\varepsilon)n_{i,T}\rceil, \]
\[
D=\frac{N}{n_{i, T}}-1,\qquad Z'=\frac{S_i(N)}{\sigma_i\sqrt{n_{i,T}}},
\qquad
Z=\frac{Z'}{\sqrt{1+D}}.
\]
Then $1 \le N \le T$ and $|D| \le 2 \varepsilon$. Furthermore, following the hypothesis in the lemma,
\[
\|D\|_2 \le \left\|\frac{N_i(T) - n_{i, T}}{n_{i, T}}\right\|_2 = O(f_T^{-1}).
\]
By Markov's inequality, this implies that 
\[
\PP\left(N \ne N_{i}(T)\right) \le \frac{\EE D^2}{\varepsilon^2} = O(f_T^{-2}).
\]
Since \(Z=Z_{i,T}\) whenever \(N=N_i(T)\), the above implies
\begin{equation}
    \label{eq:CDF-clip-control}
    \sup_{z\in\mathbb R}
\left|
\mathbb P(Z+\delta_i\leq z)
-
\mathbb P(Z_{i,T}+\delta_i\leq z)
\right|
=
O(f_T^{-2}).
\end{equation}
Hence it suffices to analyze the CDF of $Z$ instead of $Z_{i, T}.$ To that end, we exploit the Gaussian structure, the property of the \texttt{UCB}-f algorithm, as well as the Gaussian exponential martingale introduced in Lemma~\ref{lem:exponential-martingale}.

Complete every potential-reward sequence on each arm $a$ to $T$ observations. Define 
\[
\xi_T
=
\frac{
\sum_{a=1}^K\sigma_a^{-2}\sum_{s=1}^T Y_a(s)
}{
T\sum_{a=1}^K\sigma_a^{-2}
},
\qquad
\mathcal R_T
=
\sigma\{Y_a(s)-\xi_T:a\in[K],\ 1\leq s\leq T\}.
\]
Since we assume Gaussian rewards, $(\xi_T,\, \{Y_a(s) - \xi_T\}_{a \in [K], s \in [T]})$ is jointly Gaussian. Furthermore,
\[
\EE\xi_T = 0, \qquad \operatorname{Var}(\xi_T)
=
\frac{1}{T\sum_a\sigma_a^{-2}},
\qquad
\operatorname{Cov}\bigl(\xi_T,Y_a(s)-\xi_T\bigr)=0.
\]
Joint Gaussianity therefore implies that \(\xi_T\) is independent of \(\mathcal R_T\). Notice that under the generalized \texttt{UCB-f} algorithm, the algorithm's allocation pattern is invariant if all rewards on all arms are shifted by a common constant. Consequently, conditional on $\mathcal{R}_T$, the allocation sequence and therefore $N_a(T)$ (hence $N$) are all determined, and that is regardless of the realization of $\xi_T$. Let $b = \EE[Z'|\mathcal{R}_T]$. Since $Y_i(t) - \xi_T$ and $N$ are $\mathcal{R}_T$-measurable \jledit{for every $1\leq t\leq T$}, and since $\xi_T$ is independent of $\mathcal{R}_T$ with zero mean, \(b = \frac{1}{\sigma_i \sqrt{n_{i, T}}}\sum_{t = 1}^N (Y_i(t) - \xi_T)\). And
\[
Z' = \frac{\sum_{t = 1}^N (Y_i(t) - \xi_T)}{\sigma_i \sqrt{n_{i, T}}} + \frac{N}{\sigma_i \sqrt{n_{i, T}}}\xi_T =
b + (1 + D)\frac{\sqrt{n_{i, T}}}{\sigma_i} \xi_T
\]
Therefore, by relation between $Z$ and $Z'$
\[
Z + \delta_i = \delta_i + \frac{b}{\sqrt{1 + D}} + \sqrt{1 + D}\frac{\sqrt{n_{i, T}}}{\sigma_i}\xi_T.
\]
Conditional on $\mathcal{R}_T$, $b$ and $N$ are both determined, while $\xi_T$ remains random, independent, and normal. Let 
\[
v
=
\operatorname{Var}\!\left(
    \frac{\sqrt{n_{i,T}}}{\sigma_i}\xi_T
\right)
=
\frac{n_{i,T}}{\sigma_i^2 T\sum_{a=1}^K\sigma_a^{-2}}.
\]
Then we have
\[
Z + \delta_i\, |\,\mathcal{R}_T \sim N\left(\delta_i + \frac{b}{\sqrt{1 + D}}, v(1 +D)\right),
\]
which implies 
\[
\mathbb P( Z+\delta_i\leq z)
=
\mathbb E\Phi\left(
\frac{z-\delta_i-b/\sqrt{1+ D}}
{\sqrt{v(1+ D)}}
\right),
\]
To obtain a comparison between the  CDF of $Z$ and that of a standard normal, we use the exponential martingale. Particularly, We introduce an arm-specific filtration \(\mathcal{G}_{i}\), defined by
\[
\mathcal{G}_{i, k} = \sigma\left(X_{i, 1}, \dots, X_{i, k}, \,\, \{X_{a, s}: a \neq i, \ s \ge 1 \}\right).
\]
In words, reward sequences from all arms but $i$ is assumed known. $X_{i, k}$ is independent of $\mathcal{G}_{i, k-1}$, and in particular, $N_i(T)$, and therefore $N$ are \(\mathcal G_{i}\)-stopping times both bounded by $T$. We may then apply Lemma~\ref{lem:exponential-martingale} with parameter \(s/(\sigma_i\sqrt{n_{i,T}})\), and stopping time $\tau  = N$. 
\[
 1 = \EE\exp\left\{\frac{s}{\sigma_i\sqrt{n_{i,T}}} S_i(N) - \frac{s^2 \sigma_i^2}{2\sigma_i^2  n_{i, T}} N\right\} =  \EE\exp\left\{s Z' - (1 +D)\frac{s^2 }{2 } \right\}.
\]
By conditioning on $\mathcal{R}_T$ and a tower rule, we obtain
\begin{align*}
    &\EE\exp\left\{s Z' - (1 +D)\frac{s^2 }{2 } \right\}\\
    &\qquad= \EE \left[\EE\left[\exp\left\{s b + s(1  +D ) \frac{\sqrt{n_{i, T}}}{\sigma_i} \xi_T  - (1 +D)\frac{s^2 }{2 } \right\}\bigg| \mathcal R_T\right]\right],\\
    &\qquad= \EE \left[\exp\left\{s b  - (1 +D)\frac{s^2 }{2 } \right\} \EE\left[\exp\left\{ s(1  +D ) \frac{\sqrt{n_{i, T}}}{\sigma_i} \xi_T \right\}\bigg| \mathcal R_T\right]\right],\\
    &\qquad= \EE \exp\left\{s b  - (1 +D)\frac{s^2 }{2 }  + \frac{s^2}{2} v(1 + D)^2\right\},
\end{align*}  
where we have used the derived expression of $Z'$, the independence between $N$ (and therefore $D$) and $\xi_T$ conditional on $\mathcal{R}_T$, and the Gaussianity of $\xi_T$. In particular, in the last equality we used the expression of the Gaussian moment generating function (MGF) and that \(v = \operatorname{Var}\left(\frac{\sqrt{n_{i, T}}}{\sigma_i}\xi_T\right).\) Combining the above, and rearranging terms to get 
\[
\begin{aligned}
e^{s^2/2} = \mathbb E\exp\left\{
sb+\frac{s^2}{2}
\bigl[v(1+ D)^2- D\bigr]
\right\},
\qquad s\in\mathbb R,
\end{aligned}
\]
which equates the MGF of a standard normal to that of a Gaussian mixture. Uniqueness of MGFs therefore leads to
\[
\begin{aligned}
\Phi(z)
=
\mathbb E\Phi\left(
\frac{z-b}
{\sqrt{v(1+ D)^2- D}}
\right).
\end{aligned}
\]
Putting the above together, we derive
\begin{align}
\PP(Z + \delta_i \le z) - \Phi(z) = \mathbb E\Phi\left(
\frac{z-\delta_i-b/\sqrt{1+ D}}
{\sqrt{v(1+ D)}}
\right) - \mathbb E\Phi\left(
\frac{z-b}
{\sqrt{v(1+ D)^2- D}}
\right).
    \label{eq:CDF-conditional-comparison}
\end{align}
It now suffices to bound the RHS in \eqref{eq:CDF-conditional-comparison}. We expand the two conditional Gaussian CDFs to first order in their means and variances around \((b,v)\). Their zeroth-order terms cancel. We then bound the quadratic
remainder and evaluate the remaining linear terms. Define
\[
F_z(\mu,V)
=
\Phi\left(\frac{z-\mu}{\sqrt V}\right),
\qquad
k_z(x)
=
\frac1{\sqrt v}\phi\left(\frac{z-x}{\sqrt v}\right),
\]
where \(k_z'\) denote derivatives in \(x\). Representing the two conditional CDFs in \eqref{eq:CDF-conditional-comparison} in terms of $F_z$, and expanding around $(b, v)$. Note that at $(b, v)$, the two derivatives are \(\partial_\mu F_z(b,v)=-k_z(b),\) and \(\partial_V F_z(b,v)=-\frac12k_z'(b).\) By Taylor expansion 
\begin{align}
\mathbb P(Z+\delta_i\leq z)-\Phi(z)
={}&
-\delta_i\mathbb E k_z(b)
+
\frac12\mathbb E\left[
D\{bk_z(b)-(1-v)k_z'(b)\}
\right]
\nonumber\\
&+
O\left(
\delta_i^2+\mathbb E[D^2(1+b^2)]
\right), 
\label{eq:talor-raw}
\end{align}
uniformly in \(z\).  We justify the remainder bound pointwise in \(b,D\).
Our choice gives \(\varepsilon\leq1/16\) and
\(8\varepsilon\leq v\leq1\) for sufficiently large \(T\).
Since \(|D|\leq2\varepsilon\), we have \(|D|\leq1/8\), so
\[
\left|(1+D)^{-1/2}-1\right|\leq C|D|,
\qquad
\left|(1+D)^{-1/2}-1+\frac D2\right|\leq CD^2.
\]
Moreover, \(v(1+D)\geq v/2\geq4\varepsilon,\) and, since \(|2v-1|\leq1\),
\[
v(1+D)^2-D
=
v+(2v-1)D+vD^2
\geq v-|D|
\geq6\varepsilon.
\]
Thus the variances along both Taylor segments are at least
\(4\varepsilon\). The Hessian of \(F_z(\mu,V)\) is therefore
uniformly bounded in \(z,\mu\) along these segments, giving
\[
\begin{aligned}
&\left|
F_z(\mu,V)-F_z(b,v)
+k_z(b)(\mu-b)
+\frac12k_z'(b)(V-v)
\right|
\\
&\qquad\leq
C\left[(\mu-b)^2+(V-v)^2\right].
\end{aligned}
\]
The mixed term is included in this bound by
\(2|(\mu-b)(V-v)|\leq(\mu-b)^2+(V-v)^2\).

For the first conditional CDF, the squared parameter changes
\ are bounded by
\[
\left[
\delta_i+b\bigl((1+D)^{-1/2}-1\bigr)
\right]^2
+
(vD)^2
\leq
C\left[\delta_i^2+(1+b^2)D^2\right].
\]
For the second conditional CDF, the mean does not change,
and the squared variance change satisfies
\(\left[(2v-1)D+vD^2\right]^2
\leq CD^2.\) 

Finally, replacing the first mean change by
\(\delta_i-bD/2\) incurs an additional error
\(O(|b|D^2)\), while replacing the second variance change
by \((2v-1)D\) incurs an error \(O(D^2)\).
These errors multiply uniformly bounded first derivatives.
Since \(|b|\leq1+b^2\), the total pointwise remainder is
bounded by \(C\left[\delta_i^2+(1+b^2)D^2\right].\) Taking expectations yields the stated remainder \(O\left(
\delta_i^2+\mathbb E[D^2(1+b^2)]
\right),\) uniformly in \(z\).

We show that this higher-order remainder term is uniformly $o(f_T^{-1})$. In fact, the MGF identity established above and positivity of its variance term imply
\[
\mathbb E e^{sb}
\leq
\mathbb E\exp\left\{
sb+\frac{s^2}{2}\bigl[v(1+D)^2-D\bigr]
\right\}
=
e^{s^2/2},
\qquad s\in\mathbb R.
\]
Consequently, \(\mathbb E|b|^6=O(1)\).
Using Hölder's inequality, boundedness of \(D\), and
\(\mathbb ED^2=O(f_T^{-2})\), 
\[
\begin{aligned}
\mathbb E[D^2(1+b^2)]
\leq
(\mathbb E|D|^3)^{2/3}
\{\mathbb E(1+b^2)^3\}^{1/3} \leq
C(\mathbb ED^2)^{2/3}
=
O(f_T^{-4/3}).
\end{aligned}
\]
Together with \(\delta_i^2=O(f_T^{-2})\), this shows that
the Taylor remainder is \(o(f_T^{-1})\), uniformly in \(z\).

It remains to evaluate the linear terms in \eqref{eq:talor-raw}.
We compare them with their Gaussian counterparts where $N$ is replaced by $n_{i, T}$, for which the cancellation is exact \jledit{up to small rounding error}. Define
\[
G=\frac{S_i(n_{i, T})}{\sigma_i\sqrt{n_{i, T}}},
\qquad
g=\mathbb E[G\mid\mathcal R_T],
\qquad
D^{(0)}
=
\frac{2\sqrt{n_{i, T}}}{f_T}
\left[
\bar Y_i({n_{i, T}})
-
\frac1{|\mathcal O|}
\sum_{a\in\mathcal O}\bar Y_a(n_{a,T})
\right].
\]
\jledit{Then by assumption,
\(
g
=
\frac{1}{\sigma_i\sqrt{n_{i,T}}}
\sum_{t=1}^{\lfloor n_{i,T}\rfloor}\bigl(Y_i(t)-\xi_T\bigr),
\)
and $(D^{(0)},g)$ is centered jointly Gaussian, with
\[
\operatorname{Var}(g)
=
\frac{\lfloor n_{i,T}\rfloor}{n_{i,T}}
-
\left(\frac{\lfloor n_{i,T}\rfloor}{n_{i,T}}\right)^2v,
\qquad
\operatorname{Cov}(D^{(0)},g)=2\delta_i,
\qquad
\|D^{(0)}\|_4=O(f_T^{-1}).
\]
Since $n_{i,T}\asymp T$ and $0\le n_{i,T}-\lfloor n_{i,T}\rfloor<1$,
\[
\operatorname{Var}(g)-(1-v)=O(T^{-1}).
\]}
We replace $(D, b)$ by $(D^{(0)}, g)$ in the linear term of \eqref{eq:talor-raw}. 
\jledit{Gaussian integration by parts yields
\[
\mathbb E[D^{(0)}gk_z(g)]
=
2\delta_i\mathbb E[k_z(g)]
+
\operatorname{Var}(g)\mathbb E[D^{(0)}k_z'(g)].
\]
Consequently,
\begin{align*}
&-\delta_i\mathbb E[k_z(g)]
+\frac12\mathbb E\!\left[
D^{(0)}
\bigl\{gk_z(g)-(1-v)k_z'(g)\bigr\}
\right]
=
\frac{\operatorname{Var}(g)-(1-v)}{2}
\mathbb E[D^{(0)}k_z'(g)]
=
O\!\left(\frac{1}{Tf_T}\right),
\end{align*}
uniformly in $z$. Here we used
$\|D^{(0)}\|_1=O(f_T^{-1})$ and
$\sup_{z,x}|k_z'(x)|<\infty$, since $v$ is bounded away from zero. Plugging back to \eqref{eq:talor-raw}, the linear terms with \((D,b)\) replaced by \((D^{(0)},g)\) sum to $O(1/(Tf_T))$.}

We now bound the error in this replacement.
To compare \(b\) and \(g\), write
\[
S_i(N)-S_i(n_{i, T})
=
\sum_{k=1}^T
\left(
\mathbf1_{\{N\geq k\}}-\mathbf1_{\{k\leq n_{i, T}\}}
\right)Y_i(k).
\]
Since \(N\) is a stopping time, the coefficient of \(Y_i(k)\)
is \(\mathcal G_{i,k-1}\)-measurable, whereas \(Y_i(k)\) is
independent of \(\mathcal G_{i,k-1}\), centered, and has variance
\(\sigma_i^2\). Thus all cross terms vanish when we square and
take expectations, giving
\jledit{\begin{align*}
\mathbb E|Z'-G|^2
&=
\frac{1}{n_{i,T}}
\mathbb E\sum_{k=1}^{T}
\left(
\mathbf 1_{\{N\ge k\}}
-
\mathbf 1_{\{k\le n_{i,T}\}}
\right)^2
\\
&=
\frac{\mathbb E|N-\lfloor n_{i,T}\rfloor|}{n_{i,T}}
\\
&\le
\frac{\mathbb E|N-n_{i,T}|}{n_{i,T}}
+
\frac{n_{i,T}-\lfloor n_{i,T}\rfloor}{n_{i,T}}
\\
&\le
\mathbb E|D|+\frac{1}{n_{i,T}}
=
O(f_T^{-1}).
\end{align*}
Here the sum counts exactly the indices between $N$ and $\lfloor n_{i,T}\rfloor$.}
By the definitions of \(b\) and \(g\), conditional Jensen's
inequality now gives
\[
\|b-g\|_2
=
\left\|\mathbb E[Z'-G\mid\mathcal R_T]\right\|_2
\leq
\|Z'-G\|_2
=
O(f_T^{-1/2}).
\]

Next, since \(v\) is bounded away from zero and
\(\phi,\phi',\phi''\) are bounded,
\[
\sup_{z,x\in\mathbb R}|k_z^{(j)}(x)|
=
v^{-(j+1)/2}\sup_{y\in\mathbb R}|\phi^{(j)}(y)|
\leq C,
\qquad j=0,1,2.
\]
In particular, the mean value theorem yields
\[
\begin{aligned}
|bk_z(b)-gk_z(g)|
&=
|(b-g)k_z(b)+g\{k_z(b)-k_z(g)\}|\\
&\leq
C|b-g|+C|g|\,|b-g|
=
C(1+|g|)|b-g|,
\end{aligned}
\]
uniformly in \(z\). Using these bounds, the bounded moments of \(b\) and \(g\), and
\(\|D^{(0)}(1+|g|)\|_2=O(f_T^{-1})\),
Cauchy--Schwarz and the Gaussian cancellation above give
\jledit{\begin{align*}
\left|
-\delta_i\mathbb E[k_z(b)]
+\frac12\mathbb E\!\left[
D\bigl\{bk_z(b)-(1-v)k_z'(b)\bigr\}
\right]
\right|
&
\le
C\|D-D^{(0)}\|_2
+
Cf_T^{-1}\|b-g\|_2
+
\frac{C}{Tf_T}
\\
&
=
C\|D-D^{(0)}\|_2
+
O(f_T^{-3/2}),
\end{align*}}
uniformly in \(z\).
Together with the Taylor remainder bound, this proves
\[
\sup_{z\in\mathbb R}
\left|
\mathbb P(Z+\delta_i\leq z)-\Phi(z)
\right|
\leq
C\|D-D^{(0)}\|_2+o(f_T^{-1}).
\]

Returning to the original statistic adds only the
\(O(f_T^{-2})\) CDF error established above.
For the tracking error, clipping can increase the distance from
\(D^{(0)}\) only when \(D^{(0)}\) lies outside the clipping interval.
Since \(D^{(0)}\) is Gaussian with variance \(O(f_T^{-2})\),
the resulting \(L^2\) error is exponentially small. Thus
\jledit{\[
\|D-D^{(0)}\|_2
\le
\left\|
\frac{N_i(T)}{n_{i,T}}-1-D^{(0)}
\right\|_2
+
Ce^{-cf_T^2}.
\]}
We therefore obtain the tracking-error norm in the lemma
statement, completing the proof.

\end{proof}

\medskip
\medskip
\medskip

\begin{proof}[Proof of Lemma~\ref{lem:two-tracking-estimates}]
Fix \(1<\gamma<\beta\). Recall the good event $\mathcal{G}_T$ in the proof of Lemma~\ref{lem: ultra tail}. Corollary~\ref{lem:terminal-reward-event} gives \(\mathbb P(\mathcal G_T^c)\leq f_T^{-2\gamma}.\)  By Lemma~\ref{lem:ucb-fluid-approximation},
\(n_{i,T}\asymp T\). Since both \(N_i(T)\) and
\(\tilde n_{i,T}\) belong to \([0,T]\),
Lemma~\ref{lem:terminal-tracking} therefore gives
\[
\begin{aligned}
\left\|
\frac{N_i(T)-\tilde n_{i,T}}{n_{i,T}}
\right\|_2
&\leq
C\frac{d_{0,T}}{f_T^2}
+
C\mathbb P(\mathcal G_T^c)^{1/2}
\\
&=
O\left(
\frac{d_{0,T}}{f_T^2}+f_T^{-\gamma}
\right)
=
o(f_T^{-1}).
\end{aligned}
\]

We next expand the empirical-fluid counts $\tilde n$. The fluid equations give
\[
\frac{\tilde n_{a,T}}{n_{a,T}}
=
\left(
1+\frac{\sqrt{n_{a,T}}}{f_T}
\bigl[
\tilde\lambda_T-\lambda_T-\bar Y_a(n_{a,T})
\bigr]
\right)^{-2}.
\]
On \(\mathcal G_T\), Lemma~\ref{lem:empirical_fluid_geometry}
implies that the term added to one is \(O(\ell_T/f_T)\),
uniformly over arms. Taylor expansion consequently gives
\[
\begin{aligned}
\tilde n_{a,T}-n_{a,T}
={}&
\frac{2n_{a,T}^{3/2}}{f_T}
\bigl[
\bar Y_a(n_{a,T})-(\tilde\lambda_T-\lambda_T)
\bigr]
+
O\left(
\frac{n_{a,T}\ell_T^2}{f_T^2}
\right).
\end{aligned}
\]
Summing over arms, using
\(\sum_a\tilde n_{a,T}=\sum_a n_{a,T}=T\) and
\(\sum_a n_{a,T}^{3/2}\asymp T^{3/2}\), yields
\[
\tilde\lambda_T-\lambda_T
=
\frac{
\sum_{a=1}^K n_{a,T}^{3/2}\bar Y_a(n_{a,T})
}{
\sum_{a=1}^K n_{a,T}^{3/2}
}
+
O\left(
\frac{\ell_T^2}{f_T\sqrt T}
\right).
\]
Substituting back gives the weighted first-order expansion
for an optimal arm, with remainder \(O(\ell_T^2/f_T^2)\)
on \(\mathcal G_T\).

On \(\mathcal G_T^c\), boundedness of
\(\tilde n_{i,T}/n_{i,T}-1\) gives
\[
\left\|
\left(\frac{\tilde n_{i,T}}{n_{i,T}}-1\right)
\mathbf 1_{\mathcal G_T^c}
\right\|_2
\leq
C\mathbb P(\mathcal G_T^c)^{1/2}
=
O(f_T^{-\gamma}).
\]
For the Gaussian linear term, Hölder's inequality bounds its
\(L^2\)-norm on \(\mathcal G_T^c\) by its \(L^4\)-norm times
\(\mathbb P(\mathcal G_T^c)^{1/4}\). Its contribution is therefore
\[
O(f_T^{-1})\,
\mathbb P(\mathcal G_T^c)^{1/4}
=
O(f_T^{-1-\gamma/2}).
\]
Splitting over \(\mathcal G_T\) and \(\mathcal G_T^c\), and applying
the triangle inequality, now gives
\[
\begin{aligned}
&\Bigg\|
\frac{\tilde n_{i,T}}{n_{i,T}}-1
-
\frac{2\sqrt{n_{i,T}}}{f_T}
\left[
\bar Y_i(n_{i,T})
-
\frac{
\sum_{a=1}^K n_{a,T}^{3/2}\bar Y_a(n_{a,T})
}{
\sum_{a=1}^K n_{a,T}^{3/2}
}
\right]
\Bigg\|_2
\\
&\qquad=
O\left(
\frac{\ell_T^2}{f_T^2}
+
f_T^{-\gamma}
+
f_T^{-1-\gamma/2}
\right)
=
o(f_T^{-1}),
\end{aligned}
\]
where the last equality uses \(\gamma>1\) and
\(\ell_T^2=o(f_T)\).

Finally, $n_{i, T}$ are equal over
\(i \in \mathcal O\), and
Lemma~\ref{lem:ucb-fluid-approximation} gives
\(\sum_{j\in\mathcal S}n_{j,T}=O(f_T^2)\). Thus
\[
\|\bar Y_a(n_{a,T})\|_2
=
\frac{\sigma_a}{\sqrt{\lfloor n_{a,T}\rfloor}}
=
O(n_{a,T}^{-1/2}).
\]
These facts imply
\[
\begin{aligned}
&\frac{2\sqrt{n_{i,T}}}{f_T}
\left\|
\frac{
\sum_{a=1}^K n_{a,T}^{3/2}\bar Y_a(n_{a,T})
}{
\sum_{a=1}^K n_{a,T}^{3/2}
}
-
\frac{1}{|\mathcal O|}
\sum_{a\in\mathcal O}\bar Y_a(n_{a,T})
\right\|_2
\\
&\qquad=
O\left(\frac{f_T}{T}\right)
=
o(f_T^{-1}).
\end{aligned}
\]
Combining the tracking estimate, the empirical-fluid expansion,
and the weight replacement proves the claim.
\end{proof}

\if false

\section{Proof of the Conditions}
\label{sec:condition verification}
Throughout this section, the index function is assumed to take the following form:
\begin{equation*}
    I_t(x,n) = x + \frac{f_t}{\sqrt n},
    \qquad
    f_T \to \infty
\end{equation*}
Two choices of $f_t$ are considered:
\begin{itemize}
    \item \texttt{UCB1}\footnote{We abuse the name \texttt{UCB1} here to allow for general value of $\rho$.}, with $f_t=\sqrt{\rho\log t}$ and $\rho/\sigma^2 > 2$;
    \item \texttt{Poly-UCB}, with $f_t=t^\alpha$ and $0<\alpha<1/2$.
\end{itemize}
The next step is to verify that both \texttt{UCB1} and \texttt{Poly-UCB} satisfy Conditions~\ref{ass:index-regularity}--\ref{ass:concentration}. We first verify Conditions~\ref{ass:index-regularity} and
\ref{ass:index_regularity2}. Denote $m=|\mathcal O|$.
\begin{lemma}[Verification of Conditions~\ref{ass:index-regularity}--\ref{ass:index_regularity2}]\label{lem:verification of conditons1-2 for ucb1}
For \texttt{UCB1} and \texttt{Poly-UCB}, 
\begin{align*}
    n_{i,T} \asymp \frac{T}{m} \quad \text{for} \quad  i\in\mathcal{O}  \qquad  n_{j,T} \asymp f_T^2 \quad \text{for} \quad j \in \mathcal S
\end{align*}
Moreover, Conditions~\ref{ass:index-regularity} and~\ref{ass:index_regularity2} are satisfied for both algorithms.
\end{lemma}
\begin{proof}
For Condition~\ref{ass:index-regularity},
the derivative identities are immediate. Thus, on any shrinking relative count box, for any arm $a \in [K]$,
\begin{equation*}
    \frac{\partial_x I_T(x,n)}{\partial_x I_T(\mu_a, n_{a,T})} = 1  \qquad \frac{-\partial_n I_T(x,n)}{-\partial_n I_T(\mu_a, n_{a,T})} = \left(\frac{n_{a,T}}{n}\right)^{3/2}
    \to 1.
\end{equation*}
To verify Condition~\ref{ass:index_regularity2}, we first recall the approximate dynamical system:
\begin{equation}
    \mu_a + \frac{f_T}{\sqrt{n_{a,T}}} = \lambda_T,
    \qquad
    \sum_a n_{a,T}=T,
\end{equation}
Then for $i\in\mathcal O$ and $j \in \mathcal S$,
\begin{equation*}
    n_{i,T}=\frac{f_T^2}{(\lambda_T - \mu^{\star})^2}, \qquad  n_{j,T}=\frac{f_T^2}{(\Delta_j+\lambda_T - \mu^{\star})^2}. 
\end{equation*}
Hence
\begin{equation*}
    \frac{mf_T^2}{(\lambda_T - \mu^{\star})^2}
    + \sum_{j\in \mathcal S}
    \frac{f_T^2}{(\Delta_j+\lambda_T - \mu^{\star})^2}
    =T.
\end{equation*}
Above two equations implies that for optimal arms and suboptimal arms,
\begin{equation*}
    n_{i,T}=\frac{T}{m}+O(f_T^2), \qquad n_{j,T} =  \frac{f_T^2}{\Delta_j^2} +O\!\left(\frac{f_T^3}{\sqrt T}\right).
\end{equation*} 
Combining with derivative identities, we have 
\begin{equation*}
    n_{j, T} = \omega(1), \qquad n_{j,T}/n_{i,T}=o(1), \qquad \frac{\partial_n I_T(\mu_i, n_{i,T})}{\partial_n I_T(\mu_j, n_{j,T})} = o(1),
\end{equation*}
which implies condition~\ref{ass:index_regularity2} holds for both \texttt{UCB1} and \texttt{Poly-UCB}. 
\end{proof}
\subsection{Logarithmic UCB1}
In this section, we introduce the following notation and events to simplify the presentation. Define
\begin{align*}
    & \mathcal{R}_j =\{n_{j,T}/\log n_{j,T} \le N_j(T) \le n_{j,T}\log n_{j,T}\} \quad \text{for} \quad j \notin \mathcal O \\
    & A_h = \{ \max_{s\notin\mathcal O}|\tilde n_{s,T} - n_{s,T}| \le h\} \quad H = \{\sum_{s\notin \mathcal{O}}N_s(T) \le T/4\} \quad U_g = \{ \max_{r\in\mathcal O}|\tilde n_{r,T} - n_{r,T}| \le g\}
\end{align*}
In the proof below, these events serve to control $N_a(T)$ and $\tilde n_{a,T}$ for any $a \in [k]$, ensuring that they remain within a typical scale; they therefore constitute certain “good” events. We also define
\begin{equation*}
    m = |\mathcal{O}|,\ \bar Y_j(s) = \bar X_j(s) - \mu_j, \ \bar{\mathbf Y}_T = \bigl(\bar Y_1(s),\ \bar Y_2(s),\ldots,\bar Y_K(s)\bigr),\ \sigma=\max_{a\in\{1,\ldots,K\}}\sigma_a.
\end{equation*}
Moreover, throughout the proof, constants that play a specific role will be denoted separately, while \(c,C>0\) denote generic constants that may depend on fixed problem parameters, are independent of \(T\), and may change from line to line.

\begin{lemma}[Fluid approximation tail]\label{lem:Tail bound of fluid approximation}
For any $y >0$ and any fixed arm $a$, there exists constants $c,C>0$ such that 
\begin{equation*}
    \Pp\left(\frac{|\tilde n_{a,T} - n_{a,T}|}{n_{a,T}/f_T} \geq y\right) \leq  C\exp\{-cy^2\} + C\exp\{-cf_T^2\} 
\end{equation*}
\end{lemma}

\begin{lemma}[Suboptimal arm upper tail]\label{lem:ucb-suboptimal-upper-tail}
For any fixed $j\notin\mathcal O$ and any $1<\beta<\rho/(2\sigma^2)$, there exist constants $C,c,L>0$ such that, for $L n_{j,T}\le y\le T$, 
\begin{equation*}
\Pp\{N_j(T)\ge y\}\le Ce^{-cy}+Cy^{-\beta}.
\end{equation*}
and
\begin{equation*}
\Pp\{\sum_{j\notin \mathcal{O}}N_j(T)\geq y\}\le Ce^{-cy}+Cy^{-\beta}.
\end{equation*}
\end{lemma}

\begin{lemma}[Suboptimal arm lower tail]\label{lem:ucb-suboptimal-lower-tail}
  For any fixed $j\notin\mathcal O$, there exist constants $C,c ,L_1>0$ such that
  \begin{equation*}
  \Pp\{N_j(T)< L_1n_{j,T}\}\le Cn_{j,T}\exp\{-cf_T^2\}+CT^{-\beta},
\end{equation*}
where $ 1< \beta <\rho/(2\sigma^2)$. 
\end{lemma}
\begin{lemma}[Suboptimal arm tracking]\label{lem:ucb-suboptimal-witness}
For any fixed $j\notin\mathcal O$, taking $h = \sqrt{n_{j,T}} \log n_{j,T}$ for event $A_h$, there exists constant $C$ such that for sufficiently large $T$,
\begin{align*}
    \Pp(\mathcal{R}_j,A_h,H,U,|N_j(T)-\tilde n_{j,T}| > n_{j,T}^{1/4}\log n_{j,T}) \leq C(\log n_{j,T})n_{j,T}^{-\beta}
\end{align*}
\end{lemma}

\begin{lemma}[Verification of Condition~\ref{ass:concentration} for suboptimal arms]\label{lem:verification of conditon3 for ucb1 suboptimal arm}
For any fixed suboptimal arm $j\notin\mathcal O$ , $s>1$ and $1<p<\beta$,
    \begin{align*}
    &\Gamma_{j,T}\norm{\frac{\tilde n_{j,T}-n_{j,T}}{n_{j,T}}}_s=O(1),
    &\Gamma_{j,T}\norm{\frac{N_j(T)-\tilde n_{j,T}}{n_{j,T}}}_p=o(1),
    \end{align*}
where $1 < \beta < \rho/(2\sigma^2)$. Moreover, for any $r>1$,
    \begin{align*}
    \sqrt{n_{j,T}}\norm{\bar X_j(N_j(T))-\bar X_j(n_{j,T})}_r=o(1).
    \end{align*}
\end{lemma}

\begin{lemma}[Optimal low-count and imbalance events]\label{lem:ucb-opt-low-imbalance}
Assume $m\ge2$, there exist constant $C>0$ and small enough $\kappa > 0$ such that
\begin{equation*}
    B_\kappa =\Bigl\{\min_{i\in O}N_i(T) \ge \kappa T,\quad \sum_{j \notin \mathcal{O}}N_j(T) \le T/f_T^2\Bigr\}
\end{equation*}
satisfies
\begin{equation*}
    \Pp(B_\kappa^c)\le C(\log T)^2 T^{-\beta},
\end{equation*}
where $ 1< \beta <\rho/(2\sigma^2)$. Furthermore, there exist constants $C, c>0, 0 < \eta < 1/4$, such that the imbalanced optimal-selection event
\begin{equation*}
W_\eta=\{\exists t\le T,\ \exists r,s\in O:\ A_t=r,\ N_r(t-1)\ge4\eta T,\ N_s(t-1)\le\eta T\}
\end{equation*}
satisfies
\begin{equation*}
\Pp(W_\eta)\le C \log T\,T^{-c}.
\end{equation*}
\end{lemma}

\begin{lemma}[Optimal arm tracking]\label{lem:ucb-optimal-witness}
Assume $m \ge 2$. For any $i\in\mathcal O$, fix $\kappa >0$ and $\eta>0$ as in Lemma~\ref{lem:ucb-suboptimal-lower-tail} and ~\ref{lem:ucb-opt-low-imbalance}. On $B_\kappa \cap W_\eta^c \cap U_g$, if 
\begin{equation*}
    |N_i(T)-\tilde n_{i,T}|>d \qquad \frac{T}{f_T} < g < \frac{T}{f_T^{1/4}} \qquad d \ge \frac{\sqrt{n_{i,T}g}}{f_T} + \frac{T}{2f_T^2}
\end{equation*}
there exists constant $c,C>0$ such that for sufficiently large $T$, 
\begin{equation*}
        \Pp\{U_g,\ B_\kappa,W_\eta^c,\ |N_i(T)-\tilde n_{i,T}|>d\}
        \le
        C\exp\left\{-c\frac{f_T^2}{T}d\right\}
\end{equation*}
\end{lemma}

\begin{lemma}[Verification of Condition~\ref{ass:concentration} for optimal arms]\label{lem:verification of conditon3 for ucb1 optimal arm}
Assume $m \ge 2$.  For each $i\in \mathcal{O}$ and any constants $s, p>1$,
    \begin{align*}
     &\Gamma_{i,T}\norm{\frac{\tilde n_{i,T}-n_{i,T}}{n_{i,T}}}_s=O(1),
    &\Gamma_{i,T}\norm{\frac{N_i(T)-\tilde n_{i,T}}{n_{i,T}}}_p=o(1).
    \end{align*}
Moreover, for any $r>1$,
    \begin{align*}
    \sqrt{n_{i,T}}\norm{\bar X_i(N_i(T))-\bar X_i(n_{i,T})}_r=o(1).
    \end{align*}
\end{lemma}

\begin{proof}[Proof of Lemma~\ref{lem:Tail bound of fluid approximation}]
Recall that the fluid approximations are solutions of 
\begin{equation} \label{eq:definition of n}
    \mu_a+f_T/\sqrt{n_{a,T}} = \lambda_T , \qquad \sum_a n_{a,T}=T,
\end{equation}
while the empirical fluid approximations are solutions of 
\begin{equation}\label{eq:definition of tn}
    \mu_a+\bar Y_{a,T}+f_T/\sqrt{\tilde n_{a,T}}
    = \tilde\lambda_T, \qquad \sum_a \tilde n_{a,T}=T.  
\end{equation}
This implies that $\lambda_T$ and $\tilde \lambda_T$ are the implicit functions that satisfy the following equations.
\begin{equation*}
    \sum_{a=1}^K \frac{f_T^2}{(\lambda_T-\mu_a)^2} = T  \qquad \sum_{a=1}^K \frac{f_T^2}{(\tilde\lambda_T-\mu_a-\bar Y_{a,T})^2} = T 
\end{equation*}
This implies that there exists a function $\Phi(x)$ such that $\tilde\lambda_T = \Phi(\bar{\mathbf Y}_T)$ and $\lambda_T = \Phi(0)$. 
Define the local event
\begin{equation*}
\xi = \{|\bar Y_{a,T}|\le f_T/\sqrt{n_{a,T}}, \ \forall a=1,\ldots,K\}
\end{equation*}
We next apply the first-order Taylor formula with integral remainder to the one-dimensional function $\Phi(x)$ at $x=\bmu$, and the resulting remainder is negligible under event $\xi$. Consequently, 
\begin{equation}\label{eq:common level taylor expansion}
    \tilde\lambda_T-\lambda_T = \bar Y_{w,T} + \mathcal{R}_{\lambda,T} \qquad \mathcal{R}_{\lambda,T} =  o(\frac{1}{\sqrt T}) 
\end{equation}
where $\bar Y_{w,T} = \sum_{b=1}^K \frac{n_{b,T}^{3/2}}
    {\sum_{c=1}^K n_{c,T}^{3/2}}\bar Y_{b,T}$ ,$\mathcal{R}_{\lambda,T} = \int_0^1 (1-\theta) \sum_{b=1}^K\sum_{c=1}^K
\frac{\partial^2 \lambda_T(\theta\bar Y_T)}{\partial \bar Y_{b,T}\partial \bar Y_{c,T}}\bar Y_{b,T}\bar Y_{c,T}\,d\theta .$
We next expand the empirical fluid counts. Again by \eqref{eq:definition of n} and \eqref{eq:definition of tn}, 
\begin{equation*}
     \tilde n_{a,T} = \frac{f_T^2}{(\tilde\lambda_T-\mu_a-\bar Y_{a,T})^2} \qquad n_{a,T} = \frac{f_T^2}{(\lambda_T-\mu_a)^2}
\end{equation*}
This implies that there exists a function $\Psi(x)$ such that  $\tilde n_{a,T} = \Psi(\tilde\lambda_T-\lambda_T-\bar Y_{a,T})$ and $n_{a,T} = \Psi(0)$. Similarly, we apply the first-order Taylor formula with integral remainder to the one-dimensional function $\Psi(x)$ at $x=0$, and the resulting remainder is negligible under event $\xi$. Consequently, 
\begin{equation*}
     \tilde n_{a,T}-n_{a,T} = \frac{2n_{a,T}^{3/2}}{f_T}
    \left\{ \bar Y_{a,T}-(\tilde\lambda_T-\lambda_T)
    \right\} + \mathcal{R}_{n,a,T}, \qquad \mathcal{R}_{n,a,T} = o\left(\frac{n_{a,T}}{f_T}\right)
\end{equation*}
where $\mathcal{R}_{n,a,T} = \int_0^1(1-\theta)\Psi''(tD_{a,T})D_{a,T}^2\,d\theta, D_{a,T} = \tilde\lambda_T-\lambda_T-\bar Y_{a,T}$. Combining with \eqref{eq:common level taylor expansion}, we have
\begin{equation}\label{eq:tn taylor expansion}
    \tilde n_{a,T}-n_{a,T} = \frac{2n_{a,T}^{3/2}}{f_T}
    \left( \bar Y_{a,T} - \bar Y_{w,T}\right) + \mathcal{R}_{a,T} \qquad  \mathcal{R}_{a,T} = o\left(\frac{n_{a,T}}{f_T}\right)
\end{equation}
where $\mathcal{R}_{a,T} = -\frac{2n_{a,T}^{3/2}}{f_T}\mathcal{R}_{\lambda,T} + \mathcal{R}_{n,a,T}$.
Thus, for any $y > 0$
\begin{equation*}
    \Pp\left(\frac{|\tilde n_{a,T} - n_{a,T}|}{n_{a,T}/f_T} \geq y\right) \leq \Pp\left(\frac{n_{a,T}^{3/2}}{f_T}
    \left|\bar Y_{a,T} - \bar Y_{w,T}\right| \geq y \right) + \Pp\left( \left|\mathcal{R}_{a,T}\right| \geq y/2 \right) +  \Pp\left( \xi^c\right)
\end{equation*}
By subgaussian concentration and $n_{i,T} \asymp T/m$ and $n_{j,T} \asymp f_T^2$, where $i \in \mathcal{O}, j \notin \mathcal{O}$ from Lemma~\ref{lem:verification of conditons1-2 for ucb1}, there exists constant $c,C >0$ such that 
\begin{equation*}
    \Pp\left(\frac{|\tilde n_{a,T} - n_{a,T}|}{n_{a,T}/f_T} \geq y\right) \leq  C\exp\{-cy^2\} + C\exp\{-cf_T^2\} 
\end{equation*}
\end{proof}
\begin{proof}[Proof of Lemma~\ref{lem:ucb-suboptimal-upper-tail}]
If \(N_j(T)\ge y\), let \(\tau\) be a decision time at which the suboptimal arm \(j\) is selected, its pre-decision count satisfies \(N_j(\tau-1)=y\), and the index of arm \(j\) is at least as large as that of every optimal arm.  Fix one optimal arm $i \in\mathcal O$ and write $q = N_i(\tau-1)$.  The UCB comparison gives
\begin{equation*}
\barY_j(y)-\barY_i(q)
\ge \Delta_j+f_\tau(q^{-1/2}-y^{-1/2}).
\end{equation*}
For $y \ge Ln_{j,T}$, choose $L$ so large that $f_T/\sqrt{y} \le \Delta_j/4$. Then
\begin{equation*}
    \barY_j(y)-\barY_i(q)\ge 3\Delta_j/4+f_\tau/\sqrt q.
\end{equation*}
Thus either
\begin{equation*}
    \barY_j(y)\ge\Delta/4
    \qquad\text{or}\qquad
    -\barY_i(q) \ge \Delta/2+\frac{f_\tau}{\sqrt q}.
\end{equation*}
The first event has probability at most $\exp\{-cy\}$ for some constant $c > 0$ by Subgaussian concentration. For the second event, the time $\tau$ is random, but on the event under consideration it satisfies $\tau\ge y+q$; since $f_t$ is increasing, the event implies the deterministic-threshold event for each fixed $q$,
\begin{equation*}
    -\barY_i(q) \ge \Delta/2+\frac{f_{y+q}}{\sqrt q}.
\end{equation*}
Since $f_{y+q} = \sqrt{\rho\log(y+q)}$ and $\rho/\sigma^2 > 2$, there exists a constant $c>0$ such that this event has probability at most
$\exp\{-cq\}(y+q)^{-\beta}$ for $\beta < \rho/(2\sigma^2)$. Summing over $q$ gives $Cy^{-\beta}$ for some constant $C > 0$. The upper tail bound is obtained by combining the probabilities of the above two events.

Furthermore, for $\{\sum_{r \notin \mathcal O} N_r(T)\geq y\}$, since number of suboptimal arms is fixed, a union bound over the fixed set of suboptimal arms completes the proof.
\end{proof}

\begin{proof}[Proof of Lemma \ref{lem:ucb-suboptimal-lower-tail}]
Recall that $H =\{\sum_{r \notin \mathcal O} N_r(T) \le T/4\}$. On $H$, the optimal arms receive at least $3T/4$ total pulls. Hence there exists an optimal arm $i\in\mathcal O$ has final count at least $3T/4m$. Let $q=3T/4m$. Define $\tau$ as the decision time at which the optimal arm $i$ is selected, its pre-decision count satisfies $N_i(\tau-1)=q$ and the index of arm $r$ is at least as large as that of the suboptimal arm $j$. For notation simplification, write $\ell=N_j(\tau-1)$. On the event $\{N_j(T)< L_1n_{j,T}\}$, one also has $\ell \le L_1n_{j,T}$. The UCB comparison gives
\begin{equation*}
    \bar Y_i(q)-\bar Y_j(\ell)
    \ge
    -\Delta_j
    +f_\tau\left(\frac1{\sqrt \ell}-\frac1{\sqrt q}\right).
\end{equation*}
Taking small enough constant $L_1$ such that $f_T/\sqrt{L_1n_{j,T}} > \Delta_j$. Then, there exists constant $c$ such that for sufficiently large $T$,
\begin{equation*}
    f_\tau\left(\frac1{\sqrt \ell}-\frac1{\sqrt q}\right)-\Delta_j \ge c\frac{f_T}{\sqrt \ell},
\end{equation*}
which holds by $f_\tau/f_T=1+O(1/\log T)$,
$f_\tau/\sqrt q=o(1)$. Thus, $\{N_j(T)<L_1n_{j,T}\}\cap H$ implies that, for $\ell\le L_1n_{j,T}$, either
\begin{align*}
    \barY_i(q)
    \ge
    \frac{c}{2}\frac{f_T}{\sqrt \ell}
    \qquad\text{or}\qquad
    -\barY_j(\ell)
    \ge
    \frac{c}{2}\frac{f_T}{\sqrt \ell}.
\end{align*}
There exist constants $C,c>0$ such that the first event has probability at most
$Cn_{j,T}\exp\{-cT\}$ and the second event has probability at most $Cn_{j,T}\exp\{-cf_T^2\}$ after union over $\ell \le L_1n_{j,T}$ and the fixed set of optimal arms. Finally, $\Pp(H^c)\le CT^{-\beta}$ by Lemma \ref{lem:ucb-suboptimal-lower-tail}. The lower tail bound is proved by combining above results.
\end{proof}
\begin{proof}[Proof of Lemma~\ref{lem:ucb-suboptimal-witness}]
For notation simplification, define $d = n_{j,T}^{1/4}\log n_{j,T}$. There are two cases.

\emph{Case 1: $N_j(T)-\tilde n_{j,T}>d$. } 
Define $\tau$ as a decision time at which the suboptimal arm $j$ is selected and its pre-decision count satisfies $N_j(\tau-1)=\tilde n_{j,T}+d$. At this time, arm $j$ beats every optimal arm in terms of its UCB index. Fix an optimal arm $i \in\mathcal O$, and write $\ell=N_j(\tau-1), q=N_i(\tau-1).$
The UCB comparison gives
\begin{equation*}
    \bar Y_j(\ell)+\frac{f_\tau}{\sqrt \ell}
        \ge \Delta_j + \bar Y_i(q)+\frac{f_\tau}{\sqrt q}.
\end{equation*}
Recall the definition of the empirical fluid approximation, we have
\begin{equation*}
    \barY_j(n_{j,T}) - \barY_i(n_{i,T}) =  \Delta_j + \frac{f_T}{\sqrt{\tilde n_{i,T}}} - \frac{f_T}{\sqrt{\tilde n_{j,T}}}.
\end{equation*}
This gives
\begin{align}\label{eq:ucb-suboptimal-witness-1}
       \{\barY_j(\ell) - \barY_j(n_{j,T})\}+\{\barY_i(n_{i,T}) - \barY_i(q)\} \ge
       f_\tau\left(\frac{1}{\sqrt q}-\frac{1}{\sqrt {\ell}}\right)
        +f_T\left(\frac{1}{\sqrt{\tilde n_{j,T}}}-\frac{1}{\sqrt{\tilde n_{i,T}}}\right),
\end{align}
where
\begin{equation*}
    - \frac{f_\tau}{\sqrt {\ell}}+\left(\frac{f_T}{\sqrt{\tilde n_{j,T}}}-\frac{f_T}{\sqrt{\tilde n_{i,T}}}\right) \ge \frac{ f_T d}{2(\tilde n_{j,T}+d)^{3/2}} - \frac{f_T}{\sqrt{\tilde n_{i,T}}},
\end{equation*}
since $\ell = \tilde n_{j,T}+d$ and $f_t(t \geq 1)$ is an increasing function. Under events $\mathcal{R}_j \cap A \cap H$, we have $\tilde n_{i,T} \ge n_{i,T} - \xi T$, $\tilde n_{j,T} \le   n_{j,T}+\sqrt{n_{j,T}} \log n_{j,T}$. Combining with event \eqref{eq:ucb-suboptimal-witness-1}, there exists a constant $c$ such that for all sufficiently large $T$, at least one of the following three events occurs:
\begin{equation*}
\left\{ \barY_j(\ell)-\barY_j(n_{j,T}) \ge
\frac{cf_T d}{n_{j,T}^{3/2}} \right\}
\quad \text{or} \quad
\left\{ -\barY_i(q) \ge \frac{f_\tau}{\sqrt q}\right\}
\quad \text{or} \quad
\left\{ \barY_i(n_{i,T}) \ge \frac{c f_T d}{n_{j,T}^{3/2}} \right\}
\end{equation*}

We first consider the first event, under $\{\max_{s \notin \mathcal{O}} |\tilde n_{s,T} - n_{s,T}| \le \sqrt{n_{j,T}} \log n_{j,T} \}$, 
\begin{equation*}
    |\ell - n_{j,T}| \le \sqrt{n_{j,T}} \log n_{j,T}+d
\end{equation*}
Since $n_{j,T} \asymp f_T^2$ by Lemma~\ref{lem:verification of conditons1-2 for ucb1}, Lemma~\ref{lem:local-modulus} implies that there exist constants $c,C>0$ such that
\begin{equation*}
    \Pp\left\{\Omega(\sqrt{n_{j,T}}\log n_{j,T}+d)\ge \frac{cf_Td}{n_{j,T}^{3/2}}\right\}
        \le C\exp\left\{-c(\log n_{j,T})^3\right\}.
\end{equation*}

For the second event, since $1 \le q \le T$, we peel $q$ into geometric blocks $[u,(1+b)u]$, with $b>0$ small enough that $\frac{\rho}{2\sigma^2(1+b)}>1$. Note that we have assumed $\frac{\rho}{2\sigma^2} > 1$. Since $\tau \geq q+n_{j,T}/3$ by the definition, in each block, and there exits constant $C$ such that the block probability at most
\begin{equation*}
    \exp\left\{-\frac{u f_{u+n/3}^2}{2\sigma^2(1+b)u}\right\}
        \le C(u+n_{j,T})^{-\beta}
\end{equation*}
by Lemma \ref{lem:maximal inequality} and $f_t = \sqrt{\rho \log T}$. There are $O(\log n_{j,T})$ blocks with $q \le T$, and their total contribution is $C(\log n_{j,T})n_{j,T}^{-\beta}$ by union bound. 

For the third event, by $n_{r,T} \asymp T$,$n_{j,T} \asymp f_T^2$ and subgaussian concentration, there exists constant $c,C$ such that the probability of this event is at most $C\exp\{-cT\}$. Combining above results, we otbain the probability bound.

\emph{Case 2: $\tilde n_{j,T}-N_j(T)>d$.}  Under event $H= \{\sum_{s\notin \mathcal{O} }N_{s,T} \le T/4\}$, the other optimal arms receive at least $3T/(4m)$  and there exists an optimal arm $i$ has final count at least $3T/(4m)$.  Define $\tau$ as a decision time at which the optimal arm $i$ is selected and its pre-decision count satisfies $N_i(\tau-1)=3T/(4m)$. For notation simplification, write $q=N_i(\tau-1),\ell=N_j(\tau-1),$. The UCB comparison gives
\begin{equation*}
    \barY_i(q)+\frac{f_\tau}{\sqrt q}
        \ge -\Delta_j + \barY_j(\ell)+\frac{f_\tau}{\sqrt \ell}.
\end{equation*}
Recall the definition of the empirical fluid approximation, we have
\begin{equation*}
    \barY_j(n_{j,T}) - \barY_i(n_{i,T}) = \Delta_j + \frac{f_T}{\sqrt{\tilde n_{i,T}}} - \frac{f_T}{\sqrt{\tilde n_{j,T}}}.
\end{equation*}
This gives
\begin{align}\label{eq:ucb-suboptimal-witness-lower-1}
       \{\barY_j(n_{j,T}) - \barY_j(\ell)\}+\{\barY_i(q)-\barY_i(n_{i,T})\} \ge
       f_\tau\left(\frac1{\sqrt {\ell}} - \frac1{\sqrt q}\right)
        +f_T\left(\frac{1}{\sqrt{\tilde n_{i,T}}} -\frac{1}{\sqrt{\tilde n_{j,T}}}\right)
\end{align}
If $\ell \le \alpha n_{j,T}$. For all sufficiently large $T$, $\tilde n_{j,T}\in[n_{j,T}/2,2n_{j,T}]$ under event $A_h$,  $f_\tau/f_T=1+O(1/\log T)$, while $f_T/\sqrt{\ell}\ge f_T/\sqrt{\alpha n_{j,T}}$. Choosing a sufficiently small $\alpha$ such that $f_T/\sqrt{\alpha n_{j,T}} \ge 3\Delta_j$, then there exists a constant $c>0$ such that
\begin{equation*}
        \frac{f_\tau}{\sqrt \ell}-\frac{f_T}{\sqrt{\tilde n_{j,T}}}  = \frac{f_T}{2\sqrt{\ell}}- \frac{f_T - f_\tau}{\sqrt{\ell}} + \frac{f_T}{2\sqrt{\ell}} - \frac{f_T}{\sqrt{\tilde n_{j,T}}} 
        \ge 2c\frac{f_T}{\sqrt \ell}+\Delta_j .
\end{equation*}
Combining with event \eqref{eq:ucb-suboptimal-witness-lower-1} and $\tilde n_{i,T} < T$, at least one of the following four events occurs:
\begin{equation*}
\left\{ - \barY_j(\ell)\ge c\frac{f_T}{\sqrt \ell}
 \right\}
\quad \text{or} \quad
\left\{ \barY_j(n_{j,T}) \ge c\frac{f_T}{\sqrt \ell}
 \right\}
\quad \text{or} \quad
\left\{ -\barY_i(n_{i,T}) \ge \frac{f_T}{\sqrt{T}} \right\}
\quad \text{or} \quad
\left\{ \barY_i(q) \ge \Delta_j \right\}
\end{equation*}

For the first event, since $l \leq \alpha n_{j,T}$ and we peel $q$ into geometric blocks $(u,2u]$. In each block, and there exits constant $c,C>0$ such that the block probability at most $C\exp\{-cf_T^2\}$ by Lemma \ref{lem:maximal inequality}. There are $O(\log n_{j,T})$ blocks and their total contribution is $C\log n_{j,T}\exp\{-cf_T^2\}$ by union bound. 

For the other three events, by subgaussian concentration, $l \leq \alpha n_{j,T}$ and $q = \frac{3T}{4m}$,  there exits constant $c>0$ such that the probability of three events at most $3\exp\{-cf_T^2\}$.

We next consider $\alpha n_{j,T} \le \ell \le \tilde n_{j,T} - d$. Since $f_t = \sqrt{\rho\log t}$ for any $t\geq 1$ and $\tau \geq \frac{3T}{4m}$, $f_\tau/f_T=1+O(1/\log T)$. Combining with $ \ell =N_j(\tau-1)$, $ \ell = O(n_{j,T})$ under event $A_h$, $n_{j,T}=O(f_T^2)$, $\tilde n_{j,T} - N_j(\tau-1) \ge d$ and $d = n_{j,T}^{1/4}\log n_{j,T}$, we have $(f_T - f_\tau)/\sqrt{\ell}=o(\frac{f_Td}{n_{j,T}^{3/2}})$. Thus, there exist a constant $c$ such that 
\begin{equation*}
        \frac{f_\tau}{\sqrt \ell}-\frac{f_T}{\sqrt{\tilde n_{j,T}}} = \frac{f_T}{\sqrt{\ell}}-\frac{f_T}{\sqrt{\tilde n_{j,T}}} - \frac{f_T - f_\tau}{\sqrt{\ell}} 
        \ge \frac{cf_T(\tilde n_{j,T} - N_j(\tau-1))}{n_{j,T}^{3/2}}.
\end{equation*}
Define the dyadic grid
\begin{equation*}
    \mathcal D(d) = \left\{ 2^s d/4: s \ge 0,2^s d/4 \le n_{j,T}\log n_{j,T} \right\}.
\end{equation*}
Choose $D \in\mathcal D(d)$ so that $D \le \tilde n_{j,T} - N_j(\tau-1) <2D$.
Under event $ A_h = \{ \max_{s\notin\mathcal O}|\tilde n_{s,T} - n_{s,T}| \le h\}$, 
\begin{equation*}
    |\ell - n_{j,T}| \le h+2D,
\end{equation*}
Combining with event \eqref{eq:ucb-suboptimal-witness-lower-1}, there exists a constant $c$ such that for all sufficiently large $T$, at least one of the following three events occurs:
\begin{equation*}
\left\{ \barY_j(n_{j,T}) - \barY_j(\ell)\ge \frac{c f_T D}{n_{j,T}^{3/2}}
 \right\}
\quad \text{or} \quad
\left\{ \barY_i(q) \ge \frac{c f_T d}{n_{j,T}^{3/2}} \right\}
\quad \text{or} \quad
\left\{ -\barY_i(n_{i,T}) \ge \frac{2f_T}{\sqrt{n_{i,T}}} \right\},
\end{equation*}
where $D \in D(d)$. For fixed $D$ and $j$, Lemma \ref{lem:local-modulus} bounds the probability of the above event by $C\exp\left\{-c\frac{f_T^2D^2}{n_{j,T}(h+2D)}\right\}$. The dyadic grid $D(d)$ contains $O(\log T)$ points, $h = \sqrt{n_{j,T}} \log n_{j,T}$ and $d = n_{j,T}^{1/4}\log n_{j,T}$. Combining the above observations yields constants $C,c>0$ such that the probability bound of this event is $C\log n_{j,T} \exp\{-c (\log n_{j,T})^3\}$.

For the other two events, by subgaussian concentration and $q = \frac{3T}{4m}$,  there exits constant $c$ such that the probability of three events at most $3\exp\{-c f_T^2\}$. Combining above results, we otbain the probability bound.
\end{proof}

\begin{proof}[Proof of Lemma~\ref{lem:verification of conditon3 for ucb1 suboptimal arm}]
For any fixed suboptimal arm $j \notin \mathcal{O}$, we first verify the condition about $\Gamma_{j,T}$. By Definition~\ref{def:exploration rate}, $\Gamma_{j,T} = \sqrt{\rho\log T}$ in UCB1 and $\Gamma_{j,T} = \omega(1)$. 

Next, we focus on other parts of Condition~\ref{ass:concentration} .By the tail bound for empirical fluid approximations in Lemma~\ref{lem:Tail bound of fluid approximation}, we directly obtain that for any $s>1$,
\begin{equation*}
    \frac{f_T}{\sqrt{n_{j,T}}}\norm{\tilde n_{j,T}-n_{j,T}}_s=O(1) 
\end{equation*}
For notation simplification, define event  $D_j = \{|N_i(T)-\tilde n_{i,T}| > n_{j,T}^{1/4}\log n_{j,T}\}$ for arm $j \notin \mathcal{O}$.
Next, fix any $1<p<\beta$. We decompose
\begin{align*}
\norm{N_j(T)-\tilde n_{j,T}}_p
&\leq 
\norm{|N_j(T)-\tilde n_{j,T}|\mathbb{I}_{\mathcal R^c}}_p
+\norm{|N_j(T)-\tilde n_{j,T}|\mathbb{I}_{\mathcal R\cap(A^c\cup H^c\cup U^c)}}_p  \\
&\quad
+\norm{|N_j(T)-\tilde n_{j,T}|\mathbb{I}_{\mathcal R\cap A\cap D_j^c}}_p
+\norm{|N_j(T)-\tilde n_{j,T}|\mathbb{I}_{\mathcal R\cap A\cap D_j\cap H\cap U}}_p .
\end{align*}
On $\mathcal R^c$, we first use
\begin{align*}
\norm{|N_j(T)-\tilde n_{j,T}|\mathbb{I}_{\mathcal R^c}}_p \leq  \norm{N_j(T)\mathbb{I}_{\mathcal R^c}}_p + \norm{\tilde n_{j,T}\mathbb{I}_{\mathcal R^c}}_p .
\end{align*}
The first term on the right is $o(\sqrt{n_{j,T}})$ by the tail bounds for
$\mathcal R^c$ from Lemmas~\ref{lem:ucb-suboptimal-upper-tail} and
\ref{lem:ucb-suboptimal-lower-tail}, together with the layer-cake representation. For the second term, writing $\tilde n_{j,T}=n_{j,T}+(\tilde n_{j,T}-n_{j,T})$ and using
$f_T\sqrt{n_{j,T}}\norm{\tilde n_{j,T}-n_{j,T}}_s=O(1)$ together with
$\mathbb P(\mathcal R^c)=o(1)$ shows that it is also $o(\sqrt{n_{j,T}})$. Hence the contribution from $\mathcal R^c$ is $o(\sqrt{n_{j,T}})$.

Similarly, on $\mathcal R\cap(A^c\cup H^c\cup U^c)$, we use
\begin{align*}
&\norm{|N_j(T)-\tilde n_{j,T}|
\mathbb{I}_{\mathcal R\cap(A^c\cup H^c\cup U^c)}}_p  \\
&\qquad\leq
\norm{N_j(T)\mathbb{I}_{\mathcal R\cap(A^c\cup H^c\cup U^c)}}_p + \norm{\tilde n_{j,T}\mathbb{I}_{\mathcal R\cap(A^c\cup H^c\cup U^c)}}_p .
\end{align*}
The first term is $o(\sqrt{n_{j,T}})$ by the growth bound $N_j(T)=O(n_{j,T}\log n_{j,T})$ on $\mathcal R$ and the probability bounds for
$A^c\cup H^c\cup U^c$ from Lemmas~\ref{lem:Tail bound of fluid approximation} and~\ref{lem:ucb-suboptimal-upper-tail}. The second term is handled in the same way as above, using
$\tilde n_{j,T}=n_{j,T}+(\tilde n_{j,T}-n_{j,T})$,
$f_T\sqrt{n_{j,T}}\norm{\tilde n_{j,T}-n_{j,T}}_s=O(1)$, and the same negligible probability bound. Therefore, this contribution is also $o(\sqrt{n_{j,T}})$.

On $\mathcal R\cap A\cap D_j^c$, the deterministic bound $|N_j(T)-\tilde n_{j,T}|\leq n_{j,T}^{1/4}\log n_{j,T}$ immediately makes the
corresponding contribution $o(\sqrt{n_{j,T}})$.

Finally, on $\mathcal R\cap A\cap D_j \cap H\cap U$, Lemma~\ref{lem:ucb-suboptimal-witness}
shows that the corresponding $L^p$ contribution is $o(\sqrt{n_{j,T}})$.
Therefore,
\begin{align*}
   \frac{f_T}{n_{j,T}}\norm{N_j(T)-\tilde n_{j,T}}_p=o(1). 
\end{align*}
Finally, for any $r>1$,
\begin{align*}
    \norm{\bar X_j(N_j(T))-\bar X_j(n_{j,T})}_r & \leq \norm{|\bar X_j(N_j(T))-\bar X_j(n_{j,T})|\mathbb{I}_{\{|N_j(T) - n_{j,T}|\leq n_{j,T} /\sqrt{f_T}\}}}_r \\
    & \ \ + \norm{|\bar X_j(N_j(T))-\bar X_j(n_{j,T})|\mathbb{I}_{\{|N_j(T) - n_{j,T}|> n_{j,T} /\sqrt{f_T}\}}}_r 
\end{align*}
For the first term, Lemma~\ref{lem:local-modulus} applies directly on the local event
$\{|N_j(T)-n_{j,T}|\leq n_{j,T}/\sqrt{f_T}\}$ and shows that this term is
$o(1/\sqrt{n_{j,T}})$. For the second term, Lemma~\ref{lem:ucb-suboptimal-upper-tail} and \ref{lem:ucb-suboptimal-lower-tail} gives the probability of $|N_j(T)-n_{j,T}| > n_{j,T}/\sqrt{f_T}$ is $O(1/n_{j,T}^\beta)$ with $\beta > 1$. Therefore,
\begin{align*}
    \sqrt{n_{j,T}}\norm{\bar X_j(N_j(T))-\bar X_j(n_{j,T})}_r=o(1).
\end{align*}
\end{proof}

\begin{proof}[Proof of Lemma~\ref{lem:ucb-opt-low-imbalance}]
We first consider optimal low-count event. 
\begin{equation*}
    B_\kappa = \Bigl\{\min_{i\in O}N_i(T) \ge \kappa T,\quad \sum_{j \notin \mathcal{O}} N_j(T) \le T/f_T^2\Bigr\}
\end{equation*}
The suboptimal part of $B_T^c$ is controlled by Lemma~\ref{lem:ucb-suboptimal-upper-tail}, then it remains to prove optimal part. Consider event $H=\{\sum_{j\notin \mathcal O}N_j(T) \le T/4\}$. By Lemma~\ref{lem:ucb-suboptimal-upper-tail}, there exists constant $C$ such that $\Pp(H^c)\le CT^{-\beta}$ with $1 < \beta< \rho/2\sigma^2$. Next, consider event ${N_i(T) < \kappa T}$ on $H$ for the fixed optimal arm $i$.  If $\kappa$ is small, the other optimal arms receive at least $(3/4-\kappa)T$ pulls, so some $r\in O\setminus\{i\}$ has final count at least $\frac{3/4-\kappa}{m-1}T$. Let $q = \frac{3/4-\kappa}{m-1}T$. Define $\tau$ as the decision time at which the optimal arm $r$ is selected and its pre-decision count satisfies $N_r(\tau-1)=q$. For notation simplification, , write $\ell=N_i(\tau-1)$. The UCB comparison gives
\begin{equation*}
\barY_r(q)-\barY_i(\ell)\ge f_\tau(\frac{1}{\sqrt{\ell}}-\frac{1}{\sqrt{q}}).
\end{equation*}
Since $\ell \leq \kappa T$, $q = \frac{3/4-\kappa}{m-1}T$, $\tau \geq q+\ell$, then there exist $\xi > 0$ such that $f_\tau \geq \sqrt{(1-\xi) \rho \log T}$ for sufficiently large $T$. We next cover the possible values of $\ell$ and $q$ by multiplicative ranges. Fix a small constant $b>0$. For range endpoints $a$ and $c$, consider
\begin{equation*}
    \ell \in(a,(1+b)a],\qquad q\in(c,(1+b)c].
\end{equation*}
There are $O((\log T)^2)$ range pairs. On each pair, 
\begin{equation*}
    f_\tau\left(\frac{1}{\sqrt{\ell}}-\frac{1}{\sqrt{q}}\right) \geq d_{a,c,\xi,\rho},\qquad d_{a,c,\xi,\rho} = \sqrt{(1-\xi) \rho \log T}\left(\frac{1}{\sqrt{(1+b)a}}-\frac{1}{\sqrt{c}}\right)
\end{equation*}
and for $0< \theta <1$, The above UCB comparison event implies that either 
\begin{equation*}
    \barY_r(q) \ge \theta d_{a,c,\xi,\rho} \quad \text{or} \quad -\barY_i(\ell) \ge (1-\theta)d_{a,c,\xi,\rho}
\end{equation*}
For the first branch, the event implies
\begin{equation*}
\max_{c < q \le (1+b)c}S_r(q)\ge c\,\theta d_{a,c,\xi,\rho};
\end{equation*}
for the second branch, it implies
\begin{equation*}
\max_{a < \ell \le (1+b)a}(-S_i(\ell))\ge a(1-\theta)d_{a,c,\xi,\rho}.
\end{equation*}
Using Lemma~\ref{lem:maximal inequality} and choosing
\begin{equation*}
    \theta=\frac{\sqrt{\phi_{a,c}(1+b)}}{1+b+\sqrt{\phi_{a,c}(1+b)}},\qquad \phi_{a,c}=\frac{(1+b)a}{c},
\end{equation*}
then there exists constant $C$ such that the range probability is therefore bounded by
\begin{equation*}
    C\exp\left\{-\frac{\rho}{2\sigma^2}(1-\xi)\chi(\phi_{a,c},b)\log T\right\},
\end{equation*}
where
\begin{equation*}
    \chi(\phi_{a,c}):=
        \frac{(1-\sqrt{\phi_{a,c}})^2}{\left(1+b+\sqrt{\phi_{a,c}(1+b)}\right)^2},
        \qquad
\chi(\phi_{a,c})\longrightarrow1\quad\text{as }\phi_{a,c},b\downarrow 0.
\end{equation*}
By taking $\kappa$ small, all possible ratios $\phi_{a,c}$ are bounded by a number $\phi_{\kappa,m}\downarrow0$ as $\kappa\downarrow0$, and $d_{a,c,\xi,\rho}>0$. Moreover, by choosing $\xi,b$ small and then $\kappa$ small so that
\begin{align*}
    \frac{\rho}{2\sigma^2}(1-\xi)\inf_{0\le \phi_{a,c} \le \phi_{\kappa,m}}\chi(\phi_{a,c},b)>\beta,
\end{align*}
where $1 < \beta < \frac{\rho}{2\sigma^2}$.
Summing over the $O((\log T)^2)$ range pairs and finitely many optimal arms gives
\begin{align*}
    \Pp\{\min_{i\in O}N_i(T)<\kappa T,\ H\}\le C(\log T)^2T^{-\beta}.
\end{align*}
Together with probability of $H^c$ and suboptimal arm event proves the bound on $B_T^c$.

Next, we consider imbalanced optimal-selection event. 
\begin{equation*}
W_\eta = \{\exists t\le T,\ \exists r,s\in O:\ A_t=r,\ N_r(t-1)\ge4\eta T,\ N_s(t-1)\le\eta T\}
\end{equation*}
For notation simplification, write $q=N_r(t-1), \ell=N_s(t-1)$. At time $t$, arm $r$ is selected by the definition of $W_\eta$. Therefore, the UCB comparison gives
\begin{equation*}
     \barY_r(q)-\barY_s(\ell)
        \ge f_t\left(\frac1{\sqrt\ell}-\frac1{\sqrt q}\right)
        \ge \frac{f_t}{2\sqrt\ell}.
\end{equation*}
Since the two arms have the same mean and $t \ge q +\ell \ge 4\eta T$, so there exists a constant $c$ such that $f_t^2\ge c f_T$ for all large $T$.  Therefore, at time $t$, either
\begin{equation*}
    \barY_r(q)\ge \frac{c}{4}\frac{f_T}{\sqrt\ell}
        \qquad\text{or}\qquad
        -\barY_s(\ell)\ge \frac{c}{4} \frac{f_T}{\sqrt\ell}.
\end{equation*}
For the first branch, we use $\ell \le \eta T$ and peel $q \in [4\eta T,T]$ into dyadic ranges, producing $O(\log T)$ ranges. Since there are only finitely many ordered pairs $(r,s)$, the resulting union bound contributes only a constant factor. 
For the second branch, we similarly peel $\ell$ into dyadic ranges, again producing $O(\log T)$ ranges and only finitely many ordered pairs $(r,s)$. Combining with subgaussian concentration and $f_T = \sqrt{\rho \log T}$, there exist constants $c,C>0$ such that
\begin{equation*}
    \Pp(W_\eta) \leq C\log T T^{-c},
\end{equation*}
completing the proof.
\end{proof}
\begin{proof}[Proof of Lemma~\ref{lem:ucb-optimal-witness}]
There are two cases.

\emph{Case 1: $N_i(T)-\tilde n_{i,T}>d$.}  Define $\tau$ as the decision time at which the optimal arm $i$ is selected and its pre-decision count satisfies $N_i(\tau-1) = \tilde n_{i,T}+d$. At time $\tau-1$ the total deficit of the other optimal arms relative to their balance counts is
\begin{equation*}
     \sum_{r\in\mathcal O\setminus\{i\}}(\tilde n_{r,T}-N_r(\tau-1))
        =T-\tilde n_{i,T}-\sum_{r\ne i}N_r(\tau-1)
        =T - \left(\tau-1-\sum_{j \notin \mathcal{O}}N_j(\tau-1)-d \right),
\end{equation*}
where $\sum_{j \notin \mathcal{O}}N_j(\tau-1)$ is the number of suboptimal pulls before $\tau$. Hence, there exists an optimal arm $r \in\mathcal O\setminus\{i\}$ has
\begin{equation*}
     \tilde n_{r,T} - N_r(\tau-1) \ge \frac{d}{m-1}.
\end{equation*}
For notation simplification, write $q=N_r(\tau-1)$ and $\ell = N_i(\tau-1)$. Since $\max_{r \in \mathcal O}|\tilde n_{r,T} - n_{r,T}| \le g \le \frac{T}{f_T^{1/4}}$, there exist a constant such that every balance count is at least $4\eta T$.  Hence the arm $r$ has $ \ell \ge 4\eta T$ for all large $T$.  If $q \le \eta T$, then the decision at time $\tau$ belongs to $W_\alpha$, which is excluded.  Hence $q \ge \eta T$.  Both $\ell$ and $q$ are at most $T$, so they lie in the interval $(\eta T,T]$.

Because arm $i$ is selected at time $\tau$, the UCB index comparison gives
\begin{equation*}
    \barY_i(\ell)+\frac{f_\tau}{\sqrt \ell}
        \ge \barY_r(q)+\frac{f_\tau}{\sqrt q}.
\end{equation*}
Recall the definition of $\tilde n_{i,T}$, we have
\begin{equation*}
    \barY_i(n_{i,T})+\frac{f_T}{\sqrt{\tilde n_{i,T}}}
        = \barY_r(n_{r,T})+\frac{f_T}{\sqrt{\tilde n_{r,T}}}.
\end{equation*}
This gives
\begin{align}\label{eq:ucb-optimal-witness-1}
       \{\barY_i(\ell) - \barY_i(n_{i,T})\}+\{\barY_r(n_{r,T}) - \barY_r(q)\} \ge
        f_\tau\left(\frac1{\sqrt q}-\frac1{\sqrt {\ell}}\right)
        +f_T\left(\frac1{\sqrt{\tilde n_{i,T}}}-\frac1{\sqrt{\tilde n_{r,T}}}\right).
\end{align}
Since $\ell$ and $q$ both lie in the interval $(\eta T,T]$, then there exists constant $c$ such that for sufficiently large $T$, $f_\tau \ge cf_T$. Moreover, by $\ell = \tilde n_{i,T} + d$, $\tilde n_{r,T} - N_r(\tau-1) \ge \frac{d}{m-1}$ and $q = N_r(\tau-1)$, there exists a constant $c$ such that 
\begin{equation}\label{eq:ucb-optimal-witness-2}
    \frac{1}{\sqrt q} - \frac{1}{\sqrt{\tilde n_{r,T}}} \ge c\frac{\tilde n_{r,T}-N_r(\tau-1)}{\tilde  n_{r,T}^{3/2}} \qquad \frac{1}{\sqrt{\tilde n_{i,T}}} - \frac{1}{\sqrt \ell} \ge c\frac{d}{\tilde n_{i,T}^{3/2}}
\end{equation}
For $d>0$, define the dyadic grid
\begin{equation*}
    \mathcal D(d) = \left\{ 2^s d: s \ge 0,2^s d \le T \right\}.
\end{equation*}
Choose $D \in\mathcal D(d)$ so that $D \le d + \tilde n_{r,T}-N_r(\tau-1) <2D$.
Since $\tilde n_{r,T},\tilde n_{i,T}\le T$, equations~\eqref{eq:ucb-optimal-witness-1} and~\eqref{eq:ucb-optimal-witness-2} imply that, for sufficiently large $T$,
\begin{equation}
\label{eq:ucb-optimal-witness-3}
    \barY_i(\ell) - \barY_i(n_{i,T}) \ge c\frac{f_TD}{2T^{3/2}} \qquad \text{or} \qquad \barY_r(n_{r,T}) - \barY_r(q) \ge c\frac{f_TD}{2T^{3/2}}
\end{equation}
Moreover, under events $\{\max_{r\in\mathcal O}|\tilde n_{r,T} - n_{r,T}| \le g \}$,
\begin{equation}
\label{eq:ucb-optimal-witness-4}
    |\ell - n_{i,T}| \le g+2D \qquad |q-n_{r,T}|\le g+2D,
\end{equation}
Combining \eqref{eq:ucb-optimal-witness-3} and \eqref{eq:ucb-optimal-witness-4}, and applying Lemma~\ref{lem:local-modulus}, on events $B_\kappa \cap W_\eta^c \cap \{\max_{r\in\mathcal O}|\tilde n_{r,T} - n_{r,T}|\}$, 
\begin{equation*}
    \Omega_i(g+2D) \ge c\frac{f_TD}{2T^{3/2}} \qquad \Omega_r(g+2D) \ge c\frac{f_TD}{2T^{3/2}}
\end{equation*}
where $D \in D(d)$. For fixed $D$ and $r$, Lemma \ref{lem:local-modulus} bounds the probability of the above event by
$C\exp\left\{-c\frac{f_T^2D^2}{\sigma^2T(g+2D)}\right\}$.
Although the dyadic grid $D(d)$ contains $O(\log T)$ points, the exponent increases rapidly along the grid. Indeed, with $\psi(D)=D^2/(g+2D)$, the conditions $D\ge d$ and $d \ge \frac{\sqrt{n_{i,T}g}}{f_T}$ imply that, for sufficiently large $T$, $2<\psi(2D)/\psi(D)<4$. Hence the exponential terms are dominated by those at the first scale $d$. Since there are only $m$ choices of $r$, combining the above observations yields constants $C,c>0$ such that
\begin{equation*}
        \Pp\{A_T\le a,\ B_\kappa,W_\eta^c,\ N_i(T) - \tilde n_{i,T}>d\}
        \le C\exp\left\{-c\frac{f_T^2}{T}d\right\}
\end{equation*}

\emph{Case 2: $\tilde n_{i,T}-N_i(T)>d$.}
At time $T$ the total deficit of the other optimal arms relative to their balance counts is
\begin{equation*}
     \sum_{r\in\mathcal O\setminus\{i\}}(N_{r,T} - \tilde n_{r,T})
        = \tilde n_{i,T}- N_i(T) - \sum_{j \notin \mathcal{O}}N_j(T) + \sum_{j \notin \mathcal{O}}\tilde n_{i,T} \ge d/2
\end{equation*}
where $\sum_{j \notin \mathcal{O}}N_j(T)$ is the number of suboptimal pulls and the last inequality holds by $d \geq T/2f_T^2$ and $\sum_{j \notin \mathcal{O}}N_j(T) \leq T/f_T^2$ under $B_\kappa$. Thus, there exists a optimal arm $r$ such that 
\begin{equation*}
    N_{r,T} - \tilde n_{r,T} \ge \frac{d}{2(m-1)}
\end{equation*}
The remaining steps for obtaining the target bound are the same as in \emph{Case 1: $N_i(T)-\tilde n_{i,T}>d$}. We omit the details.
\end{proof}

\begin{proof}[Proof of Lemma~\ref{lem:verification of conditon3 for ucb1 optimal arm}]
For any fixed optimal arm $i \in \mathcal{O}$, we first verify the condition about $\Gamma_{i,T}$. By Definition~\ref{def:exploration rate}, $\Gamma_{i,T} = f_T = \sqrt{\rho\log T}$ in UCB1 and $\Gamma_{i,T} = \omega(1)$. 

By the tail bound for empirical fluid approximations in Lemma~\ref{lem:Tail bound of fluid approximation}, we directly obtain that for any $s>1$,
\begin{equation*}
    \frac{f_T}{n_{i,T}}\norm{\tilde n_{i,T}-n_{i,T}}_s=O(1) \quad \text{or} \quad \Gamma_{i,T}\norm{\frac{\tilde n_{i,T}-n_{i,T}}{n_{i,T}}}_s=O(1)
\end{equation*}
For notation simplification, define event  $\Upsilon_i = \{|N_i(T)-\tilde n_{i,T}|>d, d \ge \frac{\sqrt{n_{i,T}a}}{f_T} + \frac{T}{2f_T^2}\}$ for arm $i \in \mathcal{O}$.
Next,
\begin{align*}
    \norm{N_i(T)-\tilde n_{i,T}}_p & \leq \norm{|N_i(T)-\tilde n_{i,T}|\mathbb{I}_{\Upsilon_i^c}}_p + \norm{|N_i(T)-\tilde n_{i,T}|\mathbb{I}_{U_g\cap\ B_\kappa\cap W_\eta^c\cap\Upsilon_i}}_p \\
    & \ \ + \norm{|N_i(T)-\tilde n_{i,T}|\mathbb{I}_{B^c_\kappa \cup W_\eta \cup U_g^c}}_p 
\end{align*}
On the event $\Upsilon_i^c$, the desired bound 
$|N_i(T)-\tilde n_{i,T}|=o(T/f_T)$ holds directly. 
On the event $U_g \cap B_\kappa\cap W_\eta^c\cap\Upsilon_i$, Lemma~\ref{lem:ucb-optimal-witness}, applied with 
$g = T/\sqrt{f_T}$ for $U_g$, shows that the corresponding contribution to the $L^p$ norm is $o(T/f_T)$.  Finally, the remaining event 
$B_\kappa^c\cup W_\eta\cup U_g^c$ has probability 
$O(\log T \exp\{-cf_T^2\})$ for some constant $c>0$ by Lemmas~\ref{lem:Tail bound of fluid approximation} and~\ref{lem:ucb-opt-low-imbalance}; hence its contribution is negligible by H\"older's inequality. Therefore,
\begin{equation*}
     \frac{f_T}{n_{i,T}}\norm{N_i(T)-\tilde n_{i,T}}_p=o(1) \quad \text{or} \quad \Gamma_{i,T}\norm{\frac{N_i(T) - \tilde n_{i,T}}{n_{i,T}}}_p = o(1)
\end{equation*}
Finally, for any $r>1$,
\begin{align*}
    \norm{\bar X_i(N_i(T))-\bar X_i(n_{i,T})}_r & \leq \norm{|\bar X_i(N_i(T))-\bar X_i(n_{i,T})|\mathbb{I}_{\{|N_i(T) - n_{i,T}|\leq n_{i,T} /\sqrt{f_T}\}}}_r \\
    & \ \ + \norm{|\bar X_i(N_i(T))-\bar X_i(n_{i,T})|\mathbb{I}_{\{|N_i(T) - n_{i,T}|> n_{i,T} /\sqrt{f_T}\} \cap B_\kappa}}_r \\
    & \ \ + \norm{|\bar X_i(N_i(T))-\bar X_i(n_{i,T})|\mathbb{I}_{B^c_\kappa}}_r 
\end{align*}
For the first term, Lemma~\ref{lem:local-modulus} applies directly on the local event
$\{|N_i(T)-n_{i,T}|\leq n_{i,T}/\sqrt{f_T}\}$ and shows that this term is
$o(1/\sqrt{n_{i,T}})$. For the second term, the decomposition $|N_i(T)-n_{i,T}|
\leq |N_i(T)-\tilde n_{i,T}|+|\tilde n_{i,T}-n_{i,T}|$ together with Lemmas~\ref{lem:Tail bound of fluid approximation},
\ref{lem:ucb-opt-low-imbalance}, and~\ref{lem:ucb-optimal-witness}
gives the required tail bound for the non-local event. Hence, by Lemma~\ref{lem:maximal inequality} and H\"older's inequality, the second term is also $o(1/\sqrt{n_{i,T}})$. The third term is handled in the same way: Lemma~\ref{lem:ucb-opt-low-imbalance} gives $\mathbb{P}(B_\kappa^c)=O((\log T)^2T^{-\beta})$ for some $\beta>1$, and Lemma \ref{lem:maximal inequality} combined with H\"older's inequality makes its
contribution $o(1/\sqrt{n_{i,T}})$. Therefore,
\begin{align*}
    \sqrt{n_{i,T}}\norm{\bar X_i(N_i(T))-\bar X_i(n_{i,T})}_r=o(1).
\end{align*}
\end{proof}

\subsection{Poly-UCB}
By Lemma~\ref{lem:verification of conditons1-2 for ucb1}, Poly-UCB already satisfies  Conditions~\ref{ass:index-regularity} and~\ref{ass:index_regularity2}. Since the proof of Condition~\ref{ass:concentration} is very close to that for UCB1, we only give a proof sketch and omit
routine technical details. The key in both cases is to establish tail bounds for $N_a(T)$ and $N_a(T)-\tilde n_{a,T}$. The upper-tail argument for $N_a(T)$ is essentially identical for the two algorithms: in both cases, if a suboptimal arm is pulled too many times, then at some witness time its index must exceed that of an optimal arm, and this index comparison is ruled out by concentration. 

The main difference arises in the sharper tracking bound for $N_a(T)-\tilde n_{a,T}$. In the UCB1 proof, one can compare the index at the decision time $\tau$ to a terminal auxiliary target because $f_\tau$ is comparable to $f_T$ if $\tau = O(T)$.  
For Poly-UCB this step is invalid.  Indeed, for $f_t = t^\alpha$ even a suboptimal arm at $\tau \asymp T$ has $f_\tau - f_T \asymp T^\alpha,$ so the index error created by replacing $f_\tau$ with $f_T$ is of order $T^\alpha$.  Relative to the tracking scale needed for $N_a(T)-\tilde n_{a,T}=o(n_{a,T}/f_T)$, this is too large to be negligible. Hence the Poly-UCB proof of $N_a(T)-\tilde n_{a,T}$ must use the
actual decision-time scale $f_\tau$ and the exact moving target $\tilde n_{a,\tau}$, rather than a terminal frozen target.  To make the decision-time scale $f_\tau$ usable, the Poly-UCB proof replaces terminal auxiliary targets by moving targets.  At each decision time $\tau$, define $n_{a,\tau}$ and $\tilde n_{a,\tau}$ by the definition of fluid approximations. Then the actual index of an arm at time $\tau$ could be compared to the same decision-time frozen equation.

\fi
\end{document}